\documentclass[11pt]{article}

\usepackage{mystyle}
\usepackage{fullpage}

\newcommand{\han}[1]{{\textcolor{cyan}{[Han: #1]}}}

\usepackage{color,xcolor}
\usepackage{colortbl}
\usepackage{bbm}
\usepackage{tcolorbox}

\def\GEC{\mathrm{GEC}}

\def\mbo{\mathbf{1}}

\def\reg{\mathrm{Reg}}
\def\psr{\mathrm{PSR}}
\def\rank{\mathrm{rank}}
\def\nb{N_{\mathrm{batch}}}

\def\dummy{\mathrm{dummy}}

\definecolor{red1}{HTML}{f47983}
\definecolor{blue1}{HTML}{3eede7}
\definecolor{yellow1}{HTML}{f5dd6f}

\title{\LARGE GEC: A Unified Framework for Interactive Decision Making in MDP, POMDP, and Beyond}
\author{Han Zhong\thanks{Equal contribution, random order.}~\thanks{Peking University. Email: \texttt{hanzhong@stu.pku.edu.cn}} \qquad Wei Xiong$^*$\thanks{The Hong Kong University of Science and Technology. Email: \texttt{wxiongae@connect.ust.hk}}  \qquad Sirui Zheng\thanks{Northwestern University. Email: \texttt{siruizheng2025@u.northwestern.edu}}  \qquad Liwei Wang\thanks{Peking University. Email: \texttt{wanglw@cis.pku.edu.cn}} \\ Zhaoran Wang\thanks{Northwestern University. Email: \texttt{zhaoranwang@gmail.com}} \qquad  Zhuoran Yang\thanks{Yale University. Email: \texttt{zhuoran.yang@yale.edu}} \qquad Tong Zhang\thanks{The Hong Kong University of Science and Technology. Email: \texttt{tongzhang@tongzhang-ml.org}} }
\date{}
\begin{document}
\maketitle


\begin{abstract}
     We study sample efficient reinforcement learning (RL) under the general framework of interactive decision making, which includes Markov decision process (MDP), partially observable Markov decision process (POMDP), and predictive state representation (PSR) as special cases. Toward finding the minimum assumption that empowers sample efficient learning, we propose a novel complexity measure, generalized eluder coefficient (GEC), which characterizes 
    the fundamental tradeoff between exploration and exploitation in online interactive decision making in the context of function approximation. 
    In specific, GEC captures the hardness of exploration by comparing  the  error of predicting the performance of the updated policy with the in-sample training error evaluated on the historical data. We show that RL problems with low GEC form a remarkably rich class, which subsumes low Bellman eluder dimension problems, bilinear class, low witness rank problems, PO-bilinear class, and generalized regular PSR, where generalized regular PSR, a new tractable PSR class identified by us, includes nearly all known tractable partially observable RL models.  
    Furthermore, in terms of algorithm design, we propose a generic posterior sampling algorithm, which can be implemented in both model-free and model-based fashion, under both fully observable and partially observable settings. 
    The proposed algorithm   modifies the  standard posterior sampling algorithm in two aspects: (i) we use an optimistic prior distribution that biases towards hypotheses with higher values  and (ii) a  loglikelihood function is set to be the empirical loss evaluated on the historical data, where the choice of loss function supports both model-free and model-based learning. We prove that the proposed generic posterior sampling algorithm is sample efficient by establishing  a sublinear regret upper bound in terms of GEC.  In summary, we provide a new and unified understanding of both fully observable and partially observable RL.
\end{abstract}

\newpage
{
  \hypersetup{linkcolor=black,anchorcolor=blue,
citecolor=blue,filecolor=blue,urlcolor=blue}
  \tableofcontents
}
\newpage

\section{Introduction}

In a single-agent interactive decision making problem,
an agent interacts with an environment by executing an action at each step, and collects an immediate reward and receives an observation emitted from the environment. 
The goal of the agent is to find the optimal policy that yields the highest cumulative rewards in expectation, where the policy specifies the  action  taken at each step  based on the past observations at hand. 
Such a general description of interactive decision making includes various notable reinforcement learning \citep{sutton2018reinforcement} models   such as  the Markov decision process (MDP), partially observable Markov decision process (POMDP), and predictive state representations (PSR) \citep{littman2001predictive} as special cases. 

In the online setting, without the knowledge of the reward and observation models, we aim to learn an approximately optimal policy in a trial-and-error fashion --- iteratively executing the current policy to collect data and improving the current 
based on the historical data. 
The sample efficiency of the reinforcement learning algorithm is characterized by the \emph{regret}, which is defined as the cumulative suboptimality of the policies executed by the algorithm. 
Due to the unknown model, 
designing a low-regret algorithm 
involves a fundamental tension between exploration and exploitation \citep{sutton2018reinforcement}. 
In particular, 
to minimize the regret, we need to execute  the \emph{exploitative policy}, i.e., the policy that  is most likely to be optimal given the historical data. 
The performance of such an exploitation policy 
hinges on the estimation accuracy of model or value function  (depending on whether the algorithm is \emph{model-based} or \emph{model-free}) based on the historical data. 
To achieve high  estimation accuracy, however, the data collected should have good coverage over the observation space, which necessitates the usage of some \emph{explorative policy} in the algorithm to collect a more diverse dataset. 
The tension between deploying either an explorative or exploitative policy embodies the exploration-exploitation tradeoff.

To gracefully strike a balance between the exploration-exploitation tradeoff, 
most of existing sample efficient RL algorithms are based on the (i) principle of \emph{optimism in the face of uncertainty} \citep{auer2002finite, auer2002using, jaksch2010near} or (ii) \emph{posterior sampling} \citep{thompson1933likelihood, strens2000bayesian, russo2014learning}. 
In particular, algorithms based on the optimism principle 
constructs confidence sets of the model or value function, and executes the policy by planning according to the most optimistic model or value function  in the confidence set. 
By constructing optimistic estimates, the estimation uncertainty that depends on the quality of the data is incorporated into planning, which adjusts the learned policy to also explore areas  with high uncertainty. 
Meanwhile, algorithms based on the idea of posterior sampling first construct a posterior distribution over the model or value function  
and then update the policy by planning with respect to a random model or value function sampled from the posterior. 
Intuitively, exploration is attributed to the  randomness of the posterior distribution. 

Although the sample efficiency of RL algorithms in the tabular setting of MDP is relatively well-understood \citep{auer2008near, azar2017minimax,jin2018q, agrawal2017optimistic}, 
such a problem remains elusive in the contexts of (i) the function approximation setting where the state space is very large, or (ii) the partially observable setting such as POMDP and PSR. 
In particular, under the function approximation setting, the model or value function is estimated within a certain function class. 
To characterize the estimation uncertainty that is tailored to the employed function class, and  to quantify how the regret is connected to such estimation error, 
various notions of complexity measures and structural assumptions have been introduced in the literature. 
For example,  linear MDP structure \citep{yang2019sample, jin2020provably}, linear mixture MDP \citep{ayoub2020model,modi2020sample,cai2020provably,zhou2021nearly}, Bellman completeness \citep{munos2003error,zanette2020learning},   Bellman rank \citep{jiang@2017},   eluder dimension \citep{osband2014model, wang2020reinforcement},  witness rank \citep{sun2019model}, Bellman eluder dimension \citep{jin2021bellman} and bilinear class \citep{du2021bilinear}. 
Moreover, these structural assumptions mainly focus on the MDP model and the  corresponding algorithms often leverage these particular  structures to achieve sample efficiency.
Furthermore, under the partially observable setting, \cite{krishnamurthy2016pac,jin2020sample} prove that, even when the observation space is finite and there are two latent states, there exists a POMDP instance for which any learning algorithm requires a number of samples that is exponential in horizon length.
Such a depressing result necessitates additional structural assumptions for achieving sample efficient learning. 
To this end, a line of recent works proposes sample efficient RL algorithm under various POMDP and PSR settings \citep{jin2020sample, liu2022partially, kwon2021rl, du2019provably, efroni2022provable, cai2022reinforcement, wang2022embed, zhan2022pac, uehara2022provably}. These algorithms mainly implement the optimism principle. 

In this work, we aim to establish a unified framework for online interactive decision making in the context of function approximation. 
In particular, we aim to address the following two questions: 
\begin{tcolorbox}
First, can we propose a general complexity measure that captures the exploration-exploitation tradeoff in most of the existing models studied in the literature? 
Second, in terms of algorithm design, can we establish a general-purpose sample efficient algorithm and establish regret bounds for such an algorithm in terms of the proposed complexity measure? 
\end{tcolorbox}

We provide affirmative answers to both questions. In particular, we propose a generic  complexity  measure, dubbed as 
generalized eluder coefficient (GEC), which characterizes the foundational challenge of exploration-exploitation tradeoff by connecting (i) the training error of minimizing a loss function $\ell$ over a hypothesis class $\mathcal{H}$ on the historical data and (ii) the prediction error of evaluating the updated policy. 
Here, depending on the  choice of $\ell$ and $\cH$, the algorithm can be either \emph{model-based} or \emph{model-free}. 
The rationale of GEC is that if the historical data is well-explored, then the out-of-sample prediction error of the newly learned policy is small, and therefore  the relative ratio between the  training error and prediction error characterizes how challenging it is to explore the environment.
Moreover, in online setting, the goal is to minimize the regret. As a result, we only need to focus on the prediction error of the updated policy, which incurs suboptimality in the next rounds. 
Such a reasoning is motivated from the eluder dimension \citep{osband2014model}, thus giving the name of GEC. 
Besides, based on the optimism principle, the prediction error of the updated policy serves as an upper bound of its incurred regret. Therefore, GEC can be used to establish a regret upper bound for the algorithms that implement the optimism principle. 
Then we show that GEC can capture a majority of known tractable interactive decision making. In specific, we prove that low Bellman eluder dimension problems \citep{jin2021bellman}, bilinear class \citep{du2021bilinear}, problems with low witness rank \citep{sun2019model}, PO-bilinear class \citep{uehara2022provably}, and \emph{generalized regular PSR} have a low GEC. Here generalized regular PSR is a rich subclass of PSR identified by us and subsumes weakly revealing POMDPs \citep{jin2020sample,liu2022partially}, latent MDP with full-rank test condition\footnote{Since solving latent MDPs requires exponential samples in the worst case \citep[][Theorem 3.1]{kwon2021rl}, we adopt a full-rank test assumption \citep[][Condition 1]{kwon2021rl} to  make latent MDPs tractable. } \citep{kwon2021rl}, decodable POMDP \citep{du2019provably,efroni2022provable}, low-rank POMDP \citep{wang2022embed}, and regular PSR \citep{zhan2022pac}. We visualize and compare previous complexity measures and GEC in Figure~\ref{fig:venn}.

\begin{figure}
    \centering
    \includegraphics[width=0.8\textwidth]{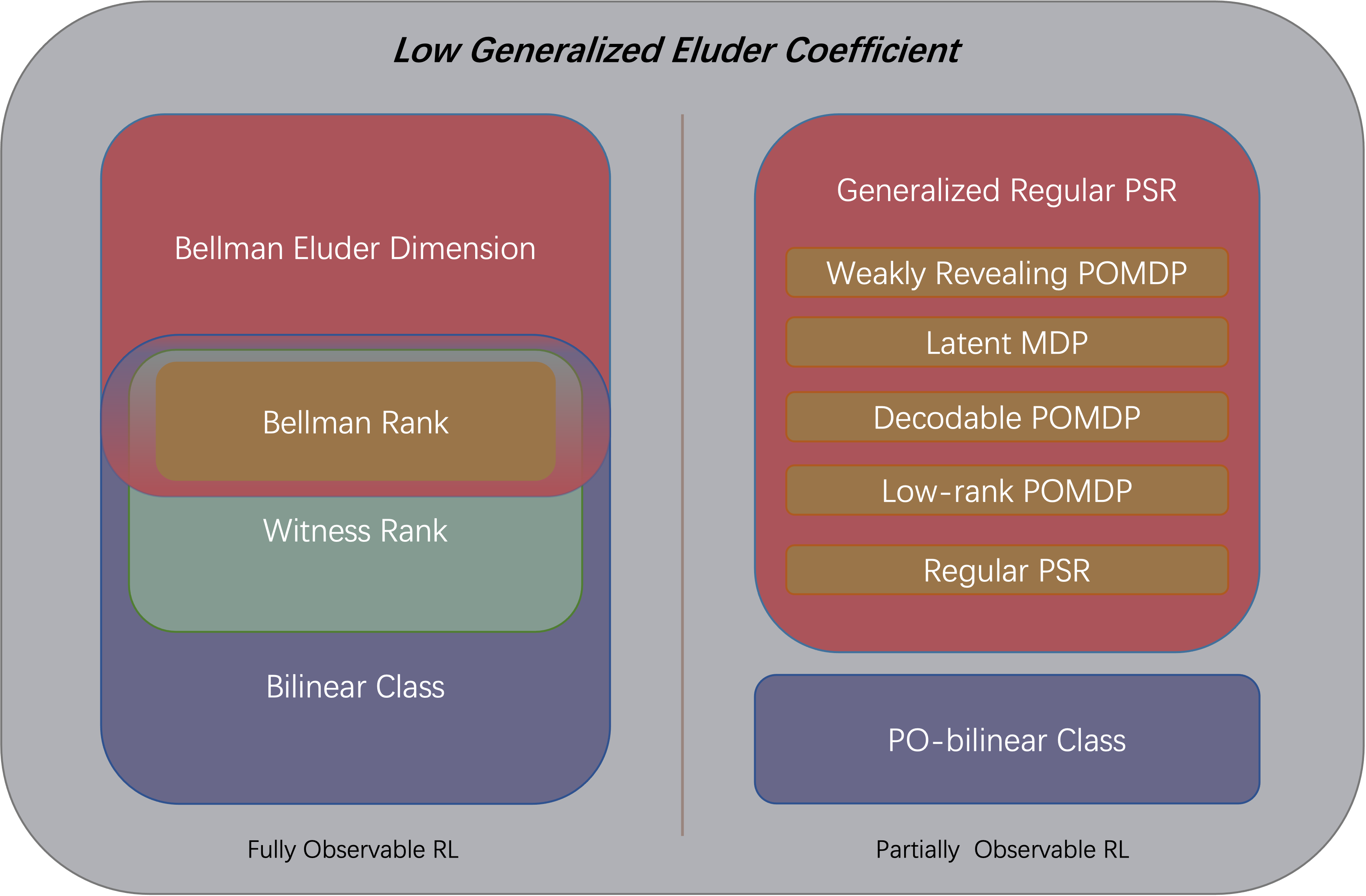}
    \caption{A summary of interactive decision making that can be captured by GEC. 
    Here generalized regular PSR and PO-bilinear class are disjoint because in the generalized regular PSR we can find the globally optimal policy while in PO-bilinear class the goal is to find an optimal $M$-memory policy where $M \in \NN^+$.} 
    \label{fig:venn}
\end{figure}

Furthermore, to complement optimistic algorithms, 
we propose a generic posterior sampling algorithm for interactive decision making,
which can be applied to both model-free and model-based  RL with full and partial observations. 
Posterior sampling algorithms seem more agreeable to computationally tractable implementations via  first-order sampling methods \citep{welling2011bayesian} or ensemble methods \citep{osband2016deep,lu2017ensemble,chua2018deep,nagabandi2020deep}.
In particular, the proposed  algorithm maintains an \emph{optimistic posterior distribution} over the hypotheses where the loglikelihood function is given by the loss function evaluated on historical data. 
Such a posterior is optimistic in the sense that the prior distribution is biased towards hypotheses with higher values, which is inspired by the feel-good Thompson sampling algorithm proposed by \cite{zhang2022feel}.  
On MDP, POMDP, and PSR, 
we prove that 
optimistic posterior sampling achieves a sublinear regret and thus is provably sample efficient.
In particular, when the hypothesis class $\cH$ is finite, the regret is $  \cO( \sqrt{d_{\mathrm{GEC} }\cdot H  T \cdot \log | \cH |  })$, where    $d_{\mathrm{GEC} }$ is GEC, $H$ is the horizon length, and $T$ is the total number of episodes. 
As we will show in \S\ref{sec:result:full} and \S\ref{sec:result:partial}, when specialized to concrete examples of interactive  decision making, such a generic regret bound involving GEC matches those established in the literature using problem-specific quantities such as the size of (latent) state space or the rank of the transition model. 
Therefore, we show that GEC serves as  a unified   complexity measure  for theoretically understanding the exploration-exploitation tradeoff in interactive decision making, and optimistic  posterior sampling  is a generic algorithm that is provably sample efficient.

\subsection{Related Works} 

\paragraph{Learning Fully Observable Reinforcement Learning.} 
One core problem in fully observable RL is identifying the minimal structural assumption under which we can learn the MDP efficiently. The tabular MDPs with finite states and actions have been well-studied by a series of works. See, e.g., \citet{auer2008near, azar2017minimax,dann2017unifying,jin2018q, agrawal2017optimistic,zanette2019tighter,zhang2021isreinforcement,menard2021ucb,li2021breaking,wu2022nearly,zhang2022horizon} and the references therein.  
However, the minimax lower bound depends on the number of states, suggesting that MDPs with large state space cannot be handled efficiently without further structural assumptions. Beyond tabular MDPs, many studies focus on RL with function approximation, which allows an enormous number of states. Among them, a line of works considers  RL with linear function approximation (e.g., \citet{wang2019optimism, yang2019sample, cai2020provably, jin2020provably, zanette2020frequentist, ayoub2020model, modi2020sample, zhou2021nearly}), which are arguably the simplest setting of function approximation. For general function approximation, \citet{russo2013eluder} propose  the notion of eluder dimension to characterize the complexity of a general function class in the literature of bandit. The eluder dimension is later leveraged to RL by \citet{wang2020reinforcement}, which includes the linear MDP \citep{jin2020provably} as a special example. Another line of works proposes  two low-rank structures, namely, the Bellman rank for model-free RL \citep{jiang@2017}, and the witness rank for model-based RL \citep{sun2019model}. \citet{zanette2020learning} consider  a different type of framework based on Bellman completeness, which assumes that the class used for approximating the optimal Q-functions is closed in terms of the Bellman operator and improves the result for linear MDP by establishing optimism only at the initial state. 
Following these works, \citet{du2021bilinear} propose  the bilinear class, which contains more MDP models with low-rank structures \citep{azar2017minimax,sun2019model,jin2020provably,modi2020sample,cai2020provably,zhou2021nearly}.  \citet{du2021bilinear} also establish an algorithm that solves bilinear class with sample efficiency.  
Parallel to the bilinear class, \cite{jin2021bellman} proposes the Bellman-eluder (BE) dimension, which unifies the Bellman rank \citep{jiang@2017} and low eluder dimension \citep{wang2020reinforcement}. They also design sample efficient algorithms for decision making problems with bounded BE dimension. However, it is known that neither BE nor bilinear class captures each other. Our work is also closely related to the online decoupling coefficient introduced in \citet{dann2021provably}, which consider the Q-type model-free RL. We present a detailed comparison with \citet{dann2021provably} when reviewing the related works on posterior sampling in the sequel. Our work can tackle all these tractable fully observable RL problems.

Recently, \citet{foster2021statistical} propose  the decision estimation coefficient (DEC) to unify the complexity measures in interactive decision making. \citet{foster2021statistical} show  that a rich class of tractable fully observable RL problems has bounded DEC. They also propose a generic model-based algorithm to solve RL problems with low DEC. However, their algorithm requires solving a complicated minimax optimization problem. In comparison, our posterior sampling algorithm is relatively clean and easy to implement in both model-free and model-based fashion. More importantly, our work can also tackle the more challenging partially observable RL problems.

\paragraph{Learning Partially Observable Reinforcement Learning.}
Our work is closely related to the study of sample efficiency in RL with partial observations. In partially observable RL, observations are no longer Markovian, which results in that the optimal policy in general depends on the full history. As shown in \citet{krishnamurthy2016pac,jin2020sample}, learning such a history-dependent policy requires exponential samples in the worst case. But this worst-case hardness result does not rule out the possibility of sample efficient learning in a rich subclass of partially observable RL. Some works on POMDP \citep{guo2016pac,azizzadenesheli2016reinforcement,xiong2021sublinear} and PSR \citep{boots2011closing,hefny2015supervised,jiang2018completing,zhang2021reinforcement} obtain polynomial sample complexity results with some exploratory data or reachability assumption. 
But these works do not address the  exploration-exploitation tradeoff, which is a central challenge in online RL.  
For the more general partially observable RL where the learner needs to tackle the challenges from partial observations and exploration, a line of recent works identifies several tractable subclasses of partially observable RL, including weakly revealing POMDPs \citep{jin2020sample,liu2022partially}, latent MDP with sufficient tests \citep{kwon2021rl}, decodable POMDPs \citep{du2019provably,efroni2022provable}, low-rank POMDPs \citep{wang2022embed}, regular PSR \citep{zhan2022pac}, and design  sample efficient algorithms for those settings. The very recent work of \citet{uehara2022provably} proposes the PO-bilinear class and shows that the definition of the PO-bilinear class captures a rich class of tractable RL problems with partial observations, including weakly revealing POMDPs, decodable POMDPs, and observable POMDPs with low-rank transitions \citep{uehara2022provably}. Our work provides a unified theoretical and algorithmic framework of interactive decision making that can be applied to all of  these known tractable partially observable RL problems.

In terms of computation, \citet{papadimitriou1987complexity} show  that planning (finding the optimal policy with the known model) in POMDP is PSPACE-complete. Recently, \citet{golowich2022learning,golowich2022planning} study the observable POMDP (which is equivalent to the weakly revealing POMDP up to polynomial factors) and designs algorithm that can find the near-optimal policy with quasi-polynomial time and quasi-polynomial samples. \citet{jin2020sample} and \citet{uehara2022computationally} propose both computationally and sample efficient algorithms for POMDPs with deterministic latent transitions. Computational efficiency for the partially observable models studied in these works is beyond the scope of this paper, and we leave this as future work.

\paragraph{Posterior Sampling in Reinforcement Learning.} Our work is also closely related to a line of work that considers posterior sampling in RL. Most existing works such as \citet{russo2014learning} study the Bayesian regret, which is weaker than the frequentist/worst-case regret.  
 To minimize the frequentist regret, \citet{osband2016generalization} and \citet{russo2019worst} consider the randomized least-squares value iteration algorithm (RLSVI) in tabular MDP and is extended to the linear setting by \citet{zanette2020frequentist}. For general function approximation, \citet{zhang2022feel} first show that it is necessary to employ a posterior with optimistic modifications to achieve the optimal frequentist regret bound in contextual bandit and certain RL problems with general function approximation. After this, \citet{dann2021provably} propose  the conditional posterior sampling algorithm to deal with Q-type model-free RL problems. \citet{agarwal2022model} propose  a general posterior sampling method for model-based RL using the loglikelihood function. \citet{agarwal2022non} generalize the V-type results to infinite actions and designs a novel two-timescale exploration strategy, which serves a similar role of one-step pure exploration for the finite-action case, to solve the problem. Among them, our work is mostly related to \citet{dann2021provably} because both our work and \citet{dann2021provably} are based on the eluder dimension technique \citep{russo2013eluder}, but the analysis in \citet{zhang2022feel,agarwal2022model,agarwal2022non} are built on the ``decoupling technique'' developed by \citet{zhang2022feel}. Compared with \citet{dann2021provably}, the most appealing feature of our work is  universality, as our work not only can deal with the Q-type problems in MDP, but also various V-type problems in MDP and even partially observable RL. 

About one month before our initial submission, several related independent works study fully observable RL \citep{chen2022general,chen2022unified} and partially observable RL \citep{chen2022partially,liu2022optimistic}. We will provide a detailed comparison with them in Appendix~\ref{appendix:concurrent}.


\subsection{Notations}
We use $\Delta_\cX$ to denote the space of all distributions over the set $\cX$. For some $n \in \NN^+$, we use the convention that $[n] = \{1, \cdots, n\}$. We also let $x_{1:n} = \{x_1, \cdots, x_n\}$. Given a matrix $\Mb \in \RR^{m \times n}$, we use $\sigma_k(\Mb)$ to denote the $k$-th largest singular value, $\Mb^\dagger$ to denote its Moore-Penrose inverse, and $\|\Mb\|_p = \max_{\|\bfx\|_p = 1} \|\Mb \bfx\|_p$ to denote its matrix $p$-norm. For any distribution $p \in \Delta_{[n]}$ and sequence $x_{1:n}$, we denote $\sum_{i \sim p} x_i = \sum_{i = 1}^n p(i) x_i$. For two distributions $P, Q \in \Delta_\cX$, we define the total variation distance, KL divergence, and Hellinger divergence by $\TV(P, Q) = \frac{1}{2} \cdot \EE_{x \sim P}[|dQ(x)/dP(x) - 1|]$, $\KL(P \| Q) = \EE_{x \sim P} [\log dP(x)/dQ(x)]$, and $D_H^2(P, Q) = \frac{1}{2} \cdot \EE_{x \sim P}[(\sqrt{dQ(x)/d P(x)} - 1)^2]$, respectively. For two non-negative sequences $\{a_n\}_{n \ge 0}$ and $\{b_n\}_{n \ge 0}$, if $\limsup a_n/b_n < \infty$, we write $a_n = \cO(b_n)$ or $a_n \lesssim b_n$. We also use $\tilde{\cO}$ to further omit logarithmic factors. To improve the readability, we provide a summary of notations in Appendix~\ref{appendix:notation}.

\section{Preliminaries}

\subsection{Interactive Decision Making}

An episodic interactive decision making problem is specified by a tuple $(\mathcal{O}, \cA, H, \PP, R)$, where $\cO$ and $\cA$ denote the space of observation and action respectively; $H$ is the length of each episode, 
$\PP = \{\PP_h\}_{h \in [H]}$ is the transition kernel, where $\PP_h(o_{h+1} \mid \tau_h)$ denotes the probability of generating the observation $o_{h+1}$ given the history $\tau_h = \{o_{1:h}, a_{1:h}\} = \{o_1, a_1, \cdots, o_h, a_h\}$; $R = \{R_h\}_{h \in [H]}$ is the reward function with $R_h(o, a)$ equal to the known deterministic reward\footnote{Our following results can be extended to the unknown stochastic reward case since learning transition kernel is more challenging than learning the rewards.} after taking action $a$ at observation $o$ and step $h$.  Furthermore, we assume the reward is non-negative and $\sum_{h = 1}^H R_h \le 1$. 

We assume each episode starts from an observation $o_1$, which is sampled from a fixed distribution. At each step $h \in [H]$, the agent observes the history $\tau_{h-1}$ and the current observation $o_h$, takes an action $a_h$, receives a reward $r_h = R_h(o_h, a_h)$. Then the environment generates a new observation $o_{h+1} \sim \PP_h(\cdot \mid \tau_h)$. The episode ends after observing the dummy observation at step $H+1$, i.e., $o_{H+1} = o_\dummy$.

A general (history-dependent) policy $\pi = \{\pi_h\}_{h \in [H]}$ consists of $H$ functions, where $\pi_h: (\cO \times \cA)^{h-1} \times \cO \rightarrow \Delta_{\cA}$ is a mapping from an observation-action sequence to a distribution over actions. Given $(\tau_h, \pi)$, the probability of observing $\tau_h$ when executing $\pi$ is specified by
\# \label{eq:def:prob1}
\PP^{\pi}(\tau_h) := \PP(\tau_h) \times \pi(\tau_h),
\#
where $\PP(\tau_h)$ and $\pi(\tau_h)$ are defined by
\# \label{eq:def:prob2}
\PP(\tau_h) := \prod_{h'=1}^{h} \PP(o_{h'} \mid \tau_{h' - 1}), \qquad \pi(\tau_h) := \prod_{h' = 1}^h \pi_{h'}(a_{h'} \mid \tau_{h' - 1}, o_{h'}).
\#
For any policy $\pi$, its value function $V^{\pi}$ is defined by 
\$
V^\pi := \EE_{\pi} \Big[\sum_{h = 1}^H r_h \Big],
\$
where the expectation is taken with respect to the randomness incurred by the transition kernel $\PP$ and policy $\pi$. Our goal is to find the optimal policy $\pi^*$ which maximizes the value function, i.e., $\pi^* = \argmax_{\pi \in \Pi} V^{\pi}$, where $\Pi$ is the history-dependent policy class. We also denote $V^* = V^{\pi^*}$. 

\paragraph{Learning Objective} We consider the regret\footnote{Our definition is slightly different from the conventional regret since we allow the learner predicts policy $\pi$ and executes the exploration policy $\pi_{\exp}$.} minimization problem. Suppose the learner predicts policy $\pi^t$ in the $t$-th episode for $t \in [T]$, then the regret for $T$ episodes is defined as
    \$
    \reg(T) = \sum_{t= 1}^T  (V^* - V^{\pi^t} ) .
    \$
We also consider the sample complexity, i.e., the number of episodes to learn an $\varepsilon$-optimal policy $\pi$ satisfying $V^{\pi} \ge V^* - \varepsilon$. If $\reg(T)$ of an algorithm is sublinear, i.e., $\reg(T) = o(T)$,  then we can find an $\varepsilon$-optimal policy within $T$ episodes, where $T$ is a sufficiently large number  satisfying  $\reg(T)/T \le \varepsilon$. 

\subsection{Concrete Models: MDP, POMDP, and PSR} 

In the following, we present the details of three   classes of models ---  Markov decision process (MDP), partially observable Markov decision process (POMDP), and Predictive State Representation (PSR), which are important special cases of the interactive decision making framework that have been extensively studied in the literature. 

\begin{example}[MDP] \label{example:mdp}
    A Markov decision process (MDP) is specified by a tuple $(\cS, \cA, H, \PP, R)$. Here the state space $\cS$ can be regarded as the observation space $\cO$ in the interactive decision making formulation. Moreover, the transition kernel for MDPs satisfies the Markov  property, i.e., $\PP_h(x_{h+1} \mid x_{1:h}, a_{1:h}) = \PP_h(x_{h+1} \mid x_{h}, a_{h})$ for any $h \in [H]$ and $(x_{1:h+1}, a_{1:h}) \in \cS^{h+1} \times \cA^h$. A Markovian policy $\pi = \{\pi_h: \cS \rightarrow \Delta_\cA\}_{h \in [H]}$ maps each state to a distribution over actions. Given a Markovian  policy $\pi$, its Q-function and value function at step $h$ are defined as 
    \$
    Q_h^\pi(x, a) = \EE_{\pi} \Big[ \sum_{h'=h}^H r_{h'} \,\Big|\, x_h = x, a_h = a\Big], \quad V_h^\pi(x) = \EE_{\pi} \Big[ \sum_{h'=h}^H r_{h'} \,\Big|\, x_h = x\Big].
    \$
    We also use $\pi^* = \{\pi_h^*\}_{h \in [H]}$, $V^* = \{V_h^*\}_{h \in [H]}$, and $Q^* = \{Q_h^*\}_{h \in [H]}$ to denote the optimal (Markovian) policy, optimal value function and optimal Q-function, respectively. We remark that here the existence of an optimal Markovian policy is because of the Markov  property of the transition kernel. For general interactive decision making problems, the optimal policy is history-dependent. It is well known that $(Q^*, V^*)$ satisfies the Bellman equation for any $(h, x, a) \in [H] \times \cS \times \cA$:
    \# \label{eq:Bellman:equation}
    Q_h^*(x,a) = (\cT_h V_{h+1}^*)(x, a) := r_h(x, a) + \EE_{x' \sim \PP_h(\cdot \mid x, a)} 
    V_{h+1}^*(x'), \quad V_h^*(x) = \max_{a \in \cA} Q_h^*(x, a). 
    \#
    Here $\cT_h$ is referred to as the \emph{Bellman operator} at step $h$.
\end{example}

\begin{example}[POMDP] \label{example:pomdp}
    A partially observable Markov decision process (POMDP) is specified by a tuple 
\$
(\cS, \cO, \cA, H, \mu_1, \TT = \{\TT_{h, a}\}_{(h, a) \in [H] \times \cA}, \OO = \{\OO_h\}_{h \in [H]}, R = \{R_h\}_{h \in [H]}),
\$
where $\cS$ with cardinality $S = |\cS|$, $\cO$ with cardinality $O = |\cO|$, and $\cA$ with cardinality $A = |\cA|$ are the state, observation, and action spaces, respectively. Here $H$ is the length of each episode, $\mu_1 \in \Delta_\cS$ is the distribution of the initial state, $\TT_{h, a}(\cdot \mid \cdot)$ is the transition kernel of action $a$ at step $h$, and $\OO_h(\cdot \mid \cdot)$ is the emission kernel at step $h$, and $R$ is the known deterministic reward function. 
\end{example}

\begin{example}[PSR]
We introduce predictive state representations (PSRs) \citep{littman2001predictive,singh2012predictive}, a more expressive class than POMDPs.

PSRs use the concept of \emph{test}, which is a sequence of future observations and actions. Specifically, for some test $t_h = (o_{h: h + W - 1}, a_{h: h + W - 2})$ with length $W$, we define the probability of test $t_h$ being successful conditioned on the reachable history $\tau_{h-1}$ as the probability of observing $o_{h: h + W - 1}$ by executing $a_{h: h + W -2}$ conditioned on the history $\tau_{h-1}$, that is, $\PP(t_h \mid \tau_{h-1}) = \PP(o_{h: h + W - 1} \mid \tau_{h-1}, \mathrm{do}(a_{h: h + W - 2}))$. When $\PP^{\pi}(\tau_{h-1}) = 0$ for any $\pi$, we let $\PP(t_h \mid \tau_{h-1}) = 0$. Then we define the one-step system dynamics matrix $\DD_h$ with tests as rows and histories as columns and the  $(t_{h+1}, \tau_h)$-th entry of $\DD_h$ is  equal to $\PP(t_{h+1} \mid \tau_{h})$. 
In other words,  
the number of rows of $\DD_h$ is equal to the number of all possible tests starting from $o_{h+1}$ and the 
number of columns is equal to the number of all possible histories before observing $o_{h+1}$. 
The possibly infinite-dimensional matrix $\DD_h$ encodes all probabilities of the form $\PP( t_{h+1} \mid \tau_{h})$, which summarizes the dynamics of PSR.  
To tractably learn a PSR, it is often assumed that $\DD_h$ is low-rank, which further leads to the definition of PSR rank, which is presented as follows.

\begin{definition}[PSR Rank] \label{def:psr:rank}
    For a PSR, its rank $d_{\psr}$ is defined as $d_{\psr} = \max_{h \in [H]} d_{\psr, h}$, where $d_{\psr, h} = \rank(\DD_h)$. 
\end{definition}

\begin{remark}
    All POMDPs defined in Example \ref{example:pomdp} are PSRs with rank at most $|\cS|$. See e.g., Theorem 2 in \citet{singh2012predictive} for the proof.
\end{remark}

\noindent{\bf Core Test Sets and Predictive States.} \quad
For a low-rank PSR, $\DD_h$ is a low-rank and infinite-dimensional matrix for any $h \in [H]$, and thus $\DD_h$ admits a more compact representation, which is given by the core tests and  predictive states introduced as follows.  
For any $h \in [H]$, we say a set $\mathcal{U}_h \in \cup_{W \in \mathbb{N}^+} \cO^W \times \cA^{W-1}$ is a \emph{core test set} if for any possible test $t_h = (o_{h:h+W-1}, a_{h: h+W-2})$ and any history $\tau_{h-1}$, there exists $\bfm(t_h) \in \RR^{|\cU_h|}$ such that
\# \label{eq:linear:weight}
\PP(t_h \mid \tau_{h - 1} )= \la \bfm(t_h), [\PP(t \mid \tau_{h-1})]_{t \in \cU_h} \ra. 
\#
Here, $\bfm $ maps each any  test  to a vector in $\RR^{ | \cU_h | }$. 
We refer to $\bar{\bfq}(\tau_{h-1}) := [\PP(t \mid \tau_{h-1})]_{t \in \cU_h} \in \RR^{| \cU_h | } $ as the \emph{predictive state} at step $h \ge 2$. 
For $h = 1$, we use $\bfq_0 : = [\PP(t)]_{t \in \cU_1}$. Besides, let us define a matrix  
\# \label{eq:def:bar:D}
\bar{\DD}_h:= [\PP(t \mid \tau_h)]_{t \in \cU_{h+1}, \tau_h \in (\cO \times \cA)^h} \in \RR^{|\cU_{h+1}| \times  |\cO|^h |\cA|^h},
\#
which is a row-wise sub-matrix of $\DD_h$ and its rows are indexed by tests in $\cU_{h+1}$.
By \eqref{eq:linear:weight}, we know any row in $\DD_h$ can be expressed by a linear combination of rows of $\bar{\DD}_h$, which further implies that $\rank(\bar{\DD}_h) = \rank(\DD_h) = d_{\psr, h}$.
We also define $\cU_{A, h}$ as the action sequences in $\cU_h$, i.e., $\cU_{A,h} = \{a_{h: h + W - 2} \mid  (o_{h: h+W-1}, a_{h : h+W-2}) \in \cU_h \text{ for some } o_{h: h+W-1} \in \cO^W\}$. For ease of presentation, we denote $U_A = \max_{h \in [H]} |\cU_{A,h}|$ and $\cU_{H+1} = o_\dummy$. Throughout this paper, we assume the core test sets $\{\cU_{h}\}_{h \in [H]}$ are known to the learner.

\vspace{4pt}
\noindent{\bf Observable Operator Representation.} \quad
It is known that any PSR permits the observable operators representation \citep{jaeger2000observable}. Formally, given a PSR with a core test set $\{\cU_h\}_{h \in [H]}$, there exists a set of matrices $\{\Mb_h(o, a) \in \RR^{|\cU_{h+1}| \times |\cU_h|}, \bfq_0 \in \RR^{|\cU_1|}\}_{o \in \cO, a \in \cA, h \in [H]}$ that can characterize its dynamics. For any history $\tau_h$, it holds that
\$
[\PP(\tau_h, t) ]_{t \in \cU_{h+1}} &:= [\PP(t \mid \tau_h) \PP(\tau_h)]_{t \in \cU_{h+1}} = \Mb_{h}(o_h, a_h)  \cdots \Mb_1(o_1, a_1) \bfq_0.
\$
Here $[\PP(\tau_h, t) ]_{t \in \cU_{h+1}}$ is regarded as a vector in $\RR^{|\cU_{h+1}|}$ with each entry indexed by $t \in \cU_{h+1}$, and $\Mb_{h}$ maps each observation-action pair to a matrix in $\RR^{|\cU_{h+1}| \times |\cU_h|}$. 
In particular, for any $\tau_H$, we have
\# \label{eq:prob:traj}
\PP(\tau_H) &= \Mb_{H}(o_H, a_H) \Mb_{H - 1}(o_{H-1}, a_{H-1}) \cdots \Mb_1(o_1, a_1) \bfq_0.
\#
Here we remark that $\Mb_H \in \RR^{1 \times |\cU_H|}$ since we choose $|\cU_{H+1}| = 1$. For any full-length test $t_h = (o_{h:H+1}, a_{h:H})$, the linear weight $\bfm(t_h)$ in \eqref{eq:linear:weight} is given by 
\# \label{eq:def:m}
\bfm(o_{h:H}, a_{h:H}) = \Mb_{H}(o_H, a_H) \Mb_{H - 1}(o_{H-1}, a_{H-1}) \cdots \Mb_h(o_{h}, a_{h}),
\#
where we omit the dummy observation $o_{H+1} = o_\dummy$.
To simplify notations, we use the shorthand $\Mb_{h:1}(o_{1:h}, a_{1:h}) = \Mb_{h}(o_h, a_h)  \cdots \Mb_1(o_1, a_1)$ and $\bfq(\tau_h) = \Mb_{h:1}(\tau_h) \bfq_0$ for any $h \in [H]$.
\end{example}

For the convenience of readers, we summarize the notations in Appendix~\ref{appendix:notation}.

\section{Complexity Measure --- Generalized Eluder Coefficient} \label{sec:GEC}

In this section, we propose a general complexity measure --- generalized eluder coefficient (GEC). Then we demonstrate the universality of our complexity measure by providing several examples. Similar to \citet{du2021bilinear}, we assume that we have access to a hypothesis class $\cH = \cH_1\times \cdots \times \cH_H$, which can be either \emph{model-based} or \emph{value-based}. The elements of $\cH$ may  vary depending on the RL problems, and we present two examples as follows.

\begin{example}[Model-based Hypothesis] \label{example:hyp:model}
    An example of a model-based hypothesis class is when $\cH$ is a set of models (transition kernel $\PP$ and reward function $R$). In this case, for any $f = (\PP_f, R_f) \in \cH$, we denote $\pi_f = \{\pi_{h,f}\}_{h \in [H]}$ and $V_f = \{V_{h,f}\}_{h \in [H]}$ as the optimal policy and optimal value functions
    corresponding to the model $f$, respectively. To facilitate the following analysis, we denote the real model by $f^*$. For MDPs, we let $Q_f = \{Q_{h,f}\}_{h \in [H]}$ be the optimal Q-functions
    with respect to the model $f$. 
\end{example}

\begin{example}[Value-based Hypothesis for MDP] \label{example:hyp:value:1}
    A value-based hypothesis for MDPs is a set of Q-function, that is,
    $\cH = \{\cH_h\}_{h \in [H]}$, where $\cH_h = \{Q_h: \cS \times \cA \rightarrow \mathbb{R}\}$. For any $f = \{Q_h\}_{h \in [H]}$, let $Q_f = \{Q_{h, f} = Q_h\}_{h \in [H]}$, $V_f = \{V_{h, f}(\cdot) = \max_{a \in \cA} Q_{h, f}(\cdot, a)\}_{h \in [H]}$, and $\pi_f = \{\pi_{h,f}(\cdot) = \argmax_{a \in \cA}Q_h(\cdot, a)\}_{h \in [H]}$. We also denote $f^* = Q^*$, where $Q^*$ is the optimal Q-function.
\end{example}

Throughout this paper, we assume that $\cH$ contains $f^*$ (cf. Example \ref{example:hyp:model} or Example \ref{example:hyp:value:1}), which is standard in the literature \citep[e.g.,][]{du2021bilinear, jin2021bellman}.
\begin{assumption}[Realizability] \label{assu:realizability} We assume $f^* \in \cH$.
\end{assumption}

Now we introduce a new complexity measure -- generalized eluder coefficient (GEC).

\begin{definition}[Generalized Eluder Coefficient] \label{def:gec}
    Given a hypothesis class $\cH$, a discrepancy function $\ell = \{\ell_f\}_{f \in \cH}$, an exploration policy class $\Pi_{\mathrm{exp}}$, and $\epsilon > 0$, the generalized eluder coefficient $\mathrm{GEC}(\cH, \ell, \Pi_{\mathrm{exp}}, \epsilon)$ is the smallest  $d$ ($d\geq 0$) such that for any sequence of hypotheses and exploration policies $\{f^t,  \{{\pi}_{\mathrm{exp}}(f^t, h)\}_{h \in [H]}\}_{t \in [T]}$:
    \$
    \sum_{t = 1}^T \underbrace{V_{f^t} - V^{\pi_{f^t}}}_{\displaystyle \text{\textcolor{red1}{prediction error}}} \le \bigg [ d \sum_{h=1}^H \sum_{t=1}^T \underbrace{\Big( \sum_{s=1}^{t-1} \EE_{{\pi}_{\mathrm{exp}}(f^s,h)} \ell_{f^s}(f^t, \xi_h) \Big)}_{\displaystyle \text{\textcolor{cyan}{training error}}} \bigg]^{1/2} + \underbrace{ 2\sqrt{dHT} + \epsilon H T}_{ \displaystyle \text{\textcolor{violet}{burn-in cost}}}.
    \$
    Here $\pi_{\exp}(f, h)$ is the exploration policy determined  by the hypothesis $f$ and the  timestep $h$, and the expectation $\EE_{\pi_{\exp}(f^s, h)}$ is taken with respect to $\xi_h$, which are samples induced by $\pi_{\exp}(f^s, h)$. 
    That is, $\xi_h$ is a subset of the trajectory generated by $\pi_{\exp}(f^s, h)$ and its specific meaning may vary according to the choice of the loss function $\ell$. 
    Moreover, the choise of the loss function $\ell$ is problem-specific. 
    For generality, we allow $\ell$ to also depend on $f^s$. 
    Examples of the exploration policies $ \pi_{\exp}(f^s, h)$ and loss function $\ell_{f^s}$ will be presented in the sequel. 
    Furthermore, here $2\sqrt{dHT} + \epsilon H T$ is just a worst-case choice of burn-in cost and we will specify this term for the specific problems. We also remark that, by choosing a small $\epsilon$ such as $1/(HT)$, the \textcolor{violet}{burn-in cost} is usually a non-dominating term in the final result. When we choose $\epsilon = 0$, we denote $\GEC(\cH, \ell, \Pi_{\exp}) = \GEC(\cH, \ell, \Pi_{\exp}, 0)$.
\end{definition}
The main intuition behind the GEC is that on average, if hypotheses have small \textcolor{cyan}{in-sample training error} on the dataset collected so far and if the dataset is well-explored, then the \textcolor{red1}{out-of-sample prediction error} on the next trajectory generated   is also small on the long run. 
As a result, intuitively, how well we can use the in-sample training error to predict the performance of the updated policy reflects the hardness of exploring the environment. 
Furthermore, such an  observation  motivates us to design algorithms based on posterior sampling or the optimism principle. 
In particular, in posterior sampling algorithms, we assign a larger probability to a hypothesis if its training error is rather small, and in  optimism-based algorithms (e.g., UCB), we construct an explicit confidence set  of   $f^*$  by focusing the set of hypotheses with   small in-sample training errors. 

\paragraph{Optimism + Low GEC $\approx$ Sample Efficiency.} Now we informally explain why \emph{low GEC} may empower the sample efficient algorithms. For simplicity, we focus on the Q-type problems in MDP, including tabular MDP \citep{azar2017minimax,jin2018q}, linear MDP \citep{yang2019sample,jin2020provably}, Q-type Bellman rank \citep{jiang@2017}, and Q-type Bellman eluder dimension \citep{jin2021bellman}. For such a type of problems, we choose $\cH$ as in Example~\ref{example:hyp:value:1} and $\EE_{\pi_{\exp}(f^s,h)} \ell_{f_s}(f^t, \xi_h) = \EE_{\pi_{f^s}}[\cE_h(f^t, x_h, a_h)^2]$, where $\cE_h(f^t, x_h, a_h) = Q_{h,f^t}(x_h,a_h) - (\cT_h V_{h+1,f^t})(x_h,a_h)$. If we use the GOLF algorithm in \citet{jin2021bellman} to solve these Q-type problems, we can obtain a sequence of hypotheses and policies $\{f^t, \pi^t = \pi_{f^t}\}_{t \in [T]}$ satisfying two properties: 
\begin{itemize}
    \item [(i)] optimism \citep[][Lemma 40]{jin2021bellman}: $V^* \le V_{f^t}$ for any $t \in [T]$ with high probability; 
    \item [(ii)] small in-sample training error \citep[][Lemma 39]{jin2021bellman}: $\sum_{s=1}^{t-1} \EE_{\pi_{f^s}}[\cE_h(f^t, x_h, a_h)^2] \le \beta$ for any $(t, h) \in [T] \times [H]$, where $\beta$ only depends on $T$ logarithmically. 
\end{itemize} 
Based on these two properties, we directly establish an upper bound on $\reg(T)$ in terms of GEC: 
\$
\reg(T) &  = \sum_{t=1}^T V^* - V^{\pi_{f^t}} \le \sum_{t=1}^T \underbrace{V_{f^t} - V^{\pi_{f^t}}}_{ \displaystyle \text{\textcolor{red1}{prediction error}}}  \notag \\
 & \lesssim \biggl ( d \sum_{h=1}^H \sum_{t=1}^T \Big( \underbrace{ \sum_{s=1}^{t-1}  \EE_{\pi_{f^s}}[\cE_h(f^t, x_h, a_h)^2] }_{ \displaystyle \text{\textcolor{cyan}{training error }} {\color{cyan}\le \beta}} \Big)   \biggr )^{1/2} \le \sqrt{dHT\beta},
\$
where 
we let $d$ denote $\mathrm{GEC}(\cH, \ell, \Pi_{\mathrm{exp}}, \epsilon)$. 
Moreover, here 
the first inequality is implied by (i) optimism, the second inequality uses the definition of GEC (here we ignore burn-in terms in Definition~\ref{def:gec} because we can choose a small $\epsilon$ such as $1/(HT)$), and the last inequality uses (ii) small in-sample training error. With such a simple analysis, we obtain a $\sqrt{T}$-regret for Q-type problems in MDP as desired. But there are two issues to be addressed: First, if $d = \Omega(\sqrt{T})$, our regret bound becomes vacuous. To address this issue, we need to show that our new complexity measure can capture these Q-type problems. Equivalently, we need to show that all these problems have a low GEC. More generally, we need to answer the following question:
\begin{center} 
    {\bf Q1}: Can GEC capture a majority of known tractable interactive decision making?
\end{center}
The second issue is that the analysis for algorithms such as GOLF crucially relies on optimism ($V^* \le V_{f^t}$ for any $t \in [T]$ with high probability). But this condition may not hold for posterior sampling algorithms, which poses an additional challenge for establishing a generic posterior sampling algorithmic framework. Due to this challenge, the theoretical analysis for posterior sampling algorithms is rather less explored in the literature, despite of its computational superiority.  Hence, the following question also needs to be answered:
\begin{center}
    {\bf Q2}: Can we design a generic posterior sampling algorithm for interactive decision making with low GEC, and provide theoretical guarantees?
\end{center}

We provide the affirmative answers to {\bf Q1} and {\bf Q2} in this section (Sections~\ref{sec:relation:mdp} and \ref{sec:relation:pomdp}) and the following section (Section~\ref{sec:alg}), respectively. To show the generality of our complexity measure, GEC, we first present some examples of discrepancy function $\ell$, and exploration policy class $\Pi_{\exp}$.

\paragraph{The Choice of Discrepancy Function $\ell$.}

The choice of $\ell$ depends on the specific structure of the interactive decision making problem. We first discuss the choice of $\ell$ in Q-type and V-type problems in MDP, including but not limited to linear MDP \citep{yang2019sample,jin2020provably}, Q/V-type Bellman rank \citep{jiang@2017,jin2021bellman}, Q/V-type Bellman eluder dimension \citep{jin2021bellman}, and Q/V-type problems in bilinear class \citep{du2021bilinear}. Before continuing, for any $f \in \cH$ and state-action pair $(x,a)$, we define the Bellman residual as 
\# \label{eq:def:residual}
\cE_h(f,x,a) = Q_{h,f}(x,a) - (\cT_h V_{h+1,f})(x,a).
\#

\begin{example}[Q/V-type Problems in MDP] \label{example:ell:q}
    For Q/V-type problems in MDP, we choose $\xi_h = (x_h, a_h)$ and define  $$\ell_{f^s}(f^t, \xi_h) = \bigl (\cE_h(f^t, x_h, a_h) \bigr ) ^2, \quad \forall f^s \in \cH .$$ 
    
\end{example}
 While Q-type and V-type problems use the same discrepancy function, they are different in how we sample the action $a_h$ with different exploration policies, which will be introduced shortly.

\begin{example}[Model-based Problems in MDP] \label{example:ell:model}
    For model-based problems in MDP, we choose $\cH$ as in Example~\ref{example:hyp:model}, $\xi_h = (x_h, a_h)$, and define 
    $$\ell_{f^s}(f^t, \xi_h) = D_H^2 \big(\PP_{h, f^t}(\cdot \mid x_h, a_h), \PP_{h, f^*}(\cdot \mid x_h, a_h) \big), \quad \forall f^s \in \cH . $$
\end{example}

In above examples, $\ell_{f^s}(f^t, \xi_h)$ does not depend on $f^s$. But in some examples such as linear mixture MDP \citep{ayoub2020model,cai2020provably,zhou2021nearly}, $\ell_{f^s}(f^t, \xi_h)$ depends on both $f^t$ and $f^s$. See Appendix~\ref{appendix:completeness} for details.

\begin{example}[PSR] \label{example:ell:psr}
    For PSR, we choose the model-based $\cH$ as in Example~\ref{example:hyp:model}, $\xi_h = \tau_H$, which is the full history, and $\ell_{f^s}(f^t, \xi_h) = \frac{1}{2} (\sqrt{\PP_{f^t}(\tau_H)/\PP_{f^*}(\tau_H)} - 1)^2$. Then we have for any policy $\pi$,   
    $$\EE_{\pi}  \bigl [ \ell_{f^s}(f^t, \xi_h)\bigr ]  =  D_H^2 \big(\PP^{\pi}_{f^t}(\tau_H), \PP^{\pi}_{f^*}(\tau_H) \big),$$
    where $\PP_{f^t}^{\pi}(\tau_H)$ and $\PP_{f^*}^{\pi}(\tau_H)$ are specified in \eqref{eq:def:prob1}.
\end{example}

\paragraph{The Choice of Exploration Policy class $\Pi_{\exp}$.} The choice of exploration policy class also depends on the structure of  interactive decision making, and we summarize several typical choices as follows.

\begin{example}[Q-type Problems in MDP] \label{example:policy:q}
    For Q-type problems in MDP, we choose $\pi_{\exp}(f, h) = \pi_f$ for any $h \in [H]$.
\end{example}
For Q-type problems, the distribution of $a_h$ matches the historical one, i.e., $a_h \sim \pi_{f}$. On the contrary, for V-type problems, the action is greedy in the function evaluated in the future, which may not match the historical one. The latter requires a sufficient exploration of the whole action space to deal with the mismatch between the historical action and future choice required for the evaluation.

\begin{example}[V-type Problems in MDP] \label{example:policy:v}
    For V-type problems in MDP, we choose $\pi_{\exp}(f, h) = \pi_f \circ_{h} \operatorname{Unif}(\mathcal{A})$ for any $h \in [H]$, which means executing $\pi_f$ for the first $h-1$ steps, taking an action which is uniformly sampled from $\cA$ as step $h$, and behaving arbitrarily afterwards. This allows   sufficient exploration over the action space to deal with the mismatch between the historical actions and the actions used to evaluate $\ell_{f'}(f, \xi_h)$. 
\end{example}

\begin{example}[PSR] \label{example:policy:psr}
    For PSRs, we choose $\pi_{\exp}(f, h) = \pi_{f} \circ_{h} \operatorname{Unif}(\mathcal{A}) \circ_{h+1} \operatorname{Unif}\left(\mathcal{U}_{A, h+1}\right)$ for any $0 \le h \le H - 1$. Here $\pi_{\exp}(f, h)$ means executing $\pi_f$ for the first $h-1$ steps, taking an action which is uniformly sampled from $\cA$ at step $h$, taking an action sequence which is uniformly sampled from $\cU_{A, h+1}$, and behaving arbitrarily afterwards.
\end{example}

Finally, before showing that GEC generalizes various complexity measures studied in the literature, we conclude this subsection with a discussion on  the online decoupling coefficient introduced in \citet{dann2021provably}, which empowers provably sample-efficient and model-free  posterior sampling on Q-type MDP models.

    \begin{definition}[Online Decoupling Coefficient Presented in \citet{dann2021provably}]
        Given an MDP $\cM$, a value-based function class $\cF$ (cf. Example~\ref{example:hyp:value:1}), and $\mu > 0$, decoupling coefficient $\dca (\cF, \cM, T, \mu)$ is the smallest number $K$ such that for any sequence $\{f^t\}_{t=1}^T$,
        \$
            \sum_{h=1}^{H} \sum_{t=1}^{T}\left[\mathbb{E}_{\pi_{f^{t}}}\left[\mathcal{E}_{h}\left(f^{t} , x_{h}, a_{h}\right)\right]\right] \leq \mu \sum_{h=1}^{H} \sum_{t=1}^{T}\bigg[\sum_{s=1}^{t-1} \mathbb{E}_{\pi_{f^{s}}}\left[\mathcal{E}_{h}\left(f^{t} , x_{h}, a_{h}\right)^{2}\right]\bigg]+\frac{K}{4 \mu}.
        \$
    \end{definition}

    In comparison, our definition does not introduce an extra tuning parameter $\mu$ and allows more general discrepancy functions beyond the (Q-type) squared Bellman residuals, probability measures (beyond those induced by greedy policies), and trajectory collection policy. These extensions allow us to capture much more general interactive decision making problems.

\subsection{Relationship with Known Tractable Interactive Decision Making with Full Observations} \label{sec:relation:mdp}

In this subsection, we prove that  our new complexity measure --- generalized eluder coefficient ---  captures a majority of known complexity measures  studied for  tractable  interactive decision making in the fully observable setting. 

For fully interactive observable decision making, there are two generic  complexity measures in  parallel --- Bellman eluder (BE) dimension \citep{jin2021bellman} and bilinear class \citep{du2021bilinear}. As shown in \citet{jin2021bellman,du2021bilinear}, RL with low BE dimension and/or bilinear class contain a rich class of tractable RL problems as special cases, including but not limited to tabular MDPs \citep{azar2017minimax,jin2018q}, linear MDPs \citep{yang2019sample,jin2020provably}, generalized linear MDPs \citep{wang2019optimism}, linear mixture MDPs \citep{ayoub2020model,modi2020sample,cai2020provably,zhou2021nearly}, kernelized nonlinear regulator \citep{kakade2020information}, RL with Bellman completeness \citep{munos2003error,zanette2020learning}, RL with bounded eluder dimension \citep{osband2014model, wang2020reinforcement}, RL with low Bellman rank \citep{jiang@2017}, and RL with low witness rank \citep{sun2019model}. 

Based on this, we show our new complexity measure, GEC,  can naturally capture all these known tractable RL problems by proving our new complexity measure can capture both BE dimension and bilinear class, that is, establishing upper bounds of GEC in terms of BE dimension or the information gain of bilinear class.
Furthermore, to better understand the model-based RL, we also show our complexity measure can capture the model-based framework witness rank \citep{sun2019model} with a Hellinger-distance-based loss function, which is different from the extensions made in bilinear class. 

\subsubsection{Bellman Eluder Dimension} 
The definition of BE dimension relies on the notion of the distributional eluder (DE) dimension.

\begin{definition}[$\epsilon$-independence between distributions] Let $\cG$ be a function class defined on $\cX$, and $\nu, \mu_1,\cdots,\mu_n$ be probability measures over $\cX$. We say $\nu$ is $\epsilon$-independent of $\{\mu_1,\mu_2,\cdots, \mu_n\}$ with respect to $\cG$ if there exists $g \in \cG$ such that $\sqrt{\sum_{i=1}^n (\EE_{\mu_i} [g])^2} \leq \epsilon$ but $|\EE_\nu [g]| > \epsilon$.
\end{definition}

\begin{definition}[Distributional Eluder (DE) dimension] Let $\cG$ be a function class defined on $\cX$, and $\Pi$ be a family of probability measures over $\cX$. The distributional eluder dimension $\de(\cG, \Pi, \epsilon)$ is the length of the longest sequence $\{\rho_1,\cdots,\rho_n\} \subset \Pi$ such that there exists $\epsilon' \geq \epsilon$ with $\rho_i$  being  $\epsilon'$-independent of $\{\rho_1,\cdots,\rho_{i-1}\}$ for all $i \in [n]$.
\end{definition}

Choosing $\cH$ as the value-based hypothesis class in Example \ref{example:hyp:value:1} and recalling the definition of the Bellman operator $\cT_h$ in \eqref{eq:Bellman:equation}, we define BE dimension as follows.

\begin{definition}[Q-type BE dimension] Let $(I-\cT_h) \cH := \{(x,a) \to (f_h - \cT_h f_{h+1})(x,a): f \in \cH\}$  be the set of Bellman residuals induced by $\cH$ at step $h$, and $\Pi = \{\Pi_h\}_{h=1}^H$ be a collection of $H$ probability measure families over $\cS \times \cA$. The Q-type $\epsilon$-Bellman eluder dimension of $\cH$ with respect to $\Pi$ is defined as
    $$\bde(\cH, \Pi, \epsilon) := \max_{h \in [H]} \bigl \{  \de\left((I-\cT_h)\cH, \Pi_h, \epsilon\right) \big\} .$$
\end{definition}

\begin{definition}[V-type BE dimension] Let $(I-\cT_h) V_{\cH} := \{x \to (f_h - \cT_h f_{h+1})(x, \pi_{f_h}(x)): f \in \cH\}$, and $\Pi = \{\Pi_h\}_{h=1}^H$ be a collection of $H$ probability measure families over $\cS$. The V-type $\epsilon$-Bellman eluder of $\cH$ with respect to $\Pi$ is defined as
    $$\mathrm{dim_{VBE}}(\cH, \Pi, \epsilon) := \max_{h \in [H]} \bigl \{ \de\left((I-\cT_h)V_{\cH}, \Pi_h, \epsilon\right) \bigr \}.$$
\end{definition}

The following lemma shows that problems with low BE dimension also have a low generalized eluder coefficient.  

\begin{lemma}[BE Dimension $\subset$ GEC] \label{lem:reduction_bed_dc} Suppose that a problem has a Q-type BE dimension of $d_Q = \bde(\cH, \cD_\cH, \epsilon)$ or a V-type BE dimension of $d_V = \mathrm{dim_{VBE}}(\cH, \cD_\cH, \epsilon)$ with $\cD_\cH$ as the distributions induced by following some $f \in \cH$ greedily. Then, we have
     \$
            \sum_{t=1}^T V_{f^t} - V^{\pi_{f^t}} & \leq \bigg(d_{\GEC}  \sum_{t=1}^T \sum_{h=1}^H \sum_{s=1}^{t-1} \EE_{\pi_{\exp}({f^s},h)} \Big[ (\cE_h(f^t, x_h, a_h))^2\Big] \bigg) ^{1/2} + \min \{d_{\GEC}, HT\} + HT \epsilon,            
    \$
    where we have $\pi_{\exp}({f^s},h) = \pi_{f^s}$  and $d_{\GEC} \leq 2 d_QH \cdot \log (T)$ for Q-type BE dimension and  $\pi_{\exp}({f^s},h) = \pi_{f^s} \circ_{h} \mathrm{Unif}(\cA)$ and $d_{\GEC} \leq 2 d_VAH \cdot \log (T)$ for V-type BE dimension. 
    This further implies that
    \$
      & \GEC(\cH, |\EE_{x_{h+1} | x_h, a_h} l_{f'}(f, \zeta_h)|^2, \Pi_{\exp}, \epsilon) \leq 2 d_Q H  \cdot \log(T), \quad \text{(Q-type)},\\
      & \GEC(\cH, |\EE_{x_{h+1} | x_h, a_h} l_{f'}(f, \zeta_h)|^2, \Pi_{\exp}, \epsilon) \leq 2 d_V AH  \cdot \log(T) , \quad \text{(V-type)}.
    \$
    where $\zeta_h = (x_h, a_h, r_h, x_{h+1})$ and $l_{f'}(f, \zeta_h)$ is the Bellman error defined as
    \# \label{eq:def:loss:BE}
   l_{f'}(f, \zeta_h) =  Q_{h, f}(x_h, a_h) - r_h - V_{h+1, f} 
   (x_{h+1})  .
   \#
\end{lemma}
\begin{proof}
    See Appendix \ref{appendix:BE} for a detailed proof.
\end{proof}

In Lemma~\ref{lem:reduction_bed_dc}, it holds that $|\EE_{x_{h+1} | x_h, a_h} l_{f'}(f, \zeta_h)|^2 = (\cE_h(f, x_h, a_h))^2 = \ell_{f'}(f, (x_h, a_h))$,  where $\ell$ is the  discrepancy function defined in Example~\ref{example:ell:q}. Here we introduce the notion of Bellman error $l_{f'}(f, \zeta_h)$ is to facilitate our following discussion of algorithm design.

\begin{remark}
\citet{jin2021bellman} also consider $\Pi=\cD_\Delta$, the collections of probability measures that put measure $1$ on a single state-action pair. The reduction to this measure family follows from the same line of the above lemma but we suffer from an additional $H\sqrt{T \log T}$ burn-in cost when converting $\sum_{t=1}^T V_{f^t} - V^{\pi_{f^t}}$ to $\sum_{t=1}^T \sum_{h=1}^H (Q_{h,f^t}-\cT_h V_{h+1,f^t})(x_h^t,a_h^t)$ by the Azuma-Hoeffding inequality \citep{azuma1967weighted} where $(x_h^t,a_h^t)$ is actually experienced state-action pair in the history.
\end{remark}

\subsubsection{Bilinear Class}
With a slight abuse of notation, we introduce the bilinear class in \citet{du2021bilinear}.
\begin{definition}[Bilinear Class] \label{def:bilinear}
    Given an MDP, a hypothesis class $\cH$, and a discrepancy function $l = \{l_f:\cH \times (\cS \times \cA \times \RR \times \cS) \times \cH \to \RR\}_{f \in \cH}$, we say the RL problem is in a bilinear class if there exist functions $W_h:\cH \to \cV$ and $X_h:\cH \to \cV$ for some Hilbert space $\cV$, such that for all $f \in \cH$ and $h \in [H]$, we have
    \begin{equation}
        \begin{aligned}
            \left|\mathbb{E}_{\pi_{f}}\left[Q_{h,f}\left(x_{h}, a_{h}\right)-r\left(x_{h}, a_{h}\right)-V_{h+1,f}\left(x_{h+1}\right)\right]\right| \leq\left|\left\langle W_{h}(f)-W_{h}\left(f^{*}\right), X_{h}(f)\right\rangle\right|, \\
            \left|\EE_{x_h \sim \pi_f, a_h \sim \tilde{\pi}}\left[l_f\left(g, \zeta_h\right)\right]\right|=\left|\left\langle W_{h}(g)-W_{h}\left(f^{*}\right), X_{h}(f)\right\rangle\right|, 
        \end{aligned}
    \end{equation}
    where $\tilde{\pi}$ is either $\pi_f$ (Q-type) or $\pi_g$ (V-type). 
    Moreover, it is required that $\sup_{f\in \cH, h \in [H]} \norm{W_h(f)}_2 \leq 1$ and $\sup_{f\in \cH, h \in [H]} \norm{X_h(f)}_2 \leq 1$.  We also define $\cX_h = \{X_h(f): f \in \cH\}$ and $\cX = \{\cX_h: h \in [H]\}$. The complexity of a bilinear class is characterized by the information gain. 
\end{definition}

\begin{definition}[Information Gain] \label{def:info:gain}
Let $\cX \subset \cV$ where $\cV$ is a Hilbert space. For $\epsilon > 0$ and integer $T > 0$, the maximum information gain $\gamma_T(\epsilon, \cX)$ is defined as
    \begin{align*}
        \gamma_T(\epsilon, \cX) = \max_{x_1,\cdots,x_{T} \in \cX } \log \det\biggl(I+\frac{1}{\epsilon} \sum_{t=1}^{T} x_tx_t^\top\biggr).
    \end{align*}
    For $\cX = \{\cX_h: h \in [H]\}$, we use the notation $\gamma_T(\epsilon, \cX) = \sum_{h \in [H]} \gamma_T(\epsilon, \cX_h)$.
\end{definition}

\begin{lemma}[Bilinear Class $\subset$ GEC] \label{lemma:relationship:bilinear}
    For a bilinear class with hypothesis class $\cH$, discrepancy function $l$, and information gain $\gamma_T(\epsilon,\cX)$, we have
    $$
        \begin{aligned}
             & \sum_{t=1}^T V_{f^t} - V^{\pi_{f^t}} \leq \Big[d_\GEC \sum_{t=1}^T \sum_{h=1}^H \sum_{s=1}^{t-1} \EE_{{\pi}_{\exp}(f^s, h)} |\EE_{x_{h+1}|x_h,a_h} l_{f'}(f,\zeta_h)|^2\Big]^{1/2} \notag + 2 \min\left\{d_\GEC, HT\right\} + HT\epsilon,       
        \end{aligned}
    $$
    where we have $\pi_{\exp}({f^s},h) = \pi_{f^s}$ and $d_\GEC \leq 2 \gamma_T(\epsilon, \cX)$ for Q-type problems and $\pi_{\exp}({f^s},h) = \pi_{f^s} \circ_{h} \mathrm{Unif}(\cA)$ and $d_\GEC \leq 2A \cdot \gamma_T (\epsilon,\cX)$ for V-type problems. In other words, for $\zeta_h = (x_h, a_h, r_h, x_{h+1})$, we have
    \$
      & \GEC(\cH, |\EE_{x_{h+1}|x_h,a_h} l_{f'}(f,\zeta_h)|^2, \Pi_{\exp}, \epsilon) \leq 2\gamma_T(\epsilon,\cX),  \quad \text{(Q-type)},\\
      & \GEC(\cH, |\EE_{x_{h+1}|x_h,a_h} l_{f'}(f,\zeta_h)|^2, \Pi_{\exp}, \epsilon) \leq 2A \cdot \gamma_T(\epsilon,\cX), \quad \text{(V-type)}.
    \$
\end{lemma}

\begin{proof}
    See Appendix \ref{appendix:bilinear} for a detailed proof.
\end{proof}

\subsubsection{Witness Rank} 
We introduce the witness rank as an example of model-based RL with function approximation. \citet{sun2019model} study  the case where $\cH_h$ is the hypothesis class containing $\PP_h$ (we assume that the reward function is known for simplicity). For  witness rank, we adopt  a discriminator class $\cV = \{\cV_h: \cS \times \cA \times \cS \to [0,1]\}_{h \in [H]}$. The original version of witness rank  in \citet{sun2019model} is a V-type one and we also consider the Q-type witness rank. 

\begin{definition}[Q-type/V-type Witness Rank] \label{def:witness}
    We say an MDP has witness rank $d$ if given two models $f, g \in \cH$, there exists $X_h: \cH \to \RR^d$ and $W_h:\cH \to \RR^d$ such that 
$$
\begin{aligned}
     &\max_{v \in \cV_h} \EE_{x_h \sim \pi_f, a_h \sim \tilde{\pi}}[\EE_{x' \sim \PP_{h, g}(\cdot|x_h,a_h)} v(x_h,a_h, x') - \EE_{x' \sim \PP_{h, f^*}(\cdot|x_h,a_h)} v(x_h,a_h, x')] \geq \inner{W_h(g)}{X_h(f)}\\
     & \kappa_{\mathrm{wit}} \cdot \EE_{x_h \sim \pi_f, a_h \sim \tilde{\pi}}[\EE_{x' \sim \PP_{h, g}(\cdot|x_h,a_h)} V_{h+1,g}(x') - \EE_{x' \sim \PP_{h, f^*}(\cdot|x_h,a_h)} V_{h+1, g}(x')] \leq \inner{W_h(g)}{X_h(f)},
\end{aligned}
$$
where $\tilde{\pi}$ is either $\pi_f$ (Q-type) or $\pi_g$ (V-type), and $\kappa_{\mathrm{wit}} \in (0, 1]$ is a constant. Moreover, we assume that $\sup_{f\in \cH, h \in [H]} \norm{W_h(f)}_2 \leq 1$ and $\sup_{f\in \cH, h \in [H]} \norm{X_h(f)}_2 \leq 1$.
\end{definition}

Then we have the following lemma, which shows that both Q-type and V-type witness rank can be captured by $\GEC(\cH, \ell, \Pi_{\exp}, \epsilon)$, where we choose the model-based hypothesis class $\cH$ as in Example~\ref{example:hyp:model}, Hellinger-distance-based discrepancy function as in  Example~\ref{example:ell:model}, and exploration policy class $\Pi_{\exp}$ as in Example~\ref{example:policy:q} (Q-type) or Example~\ref{example:policy:v} (V-type).

\begin{lemma}[Witness Rank $\subset \GEC$] \label{lem:witness_rank} For a RL problem with a Q-type witness rank $d_Q$ or a V-type one $d_V$, we have
        $$
        \begin{aligned}
             \sum_{t=1}^T V_{f^t} - V^{\pi_{f^t}}
             &\leq \bigg(d_{\GEC} \sum_{t=1}^T \sum_{h=1}^H \sum_{s=1}^{t-1} \EE_{{\pi}_{\exp}(f^s, h)} \Big[   D_H^2 \big(\PP_{h,f^t}(\cdot \mid x_h,a_h), \PP_{h,f^*}(\cdot \mid x_h,a_h) \big)  \Big] \bigg)^{1/2} \\
             &\qquad + 2 \min\left\{d_{\GEC}, HT\right\} + HT\epsilon,
        \end{aligned}
    $$
where we have $\pi_{\exp}(f^s,h) = \pi_{f^s}$ and $d_{\GEC} \leq 4d_QH \cdot \log(1+\frac{T}{\epsilon \kappa_{\mathrm{wit}}^2})/\kappa_{\mathrm{wit}}^2$ for Q-type witness rank, and we have $\pi_{\exp(f^s,h)} = \pi_{f^s} \circ_{h} \operatorname{Unif}(\mathcal{A})$ and $d_{\GEC} \leq 4d_VAH \cdot \log(1+\frac{T}{\epsilon \kappa_{\mathrm{wit}}^2})/\kappa_{\mathrm{wit}}^2$ for V-type witness rank. In other words, we have
\$
& \GEC(\cH, \ell, \Pi_{\exp}, \epsilon) \le 4d_QH \cdot \log\Big (1+\frac{T}{\epsilon \kappa_{\mathrm{wit}}^2} \Big )/\kappa_{\mathrm{wit}}^2 \quad \text{(Q-type)}, \\
&  \GEC(\cH, \ell, \Pi_{\exp}, \epsilon) \le 4d_VAH \cdot \log \Big (1+\frac{T}{\epsilon \kappa_{\mathrm{wit}}^2} \Big)/\kappa_{\mathrm{wit}}^2 \quad \text{(V-type)}.
\$
\end{lemma}

\begin{proof}
    See Appendix \ref{appendix:witness:rank} for a detailed proof.
\end{proof}

\subsection{Relationship with Known Tractable Interactive Decision Making with Partial Observations}\label{sec:relation:pomdp}

We first identify a new class of PSRs --- generalized regular PSR --- that permits  low GEC. Then we show generalized regular PSR is a rich class of models in the sense that it includes several known tractable RL problems with partial observations, such as weakly revealing POMDPs \citep{jin2020sample,liu2022partially}, latent MDPs \citep{kwon2021rl}, decodable POMDPs \citep{du2019provably,efroni2022provable}, low-rank POMDPs \citep{wang2022embed}, and regular PSR \citep{zhan2022pac} as special cases. Hence, our complexity measure GEC can naturally capture all these models. Finally, we prove that our complexity measure GEC can also capture the recent PO-bilinear class model \citep{uehara2022provably}.

\subsubsection{Generalized Regular PSRs}

The recent work of \citet{zhan2022pac} shows that regular PSR (cf. Definition \ref{def:regular:psr}) is a tractable subclass of PSR. However, regular PSR cannot incorporate some important known tractable POMDPs, such as decodable POMDPs (cf. Definition~\ref{def:decodable}), as special cases. To include these known tractable POMDPs, we propose the following \emph{generalized regular PSR} model.

\begin{definition}[Generalized Regular PSR] \label{def:general:regular:psr}
    We say a PSR is $\alpha$-generalized regular if the following two conditions hold:
    \begin{itemize}
        \item[1.] For any $h \in [H]$ and $\bfx \in \RR^{|\cU_h|}$, it holds that
              \$
              \max_{\pi} \sum_{\tau_{h:H}} | \bfm(o_{h: H}, a_{h:H}) \bfx | \cdot  \pi(o_{h: H}, a_{h:H}) \le \frac{\| \bfx\|_1}{\alpha},
              \$
              where $\tau_{h:H} = (o_{h:H}, a_{h:H}) \in (\cO \times \cA)^{H - h + 1}$.
        \item[2.]  For any $h \in [H - 1]$ and $\bfx \in \RR^{|\cU_h|}$, it holds that
              \$
              \max_{\pi} \sum_{(o_h, a_h) \in \cO \times  \cA }  \| \Mb_h(o_h, a_h) \bfx \|_1 \cdot  \pi(o_h, a_h) \le \frac{|\cU_{A, h + 1}|}{\alpha} \|\bfx\|_1.
              \$
    \end{itemize}
    For the non-tabular case, we replace summation with integral.
\end{definition}

Intuitively, regularity conditions in Definition~\ref{def:regular:psr} control the propagated estimation errors. Specifically, we choose $\bfx = \bfx_1 - \bfx_2$ where $\bfx_1, \bfx_2 \in \RR^{|\cU_h|}$, and the above regularity conditions guarantee that the propagated errors 
$$ \max_{\pi} \sum_{\tau_{h:H}} | \bfm(o_{h: H}, a_{h:H}) (\bfx_1 - \bfx_2) | \cdot \pi(o_{h: H}, a_{h:H})\quad \text{and}\quad  \max_{\pi} \sum_{(o_h, a_h) \in \cO \times   \cA }  \| \Mb_h(o_h, a_h) (\bfx_1 - \bfx_2) \|_1 \cdot \pi(o_h, a_h)$$ are controlled by the parameter $\alpha$.
Similar properties are observed by previous work on weakly revealing POMDP \citep{liu2022partially} and regular PSR \citep{zhan2022pac}. But neither shows that any PSR satisfying these two conditions can be efficiently solved. Furthermore, we will show that generalized regular PSR  subsumes not only weakly revealing POMDP and regular PSR,  but also other known tractable partially observable interactive decision making problems as special cases. See Section~\ref{sec:example:general:psr} for details.

In the following lemma, we prove that each generalized regular PSR has a low GEC.

\begin{lemma}[Generalized Regular PSR $\subset$ PSR with Low GEC] \label{lemma:general:regular:psr}
    For any $\alpha$-generalized regular PSR, it holds for any $\{f^t, \pi^t = \pi_{f^t}\}_{t \in [T]}$ that
    \$
    & \sum_{t = 1}^T V_{f^t} - V^{\pi^t}  \le \bigg[d_{\GEC}  \sum_{t = 1}^{T} \sum_{h = 0}^{H - 1}\sum_{s = 1}^{t - 1} D_H^2 \Big( \PP_{f^t}^{\pi_{\exp}(f^s, h)}, \PP_{f^*}^{\pi_{\exp}(f^s, h)} \Big)\bigg]^{1/2} + \sqrt{d_{\GEC} H T },
    \$
    where $d_{\GEC} = \cO\Big(\frac{d_{\psr}  A^3 U_A^4 H  \cdot \iota}{\alpha^4} \Big)$ and $\iota = 2 \log\Big(1 + \frac{4 d_{\psr} A^2U_A^2  \delta_{\psr}^2 T }{ \alpha^4}\Big)$ with $\delta_{\psr}$ defined as 
    \# \label{eq:def:delta}
    \delta_{\psr} = \max_{h \in [H]} \min \Big\{ \|\KK_h\|_1 \cdot \|\VV_h\|_1 \,\big|\, \bar{\DD}_h = \KK_h \VV_h, \KK_h \in 
   \RR^{|\cU_{h+1}| \times d_{\psr,h}}, \VV_h \in \RR^{ 
   d_{\psr, h} \times |\cO|^{h}|\cA|^h } \Big\},
   \#
    where $\bar{\DD}_h$ is defined in \eqref{eq:def:bar:D}. In other words, we have
    \$
    \GEC(\cH, \ell, \Pi_{\exp}) \le \cO\Big(\frac{d_{\psr} A^3 U_A^4 H \cdot \iota}{\alpha^4} \Big),
    \$
    where $\cH, \ell$, and $\Pi_{\exp}$ are specified as in Examples~\ref{example:hyp:model}, \ref{example:ell:psr}, and \ref{example:policy:psr}, respectively.
\end{lemma}

\begin{proof}
    See Appendix \ref{appendix:general:psr} for a detailed proof.
\end{proof}

\subsubsection{Examples of Generalized Regular PSRs} \label{sec:example:general:psr}

Now we present several known tractable partially observable RL problems are special cases of generalized regular PSRs. 

\paragraph{Weakly Revealing POMDPs.}

To identify a class of tractable  POMDPs, \citet{jin2020sample,liu2022partially} impose the following weakly revealing conditions. For the undercomplete setting where $S \le O$, the weakly revealing condition states that the $S$-th singular value of emission function $\OO_h$ is lower bounded.

\begin{definition}[Weakly Revealing POMDP \citep{jin2020sample}] \label{def:weak:reveal:1}
    We say a POMDP is an $\alpha$-weakly revealing POMDP if there exists $\alpha > 0$, such that $\min_{h \in [H]} \sigma_S(\OO_h) \ge \alpha$.
\end{definition}

For the overcomplete case ($S \ge O$), we define the $m$-step emission matrices $\MM = \{ \MM_h \in \RR^{A^{m-1}O^m \times S}\}_{h \in [H - m + 1]}$, where the  entries of $\MM_h$ are defined by letting 
$$[\MM_h]_{(\bfa, \bfo), s} = \PP(o_{h: h + m - 1} = \bfo \mid s_h = s, a_{h:h+m-2} = \bfa)$$ for any $(\bfa, \bfo, s) \in \cA^{m-1} \times \cO^{m} \times \cS$. Similar to the undercomplete case, we assume the $S$-th singular value of the $m$-step emission matrix $\MM_h$ is lower bounded.

\begin{definition}[Multi-step Weakly Revealing POMDPs \citep{liu2022partially}] \label{def:weak:reveal:2}
    We say a POMDP is an $m$-step weakly-revealing POMDP if there exists $\alpha > 0$, such that $\min_{h \in [H -m + 1]} \sigma_S(\MM_h) \ge \alpha$.
\end{definition}

For the (multi-step) weakly revealing POMDP defined in Definitions \ref{def:weak:reveal:1} and \ref{def:weak:reveal:2}, we have the following lemma which states that weakly revealing POMDP is a special class of generalized PSR.

\begin{lemma}[Weakly Revealing POMDP $\subset$ Generalized Regular PSR] \label{lemma:weak:reveal}
    Any ($m$-step) $\alpha$-weakly revealing POMDP is an $(\alpha/\sqrt{S})$-generalized regular PSR.
\end{lemma}

\begin{proof}
    See Appendix~\ref{appendix:weakly:revealing} for a detailed proof.
\end{proof}

\paragraph{Latent MDPs.} To identify a class of tractable latent MDPs, \cite{kwon2021rl} impose a full-rank test condition. \cite{zhan2022pac} show that a latent MDP with a full-rank test set can be viewed as a weakly revealing POMDP. In the following, we show that a latent MDP with a full-rank test set is also a generalized regular PSR. 
First, we introduce the definition of the  latent MDP.
\begin{definition}[Latent MDP \citep{kwon2021rl}]\label{def:lt}
Let $\cM=\{\cM_m\}_{m=1}^M$ be a set of MDPs with a joint state space $\cS$, a joint action space $\cA $ and a joint reward $R=\{R_h\}_{h=1}^H$ in a finite horizon $H$. We define $S=\vert\cS\vert$ and $A=\vert\cA\vert$. Each MDP $\cM_m\in\cM$ is a tuple $(\cS,\cA,\TT_m, R,\nu_m,\omega_m)$ where $\TT_m=\{\TT_{h,m}\}_{h=1}^H$, $\TT_{h,m}:\cS\times\cA\times\cS \to[0,1]$ is a transition probability measure that maps a state-action pair and a next state to a probability, and $\nu_m$ is an initial state distribution. Let $\omega_1,\cdots,\omega_M$ be the mixing weights of the latent MDP such that at the start of every episode, one MDP $\cM_m\in\cM$ is randomly chosen with probability $\omega_m$.
\end{definition}
Theorem 3.1 of \cite{kwon2021rl} shows that extra assumptions are needed to enable efficient exploration in a latent MDP. We introduce the assumptions on a latent MDP that enable efficient exploration as follows.

\begin{assumption}[$\alpha$-Full-rank  Test Set \citep{kwon2021rl}] \label{asp:lmgr}
    For a sequence of states and actions $t_{h}=(a_{h:h+W-1},s_{h+1:h+W})$, we define $\PP_{h,m}(t_{h}\mid s_h)=\PP_{h,m}(s_{h+1:h+W}\mid s_h,a_{h:h+W-1})$. For any $h\in[H]$ and a test set $\cU_{h}$, we define the matrix $L_h(s,\cU_{h})\in \RR^{\vert \cU_{h}\vert\times M}$ by $[L_h(s,\cU_{h})]_{i,j}=\PP_j(t_{h,i} \mid s)$. We say a test set $\{\cU_{h}\}_{h=1}^{H}$ is   $\alpha$-full-rank   if $\sigma_M(L_h(s,\cU_{h}))\geq \alpha$ for all $s\in\cS$ and $h\in[H]$ with some $\alpha>0$. We assume that $\cM$ has an $\alpha$-full-rank test set $\{\cU_{h}\}_{h=1}^{H}$.
\end{assumption}
Here we do not require that the tests in the test set have the same length, which is more general than the case studied in \cite{zhan2022pac}. The following lemma shows that latent MDPs that satisfy Assumption \ref{asp:lmgr} are also generalized regular PSRs.

\begin{lemma}[Latent MDP $\subset$ Generalized Regular PSR] \label{lemma:latent:mdp}
    A latent MDP $\cM$ is an $(\alpha/\sqrt{SM})$-generalized regular PSR if it satisfies Assumption \ref{asp:lmgr}.

\end{lemma}

\begin{proof}
    See Appendix~\ref{appendix:latent:mdp} for a detailed proof.
\end{proof}

\paragraph{Decodable POMDPs.}

\citet{efroni2022provable} propose the $m$-step decodable POMDP, where a suffix of the most recent length-$m$ history contains sufficient information to decode the hidden state. When $m = 1$, the $1$-step decodable POMDP is also known as the block MDP \citep{du2019provably}. For simplicity, we denote $m(h) = \min\{h - m + 1, 1\}$.
\begin{definition}[Decodable POMDP \citep{efroni2022provable}] \label{def:decodable}
    We say a POMDP is $m$-step decodable if there exists an unknown decoder $\phi^* = \{\phi_h^*\}_{h \in [H]}$ such that for any reachable trajectory $\tau = (s_{1:H}, o_{1:H}, a_{1:H})$, we have $s_h = \phi_h^*(z_h)$ for all $h \in [H]$, where $z_h = (o_{m(h):h-1},a_{m(h):h-1}, o_h)$.
\end{definition}

\begin{lemma}[Decodable POMDP $\subset$ Generalized Regular PSR] \label{lemma:decodable}
    Any ($m$-step) decodable POMDP is a $1$-generalized regular PSR.
\end{lemma}

\begin{proof}
    See Appendix \ref{appendix:decodable} for a detailed proof.
\end{proof}

\paragraph{Low-rank POMDPs.} We present the low-rank POMDP class proposed by \cite{wang2022embed} and then show that it can be subsumed by our framework.
\begin{definition}[Low-rank POMDP \citep{wang2022embed}] We assume that the transition kernel $\PP_h$ takes the following low-rank form for all $h\in[H]$,
    \$
    \PP_h(s_{h+1}\mid s_h,a_h)=\psi_h(s_{h+1})^\top \phi_h(s_h,a_h),
    \$
    where $\psi_h:\cS \rightarrow\RR_+^d$, $\phi_h:\cS\times\cA\rightarrow\Delta_{[d]}$ are unknown features. We assume that $\Vert\int \psi_h(s_{h+1})\ud s_{h+1}\Vert_{\infty}\leq q$ for all $h\in[H]$.
\end{definition}
To enable sample-efficient exploration, \cite{wang2022embed} introduce the following assumption.
\begin{assumption}[Future Sufficiency, Assumption 3.5 in \cite{wang2022embed}]\label{asp:fs35}
    We define the sequence of interactions from the $h$-th step to the $(h+m)$-th step $\tau_h^{h+m}$ as 
    \#
    \tau_h^{h+m}=(o_h,a_h,\cdots,o_{h+ m -1},a_{h+m-1},o_{h+m})=(a_{h:h+m-1},o_{h: h+m}).\nonumber
    \#
    For all $h \in [H]$, we define the following forward emission operator $\MM_h:L^1(\cS)\rightarrow L^1(\cA^m\times \cO^{m+1})$,
    \$
    (\MM_h \bfx)(\tau_h^{h+m})=\int_\cS \PP(\tau_h^{h+m}\mid s_h)\cdot \bfx(s_h)\ud s_h,\quad \forall \bfx\in L^1(\cS), \quad \forall \tau_h^{h+m}\in \cA^m\times \cO^{m+1}.
    \$
    We define the mapping $g_h:\cA^m\times \cO^{m+1}\to\RR^d$ for all $h\in[H]$ as follows,
    \$
    g_h=\bigl[\MM_h[\psi_{h-1}]_1,\cdots,\MM_h[\psi_{h-1}]_d\bigr]^\top,
    \$
    where we denote by $[\psi_{h-1}]_i$ the $i$-th entry of the mapping $\psi_{h-1}$ for all $i\in[d]$. We assume that for some $m>0$ that the matrix
    \$
    M_h=\int_{\cA^m\times \cO^{m+1}}g_h(\tau_h^{h+m})g_h(\tau_h^{h+m})^\top \ud \tau_h^{h+m}\in\RR^{d \times d} 
    \$
    is invertible. We denote by $M_h^{\dagger}$ the inverse of $M_h$ for all $f \in \cH $ and $h\in[H]$. We define the linear operator $\MM_h^{\dagger}:L^1(\cA^m\times\cO^{m+1})\rightarrow L^1(\cS)$ for all $h\in[H]$ as follows,
    \$
    (\MM_h^{\dagger} \bfx)(s_h)=\int_{\cA^m\times \cO^{m+1}}\psi_{h-1}(s_h)^\top M_h^{\dagger}g_h(\tau_h^{h+m})\cdot \bfx(\tau_h^{h+m})\ud \tau_h^{h+m},
    \$
    where $\bfx\in L^1(\cA^m\times\cO^{m+1})$ is the input of the linear  operator $\MM_h^{\dagger}$. We assume that $\Vert \MM_h^{\dagger}\Vert_{1}\leq 1/\alpha$ for all $h\in[H]$.
\end{assumption}
The following lemma shows that low-rank MDPs that satisfy Assumption \ref{asp:fs35} are also generalized regular PSRs.
\begin{lemma}[Low-rank POMDP $\subset$ Generalized Regular PSR]\label{lemma:lowrank-pomdp}
Any low-rank MDP that satisfies Assumption \ref{asp:fs35} is an $\alpha$-generalized regular PSR.
\end{lemma}

\begin{proof}
    See Appendix \ref{appendix:pomdp:linear} for a detailed proof.
\end{proof}

\paragraph{Regular PSR.}  
\vspace{4pt}

Recall that $\bar{\DD}_h$ defined in \eqref{eq:def:bar:D} is a row-wise sub-matrix of $\DD_h$ and satisfies $\rank(\bar{\DD}_h) = d_{\psr, h}$. Then we introduce the notion of \emph{core matrix} $\KK_h \in \RR^{|\cU_{h+1}| \times d_{\psr,h}}$, which is a column-wise sub-matrix of $\Bar{\DD}_h$ such that $\rank(\KK_h) = \rank(\bar{\DD}_h) = d_{\psr,h}$. Assuming $\KK_h = [\bar{\bfq}(\tau_h^1), \cdots, \bar{\bfq}(\tau_h^{d_{\psr,h}})]$, we refer to $\{\tau_h^1, \cdots, \tau_h^{d_{\psr,h}}\}$ as the \emph{minimum core histories} at step $h$. Since ${\KK}_h$ has rank $d_{\psr, h}$, we know each column of $\bar{\DD}_h$ is a linear combination of the columns in $\KK_h$. That is, for any length-$h$ history $\tau_h$, there exists a vector $\bfv_{\tau_h} \in \RR^{d_{\psr,h}}$ such that 
 \# \label{eq:k1}
\bar{\bfq}({\tau_h}) = \KK_h \bfv_{\tau_h}.
 \#
Let $\VV_h = [\bfv_{\tau_h}]_{\tau_h \in (\cO \times \cA)^h} \in \RR^{d_{\psr,h} \times O^h A^h}$, \eqref{eq:k1} further implies that
\# \label{eq:k2}
\bar{\DD}_h = \KK_h \VV_h .
\#
Then the regular PSR model \citep{zhan2022pac} is defined as follows.

\begin{definition}[Regular PSR \citep{zhan2022pac}] \label{def:regular:psr}
    We say a PSR is $\alpha$-regular if $\|\KK_h^\dagger\|_{1} \le 1/\alpha$ for all $h \in [H]$.
\end{definition}

In the following lemma, we show that regular PSR is a special case of the generalized regular PSR defined us.

\begin{lemma}[Regular PSR $\subset$ Generalized Regular PSR] \label{lemma:regular:psr}
    Any $\alpha$-regular PSR is an $\alpha$-generalized regular PSR.
\end{lemma}

\begin{proof}
    See Appendix \ref{appendix:regular:psr} for a detailed proof.
\end{proof}

In summary, we have shown that weakly revealing POMDPs, latent MDPs, decodable POMDPs, low-rank POMDPs, and regular PSRs are special cases of generalized regular PSR.
Together with Lemma~\ref{lemma:general:regular:psr}, we prove that all of these partially observable interactive decision making problems have   low GEC.

\subsubsection{PO-bilinear Class}

In this subsection, we consider the PO-bilinear class framework \citep{uehara2022provably}, which subsumes a rich class of partially observable interactive decision making such as observable tabular POMDPs \citep{golowich2022learning,golowich2022planning}, decodable POMDPs \citep{du2019provably,efroni2022provable}, and observable POMDPs with latent low-rank transition \citep{uehara2022provably}.

Let $\cZ_h = (\cO \times \cA)^{\min\{h, M\}}$ and $\bar{\cZ}_h = \cZ_{h-1} \times \cO$, where $M$ is a fixed positive integer. In the PO-bilinear class setting, we focus on the class of $M$-memory policies $\Pi =  \{\pi = \{\pi_h : \bar{\cZ}_h \rightarrow \Delta_{\cA}\}_{h \in [H]}\}$.\footnote{In PO-bilinear class, the goal is to find an optimal $M$-memory policy. As shown in \citet{uehara2022provably}, part of the results can be extended to the case where the learner aims to find the globally optimal policy.} In other words, at the $h$-th timestep, $\pi_h$ chooses the action $a_h$ only based on the data generated in the most recent $M$ steps. Given a policy $\pi \in \Pi$, its $h$-step value function is defined as
\$
V_h^{\pi}(z_{h-1}, s_h) = \EE\Big[ \sum_{h' = h}^H r_{h'} \mid z_{h-1}, s_h, a_{h:H} \sim \pi \Big], \quad \forall (z_{h-1}, s_h) \in {\cZ}_{h - 1} \times \cS.
\$

\begin{definition}[Value Link Functions] \label{def:link:function}
    Value link functions $g_h^\pi: \cZ_{h-1} \times \cO \rightarrow \RR$ at step $h \in [H]$ for a policy $\pi \in \Pi$ are defined as the solution of the following integral equation:
    \$
    \EE [ g_h^\pi (z_{h-1}, o_{h}) \mid z_{h-1}, s_h ] = V_h^\pi(z_{h-1}, s_h), \quad \forall (z_{h-1}, s_h) \in {\cZ}_{h - 1} \times \cS.
    \$
\end{definition}

\begin{definition}[PO-bilinear Class] \label{def:po:bilinear}
    We say the model is in the PO-bilinear class if $\cG$ is realizable (for any policy $\pi \in \Pi$, its value link function $g^\pi \in \cG$), and there exist $W_h : \Pi \times \cG \rightarrow \cV$ and $X_h: \Pi \rightarrow \cV$ for some Hilbert space $\cV$ such that for all $\pi, \pi' \in \Pi$, $g \in \cG$, and $h \in [H]$:
    \begin{itemize}
        \item[1.] We have
              \$
              \EE\left[r_h +  g_{h+1}\left(\bar{z}_{h+1}\right) - g_h(\bar{z}_h) ; a_{1: h-1} \sim \pi^{\prime}, a_{h} \sim \pi\right] = \left\langle W_{h}(\pi, g), X_{h}\left(\pi^{\prime}\right)\right\rangle .
              \$
        \item[2.] $W_h(\pi, g^\pi) = 0$, where $g^\pi \in \cG$ is the value link function of $\pi$.
    \end{itemize}
    We also define $\cX_h = \{X_h(\pi): \pi \in \Pi\}$ and $\cX = \{\cX_h\}_{h\in[H]}$ with a slight abuse of notation.
\end{definition}

Here we introduce the one-step value link functions and the corresponding version of PO-bilinear class, and we will discuss the extension to multi-step case in Appendix~\ref{appendix:po:bilinear}. For  ease of presentation, we denote $\cH = \Pi \times \cG$. Following \citet{uehara2022provably}, we define the loss function $l: \cH \times \bar{\cZ}_{h} \times \cA \times \RR \times \cO \rightarrow \RR$ as
\# \label{eq:def:loss:pobilinear}
l(f, \zeta_h) = |\cA| \pi_{h}(a_h \mid \bar{z}_h) \big( r_h + g_{h+1}(\bar{z}_{h+1}) \big) - g_h(\bar{z}_h) ,
\#
where $f = (\pi, g)$, $\zeta_h = (\bar{z}_h, a_h, r_h, o_{h+1})$, and $\bar{z}_{h+1} = (\bar{z}_h, a_h, o_{h+1})$. For $f = (\pi, g), $ the corresponding exploration policy is defined as
\# \label{eq:def:explore:pobilinear}
\pi_{\exp}(f, h) = \pi \circ_{h} \mathrm{Unif}(\cA).
\#
We also define $V_{f} := \EE_{\pi}[g_1(o_1)]$ and $V^{\pi}: = \EE_{\pi}[g^{\pi^*}_1(o_1)]$, where $f = (\pi, g)$ and $g^{\pi^*}$ is the link function for the optimal $M$-memory policy $\pi^*$. Since we can sample multiple initial observations to estimate the distribution of $o_1$ \citep{uehara2022provably}, we assume the distribution of $o_1$ is known without loss of generality. With these notations, we have the following lemma, which states that the PO-bilinear class has a low GEC.

\begin{lemma}[PO-bilinear Class $\subset$ GEC] \label{lemma:po:bilinear}
    For the PO-bilinear class, it holds for any $\{f^t = (\pi^t, g^t)\}_{t \in [T]}$ that
    \$
    \sum_{t = 1}^T V_{f^t} - V^{\pi^t} \le \Big[2\gamma_T(\epsilon,\cX)\sum_{t=1}^T \sum_{h=1}^H \sum_{j=1}^{t-1}  \big( \EE_{{\pi}_{\exp}(f^j, h)} l(f^t, \zeta_h) \big)^2\Big]^{1/2} +  2\min\left\{2\gamma_T(\epsilon, \cX), HT\right\} + HT\epsilon,
    \$
    where $\gamma_T(\epsilon, \cX)$ is the information gain defined in Definition \ref{def:info:gain}. This further implies that
    \$
    \GEC(\cH, l^2, \Pi_{\exp}, \epsilon) \le 2\gamma_T(\epsilon,\cX).
    \$
\end{lemma}

\begin{proof}
    See Appendix \ref{appendix:po:bilinear} for a detailed proof.
\end{proof}

We summarize all results in Sections~\ref{sec:relation:mdp} and~\ref{sec:relation:pomdp} in Table~\ref{tab:gec}, and conclude that
\begin{tcolorbox}
\begin{center}
    {\bf GEC  captures a majority of known complexity measures proposed for tractable interactive decision making.}
\end{center}
\end{tcolorbox}

\begin{table}[t]
    \centering
    \begin{tabular}{|c|c|}
    \hline  
    \cellcolor[HTML]{F5F5F5} Existing Complexity Measure & \cellcolor[HTML]{F5F5F5} $\GEC(\cH, \ell, \Pi_{\exp}, \epsilon)$  \\
    \hline
    \hline 
    \multirow{2}{*}{BE Dimension $d_Q/d_V$ } & $2d_QH\cdot \log(T)$ (Q-type)  \\ \cline{2-2} & $2 d_V A H\cdot \log(T)$ (V-type)   \\
    \hline 
    \multirow{2}{*}{ Bilinear Class with Information Gain $\gamma_T(\epsilon, \cX)$ } & $2\gamma_T(\epsilon, \cX)$ (Q-type)  \\ \cline{2-2} & $2A \cdot \gamma_T(\epsilon, \cX)$ (V-type)   \\
    \hline 
     \multirow{2}{*}{ Witness Rank $d_Q/d_V$ } & $4d_QH\cdot\log(1+ T/(\epsilon \kappa_{\mathrm{wit}}^2))/\kappa_{\mathrm{wit}}^2$  (Q-type)  \\ \cline{2-2} & $4d_VAH\cdot\log(1+ T/(\epsilon \kappa_{\mathrm{wit}}^2))/\kappa_{\mathrm{wit}}^2$ (V-type)   \\
    \hline 
     $\alpha$-Generalized Regular PSR & $\cO( (d_{\psr} \cdot A^3 U_A^4 H  \cdot \iota)/\alpha^4)$\\
     \hline
     PO-bilinear Class with Information Gain $\gamma_T(\epsilon, \cX)$  & $2\gamma_T(\epsilon, \cX)$ \\
     \hline
     \end{tabular}
    \caption{A summary of relationship between GEC and known complexity measures or model classes of tractable interactive decision making, including Bellman eluder dimension, bilinear class, witness rank, generalized regular PSR, PO-bilinear class. 
    We prove that GEC can be upper bounded by the quantities that are specific to each function class, thus showing the GEC serves as a generic complexity measure for interactive decision making. 
     Here $(\cH, \ell, \Pi_{\exp})$ and other problem-dependent parameters are specified in corresponding lemmas.}
    \label{tab:gec}
\end{table}

\section{Algorithm and Theoretical Guarantees} \label{sec:alg}

We first propose a generic posterior sampling algorithm for general interactive decision making, and then instantiate this generic algorithm to specific problems, followed by theoretical guarantees.

\subsection{Generic Algorithm}

We present our generic posterior sampling algorithm for interactive decision making in Algorithm~\ref{alg:generic}. Our algorithm has two ingredients: optimistic posterior distribution and data collection, 
which are explained in the sequel. 

\begin{algorithm}[t]
    \caption{Generic Posterior Sampling Algorithm for Interactive Decision Making (GPS-IDM)}
    \begin{algorithmic}[1] \label{alg:generic}
        \STATE {\textbf{Input: }} Hypothesis class $\cH$, discrepancy function $\ell$, prior $p^0$, set $\HH \subset \{0, 1, \cdots, H\}$, function $\cL = \{\cL_h\}_{h \in \HH}$, batch size $\nb$, and parameter $\gamma$ 
        \FOR{$t = 1, \cdots, T$}
        \STATE $p^t(f) \propto p^0(f) \cdot 
        \exp(\gamma V_{f} + \sum_{h \in \HH} L_{h}^{1:t-1}(f))$ \label{line:posterior}
        \STATE Sample $f^t \sim p^t$
        \FOR{$h \in \HH$}
        \STATE Execute $\pi_{\exp}(f^t, h)$ for $\nb$ times and collect samples $\cD_{h}^t$ \label{line:S}
        \STATE Calculate $L_h^{1:t}(f) = \cL_h(f, \{f^{s}\}_{s \in [t]}, \ell, \{\cD_{h}^{s}\}_{s \in [t]})$ \label{line:L}
         \ENDFOR
        \ENDFOR
    \end{algorithmic}
\end{algorithm}

\vspace{4pt}
\noindent{\bf Optimistic Posterior Sampling}. 
At iteration $t$, 
the learner has access to the hypotheses $\{f^s\}_{s \in [t-1]}$ generated in previous iterations, samples $\{\cD_{h}^s\}_{(s, h) \in [t - 1] \times \HH}$ and their reward/loss $\{L_{h}^{1:t-1}\}_{h \in \HH}$. 
Here, $\HH \subset \{0, 1, \cdots, H\}$ is a problem-dependent set (e.g., $\HH = \{0, \cdots. H - 1\}$ for PSR), and we will specify it for specific interactive decision making. Intuitively, the loglikelihood $\sum_{h \in \HH} L_h^{1:t-1}(f)$ represents the cumulative reward/loss of hypothesis $f$ in the historical data. The hypotheses that are more consistent with the historical data collected so far will be assigned with a larger probability. To this end, we adopt the following \textit{optimistic} posterior distribution:
\# \label{eq:generic:posteior}
p^t(f) \propto p^0(f) \cdot 
        \exp\Big(\gamma V_f + \sum_{h \in \HH} L_{h}^{1:t-1}(f)\Big),
\#
where $p^0$ is a prior distribution, and $\gamma$ is a tuning parameter that will be specified later. Here we introduce an optimistic term $\exp(\gamma V_f)$ to encourage exploration, which is inspired by the feel-good Thompson sampling proposed by \citet{zhang2022feel}. 
Such a term is optimistic because it assigns a higher density to hypothesis $f$ whose value is large.

\vspace{4pt}
\noindent{\bf Data Collection} \quad During the data collection process, the learner samples a hypothesis $f^t$ according to the posterior in \eqref{eq:generic:posteior}, i.e., $f^t \sim p^t$. 
Then for any $h \in \HH$, the learner executes the exploration policy $\pi_{\exp}(f^t, h)$ for $\nb$ times to collect samples $\cD^t_{h}$, where the exploration policy can be chosen as in Examples~\ref{example:policy:q}, \ref{example:policy:v}, and \ref{example:policy:psr}, and $\nb$ is a tuning parameter to be determined later. 
Finally, the learner calculates $\{L_h^{1:t}(f) = \cL_h(f, \{f^{s}\}_{s \in [t]}, \ell, \{\cD_{h}^{s}\}_{s \in [t]})\}_{h \in \HH}$, where $\cL$ is a carefully crafted function for the specific problem. 
Taking PSR as an example, we consider the model-based hypothesis (Example~\ref{example:hyp:model}), and $\cD_h^t$ is a length-$H$ trajectory $\tau_h^t$ collected by the exploration policy in Example~\ref{example:policy:psr}. Here $\nb = 1$. Then $L_h^{1:t}$ is chosen as $\eta \sum_{s = 1}^t \log \PP_{f}(\tau_h^s)$, where $\PP_{f}(\tau_h^s)$ is specified in \eqref{eq:def:prob2} and $\eta$ is a tuning parameter. 

\subsection{Results for Fully Observable Interactive Decision Making} \label{sec:result:full}
In this subsection, we instantiate the   algorithmic framework of  GPS-IDM on   interactive decision making problems with full observations, followed by the corresponding theoretical guarantees of the algorithm. 
We first introduce the model-free version.

\subsubsection{Model-free Approach} 
For the model-free approach, we choose $\nb = 1$, and $\HH = [H]$ as inputs of Algorithm~\ref{alg:generic}. That is, at $t$-th iteration, we execute the exploration policy $\pi_{\exp}(f^t, h)$ to collect $\zeta_h^t = (x_h^t,a_h^t,r_h^t,x_{h+1}^t)$ for all $h \in [H]$. We then adopt the conditional posterior sampling from \citet{dann2021provably} but also make some important extensions to better fit in the GEC framework. First, we impose the following assumption on the choice of the loss function. 

\begin{algorithm}[t]
    \caption{GPS-IDM (Model-free Version)}
    \begin{algorithmic}[1] \label{alg:mdp_complete}
        \STATE {\textbf{Input: }} Hypothesis class $\cH$, prior $p^0$, parameters $\eta$ and $\gamma$
        \FOR{$t = 1, \cdots, T$}
        \STATE $p^t \propto p^0(f) \cdot  
        \exp(\gamma V_{f} + \sum_{h \in \HH} L_{h}^{1:t-1}(f))$ (cf. \eqref{eqn:posterior_model_free_mdp})
        \STATE Sample $f^t \sim p^t$
        \FOR{$h = 1, 2, \cdots, H$}
        \STATE Execute $\pi_{\exp}(f^t, h)$ and observe $\cD_h^t = \zeta^t_{h} = (x_{h}^t, a_{h}^t, r_h^t, x_{h+1}^t)$
        \STATE Calculate $L_h^{1:t}(f)$ as in  \eqref{eq:L:model:free}
        \label{line:mdp:l}
        \ENDFOR
        \ENDFOR
    \end{algorithmic}
\end{algorithm}

\begin{assumption} \label{assumption:l}
    Suppose the discrepancy function in Definition~\ref{def:gec} is $\ell_{f'}(f, \xi_h)$, where $\xi_h = (x_h, a_h)$.  We assume there exists a  function $l$ satisfying
    \$
    \ell_{f'}(f, \xi_h) = |\EE_{x_{h+1} | x_h, a_h} l_{f'}((f_h, f_{h+1}), \zeta_h)|^2,
    \$
    where $\zeta_h = (x_h, a_h, r_h, x_{h+1})$. 
    Here the function $l$ can also depend on a hypothesis $f'$. We also use the notation $l_{f'}(f, \zeta_h) = l_{f'}((f_h, f_{h+1}), \zeta_h)$.
\end{assumption}

\begin{remark}
    For RL problems with low BE dimension and bilinear class, Assumption~\ref{assumption:l} holds naturally. For RL problems with low BE dimension, the discrepancy function $\ell$  and the choice of $l_{f'}(f, \zeta_h)$ are specified in Lemma~\ref{lem:reduction_bed_dc}. For bilinear class $(\cH, l', \Pi_{\exp})$, we can choose $l = l'$ (Lemma~\ref{lemma:relationship:bilinear}). In other words, $l$ can be set as the discrepancy function introduced in the definition of the bilinear class. 
\end{remark}

Before continuing, to deal with the general loss function beyond the Q-type Bellman residual, we make the following assumption about  generalized completeness and boundedness of the hypothesis class.

\begin{assumption} \label{assu:discrepancy} We make the following assumptions.
       \begin{itemize}
           \item \textbf{Generalized completeness}. There exists an operator $\cP_h: \cH_{h+1} \to \cH_{h}$ such that for all $(f',f_h, f_{h+1}) \in \cH \times \cH_h \times \cH_{h+1}$, we have
           $$l_{f'}((f_h,f_{h+1}), \zeta_h) - l_{f'}((\cP_h f_{h+1},f_{h+1}), \zeta_h) = \EE_{x_{h+1} \sim \PP_h(\cdot | x_h,a_h)} [l_{f'}((f_h,f_{h+1}),\zeta_h)],$$
           where we require that $\cP_hf^*_{h+1} = f^*_h$, i.e., $f^*$ is a fixed point of the operator.
           \item \textbf{Boundedness}. We assume that $\sup_{f' \in \cH}\norm{l_{f'}((f_h,f_{h+1}), \zeta_h)}_\infty$ is bounded by some $B_l > 0$. 
       \end{itemize}
\end{assumption}
The first condition is a generalization of the Bellman completeness assumption used in \citet{jin2021bellman}. In specific, for all the examples considered in \citet{jin2021bellman} with $l_{f'}((f_h,f_{h+1}), \zeta_h) = Q_{h,f}(x_h,a_h) - r_h - V_{h+1,f}(x_{h+1})$, the generalized completeness assumption reduces to the Bellman completeness and we can take $\cP_h$ as the Bellman operator $\cT_h$. We also remark that some RL problems, including the linear Bellman complete MDP and linear mixture MDP, naturally satisfy this assumption. See Appendix~\ref{appendix:completeness} for details. Then, in terms of algorithm design, we take $V_{f} = V_{1,f}(x_1)$ and factorize the prior into a product of $H$ priors: $p^0(f) \propto \prod_{h=1}^H p^0_h(f_h)$, where $p^0_h(\cdot)$ is a prior over $\cH_h$. We also use the following loglikelihood in \eqref{eq:generic:posteior}:
\# \label{eq:L:model:free}
L_h^{1:t-1}(f) = -\eta \sum_{s=1}^{t-1}l_{f^s} \big ((f_h, f_{h+1}),\zeta_h^s \big )^2 - \log \EE_{\tilde{f}_h \sim p^0_h(\cdot)} \bigg[  \exp\biggl(-\eta \sum_{s=1}^{t-1} l_{f^s}((\tilde{f}_h, f_{h+1}), \zeta_h^s)^2\biggr) \bigg] ,
\#
which leads to the following posterior:
\begin{equation} \label{eqn:posterior_model_free_mdp}
\begin{aligned}
    {p}^t(f) & \propto p^0(f) \cdot 
        \exp\big(\gamma V_{f} + \sum_{h \in \HH} L_{h}^{1:t-1}(f) \big) \\
    & \propto \exp({\gamma V_{f}}) \cdot \prod_{h=1}^H \frac{p^0_h(f_h)\cdot \exp
    \big(- \eta \sum_{s=1}^{t-1} \sum_{h=1}^H l_{f^s}((f_{h},f_{h+1}),\zeta_h^s)^2\big)}{\EE_{\tilde{f}_h \sim p^0_h} \exp\big(-\eta \sum_{s=1}^{t-1} \sum_{h=1}^H l_{f^s}((\tilde{f}_h,f_{h+1}),\zeta_h^s)^2\big)}.
\end{aligned}
\end{equation}
When we consider the value-based hypothesis class with $l_{f'}((f_h,f_{h+1}), \zeta_h) = Q_{h,f}(x_h,a_h) - r_h - V_{h+1, f}(x_h)$ and Bellmen completeness, \eqref{eqn:posterior_model_free_mdp} reduces to the conditional posterior sampling algorithm introduced in \cite{dann2021provably}. The following quantity is analog to that of \citet{dann2021provably}, and allows our posterior sampling algorithm to employ a prior $p^0$ over $\cH$, which favors certain parts of the hypothesis class.

\begin{definition} \label{def:kappa} Let $\zeta_h = (x_h,a_h,r_h,x_{h+1})$. We define a  set of functions $$\cF_h(\epsilon, f_{h+1}) := \{f_h \in \cH_h: \sup_{x_h,a_h,f'} |\EE_{x_{h+1} \sim \PP_h(\cdot|x_h,a_h)} l_{f'}((f_h, f_{h+1}), \zeta_h)| \leq \epsilon\}, $$ which approximately satisfy the generalized completeness condition. Then, we define 
    $$
    \kappa^c_h(\epsilon) := \sup_{f_{h+1} \in \cH_{h+1}}  - \log    p^0_h(\cF_h(\epsilon, f_{h+1})) $$ 
    and $\kappa^c(\epsilon) = \sum_{h=1}^H \kappa^c_h(\epsilon)$. We also define $\kappa^r(\epsilon):= - \log p^0(\cH(\epsilon))$, where $$\cH(\epsilon) = \bigg\{f \in \cH: V_{1,f^*}(x_1)-V_{1,f}(x_1) < \epsilon; \sup_{x_h,a_h,h,f'} |\EE_{x_{h+1} \sim \PP_h(\cdot|x_h,a_h)} l_{f'}((f_h, f_{h+1}), \zeta_h)| \leq \epsilon \bigg\} . $$
Finally, we define 
\$
\kappa(\epsilon, p^0) = \max\{\kappa^c(\epsilon), \kappa^r(\epsilon)\}.
\$
\end{definition}
The quantity $p_h^0(\cF_h(\epsilon, f_{h+1}))$ is the prior probability of hypotheses that approximately the generalized completeness condition with respect to $f_{h+1}$. Therefore, $\kappa^c(\epsilon)$ represents an approximate generalized completeness assumption in terms of the prior $p^0$. Moreover, $\kappa^r(\epsilon)$ measures how well the prior covers the hypotheses that behave similarly to the $f^*$, which is an approximate realizability condition. By Assumptions~\ref{assu:realizability} and \ref{assu:discrepancy}, $\cH(\epsilon)$ and $\cF_h(\epsilon, f_{h+1})$ are never empty. If $\cH$ is a finite hypothesis class, one can choose a uniform prior $p^0$ such that $\kappa(\epsilon, p^0) \leq  \log |\cH|$. This can be extended to the infinite hypothesis class effortlessly by performing a  covering number  argument, and $\log|\cH|$ is replaced by the log-covering number of $\cH$. We now state the main result of Algorithm~\ref{alg:mdp_complete} as follows.
\begin{theorem} \label{thm:1}
     Suppose that Assumptions~\ref{assu:realizability}, \ref{assumption:l}, and~\ref{assu:discrepancy} hold. If we set $\eta = 3/(10B_l^2)$ and $\gamma = \sqrt{\frac{T\kappa(1/T^2, p^0)}{B_l^2 d_{\GEC}}}$, then for MDP problems with  $d_{\GEC} = \GEC(\cH, \ell, \Pi_{\exp}, 1/  \sqrt{T})$, the regret of Algorithm~\ref{alg:mdp_complete} satisfies:
     $$
     \EE [\reg(T)] \leq \cO\big(B_l\sqrt{d_{\GEC} \cdot H T \cdot \kappa(1/T^2, p^0)}\big).
     $$
     If the hypothesis class $\cH$ is finite, we further have
     $$
     \EE [\reg(T)] \leq \cO\big(B_l\sqrt{d_{\GEC} \cdot H T \cdot \log|\cH|}\big).
     $$
\end{theorem}
\begin{proof}
    See Appendix~\ref{sec:model_free_mdp} for a detailed proof.
\end{proof}

\begin{remark}
For Q-type problems, the exploration policy satisfies $\pi_{\exp}(f^t,h) = \pi_{f^t}$ for any $(t,h) \in [T] \times [H]$. Hence, at the  $t$-th episode, we can use $\pi^t$ to collect a length-$H$ trajectory, and thus further improves the regret bound by a factor of $\sqrt{H}$.
\end{remark}

To illustrate our theory more, we instantiate Theorem~\ref{thm:1} to  RL problems with low BE dimension and bilinear class. For ease of presentation, we present the results with a finite  $\cH$.

\begin{theorem}[BE Dimension] \label{thm:be}
    Suppose that Assumptions~\ref{assu:realizability} and~\ref{assu:discrepancy} hold. For the MDP problem with a  Q-type BE dimension $d_Q$ or a  V-type BE dimension $d_V$, if we set $\eta = 3/(10A^2)$, then the regret of Algorithm~\ref{alg:mdp_complete} satisfies:
      \$
   \EE[\reg(T)] \le \left\{\begin{array}{l} \tilde{\cO}\big(H \sqrt{d_{Q} T  \cdot \log |\cH|} \big)  \qquad  \text{(Q-type)} \\ \tilde{\cO}\big(H \sqrt{d_{V} A  T \cdot \log |\cH|} \big)  \,\quad \text{(V-type)} 
   \end{array}\right. .
   \$
\end{theorem}

\begin{proof}
    This result is directly  implied by combining  Lemma~\ref{lem:reduction_bed_dc} and Theorem~\ref{thm:1}.
\end{proof}

\begin{theorem}[Bilinear Class] \label{thm:bilinear}
    Suppose that Assumptions~\ref{assu:realizability} and~\ref{assu:discrepancy} hold. For the bilinear class $(\cH, l, \Pi_{\exp})$ with information gain $\gamma_T(\epsilon, \cX)$, if we set $\eta = 3/(10B_l^2)$, then the regret of Algorithm~\ref{alg:mdp_complete} satisfies:
      \$
   \EE[\reg(T)] \le \left\{\begin{array}{l} \cO\Big(B_l \sqrt{\gamma_T(1/\sqrt{T}, \cX) \cdot H T \cdot\log|\cH|} \Big) \qquad  \text{(Q-type)} \\ \cO\Big(B_l \sqrt{\gamma_T(1/\sqrt{T}, \cX) \cdot A H T \cdot\log|\cH|} \Big)  \,\quad \text{(V-type)} 
   \end{array}\right. .  
   \$
\end{theorem}

\begin{proof}
    This result is direclty  implied by combining  Lemma~\ref{lemma:relationship:bilinear} and Theorem~\ref{thm:1}.
\end{proof}

Theorems~\ref{thm:be} and \ref{thm:bilinear} show that Algorithm~\ref{alg:mdp_complete} can tackle all problems with low BE dimension and bilinear class with sharp regret guarantees. This significantly improves the result in \citet{dann2021provably}, which can only solve the problems with low Q-type BE dimension.

\subsubsection{Model-based Approach} \label{alg:mdp_model_based}
In model-based RL, we choose the model-based hypothesis $\cH$ as in Example~\ref{example:hyp:model}, $\nb = 1$, and $\HH = [H]$ as inputs of Algorithm~\ref{alg:generic}. For the data collecting phase, we use the exploration policy $\pi_{\exp}(f^t,h)$ in Example~\ref{example:policy:q} (for Q-type problems) or Example~\ref{example:policy:v} (for V-type problems) to collect samples $\cD_h^t = (x_{h}^t, a_{h}^t, r_h^t, x_{h+1}^t)$ in Line~\ref{line:S} of Algorithm~\ref{alg:generic}. Then we calculate the logliklihood  $L_h^{1:t}$ in Line~\ref{line:L} of Algorithm~\ref{alg:generic} via 
$$ L_h^{1:t}(f) := \sum_{s = 1}^t L_h^s(f) :=  \eta \sum_{s = 1}^t \log \PP_{h, f} (x_{h+1}^s \mid x_{h}^s, a_{h}^s),$$ where $\eta$ is a tuning parameter. The pseudocode is given in Algorithm~\ref{alg:mdp}.

\begin{algorithm}[H]
    \caption{GPS-IDM (Model-based Version)}
    \begin{algorithmic}[1] \label{alg:mdp}
        \STATE {\textbf{Input: }} Hypothesis class $\cH$, prior $p^0$, parameters $\eta$ and $\gamma$
        \FOR{$t = 1, \cdots, T$}
        \STATE $p^t(f) \propto p^0(f) 
        \exp(\gamma V_f + \sum_{s=1}^{t-1} \sum_{h = 1}^H L_h^s(f))$
        \STATE Sample $f^t \sim p^t$
        \FOR{$h = 1, 2, \cdots, H$}
        \STATE Execute $\pi_{\exp}(f^t, h)$ and observe $\cD_{h}^t = (x_{h}^t, a_{h}^t, r_h^t, x_{h+1}^t)$
        \STATE Calculate $L_h^t(f) =  \eta \log \PP_{h, f} (x_{h+1}^t \mid x_{h}^t, a_{h}^t)$ 
        \ENDFOR
        \ENDFOR
    \end{algorithmic}
\end{algorithm}

Before presenting our theoretical results, we define the following quantity to measure how well the prior distribution $p^0$ covers the true model $f^*$.

\begin{definition} \label{def:omega:mdp}
    Given $c > 0$ and $p^0 \in \Delta_\cH$, we define
    \$
    \omega(c, p^0) = \inf_{\epsilon > 0} \big[c \epsilon - \log p^0\big(\cH(\epsilon) \big) \big],
    \$
    where $\cH(\epsilon) = \{f \in \cH : V^* - V_f \le \epsilon; \sup_{h,x,a} \mathrm{KL}(\PP_{h, f^*}(\cdot \mid x, a) \| \PP_{h, f}(\cdot \mid x, a) ) \le \epsilon \}$. 
\end{definition}

If $\cH$ is a finite model class containing the true model $f^*$ and we choose prior $p^0$ as the uniform distribution over $\cH$, then we have $\omega(c, p^0) \le \log |\cH|$ for any $c > 0$. Before stating our results for model-based RL, we specify the definition of GEC here. For the $\GEC(\cH, \ell, \Pi_{\exp}, \epsilon)$ in model-based MDP, we choose the model-based hypothesis class $\cH$ as in Example~\ref{example:hyp:model}, discrepancy function $\ell$ as in Example~\ref{example:ell:model}, exploration policy class as in Example~\ref{example:policy:q} or \ref{example:policy:v}, and the burn-in cost in Definition~\ref{def:gec} as $2\min\{d, HT\} + \epsilon H T$. Then we have the following theorem for model-based MDP with low GEC.

\begin{theorem}[Model-based MDP with Low GEC] \label{thm:mdp}
    Suppose Assumption~\ref{assu:realizability} holds. For MDP with GEC $d_{\GEC} = \GEC(\cH, \ell, \Pi_{\exp}, 1/\sqrt{H^2T})$, we choose $\eta = 1/2$ and $\gamma = 2\sqrt{\omega(HT, p^0)T/d_{\GEC}}$. For $T \ge d_{\GEC}$, the regret of Algorithm~\ref{alg:mdp} satisfies
    \$
    \EE[\reg(T)] \le 4 \sqrt{d_{\GEC} \cdot H T \cdot \omega(HT, p^0) } .
    \$
    If the model class $\cH$ is finite, we further obtain that
    \$
    \EE[\reg(T)] \le 4 \sqrt{d_{\GEC} \cdot HT \cdot \log|\cH|  }   .
    \$
\end{theorem}

\begin{proof}
    See Appendix~\ref{appendix:pf:model:mdp} for a detailed proof.
\end{proof}

Now we instantiate Theorem~\ref{thm:mdp} in an  MDP with low witness rank (cf. Definition~\ref{def:witness}).

\begin{theorem}[Witness Rank]
    For the MDP problem with Q-type witness rank $d_Q$ or V-type witness rank $d_V$, if we choose $\eta = 1/2$ and $\gamma = 2\sqrt{\omega(HT, p^0)T/d_{\GEC}}$, the regret of Algorithm~\ref{alg:mdp} satisfies
    \$
    \EE[\reg(T)] \le 4 \sqrt{d_{\GEC} \cdot H T \cdot \omega(HT, p^0) },
    \$
    where $d_{\GEC} \le  4Hd_Q \cdot  \log(1 + H T^{3/2}/\kappa_{\mathrm{wit}}^2)/\kappa_{\mathrm{wit}}^2$  for Q-type witness rank and $d_{\GEC} \leq 4AHd_V \cdot \log(1 + HT^{3/2}/\kappa_{\mathrm{wit}}^2)/\kappa_{\mathrm{wit}}^2$  for V-type one.  If the model class $\cH$ is finite, we further obtain that
    \$
    \EE[\reg(T)] \le 4 \sqrt{d_{\GEC} \cdot HT \cdot  \log|\cH|  }   .
    \$
\end{theorem}
\begin{proof}
    This result is implied by Lemma~\ref{lem:witness_rank} and Theorem~\ref{thm:mdp}.
\end{proof}
While the bilinear class can capture the witness rank class \citep{du2021bilinear}, the previous result (Theorem~\ref{thm:mdp}) relies on the (generalized) completeness assumption. In comparison, here we establish a $\sqrt{T}$ regret bound for low witness rank problems without the completeness assumption. This is due to the different choices of discrepancy functions and their corresponding algorithmic design and analyses.

\subsection{Results for Partially Observable Interactive Decision Making} \label{sec:result:partial}

In this section, we provide algorithms and theoretical results for partially observable decision making with low GEC, which includes generalized regular PSR (Lemma \ref{lemma:general:regular:psr}) and PO-bilinear (Lemma~\ref{lemma:po:bilinear}) as special cases.

\subsubsection{Generalized Regular PSR}
For PSR, we choose the model-based hypothesis $\cH$ in Example~\ref{example:hyp:model}, $\nb = 1$, and $\HH = \{0, 1, \cdots, H-1\}$ as inputs of Algorithm~\ref{alg:generic}. For the data collecting phase, we use the exploration policy $\pi_{\exp}(f^t,h)$ in Example~\ref{example:policy:psr} to collect a trajectory $\cD_{h}^t = \tau_{h}^t$ in Line~\ref{line:S} of Algorithm~\ref{alg:generic}. Then we calculate the $L_h^{1:t}$ in Line~\ref{line:L} of Algorithm~\ref{alg:generic} by $L_h^{1:t}(f) : = \sum_{s = 1}^t L_h^s := \sum_{s = 1}^t \eta \log \PP_{f} (\tau_{h}^s)$, where $\eta$ is a tuning parameter. The pseudocode is given in Algorithm~\ref{alg:psr}.

\begin{algorithm}[H]
    \caption{GPS-IDM (PSR Version)}
    \begin{algorithmic}[1] \label{alg:psr}
     \STATE {\textbf{Input: }} Hypothesis class $\cH$, prior $p^0$, parameters $\eta$ and $\gamma$
        \FOR{$t = 1, \cdots, T$}
        \STATE Define $p^t(f) \propto p^0(f) \cdot \exp(\gamma V_f + \sum_{s = 1}^{t - 1} \sum_{h = 0}^{H - 1} L_h^s(f))$ 
        \STATE Sample $f^t \sim p^t$
        \FOR{$h = 0, \cdots, H - 1$}
        \STATE Execute ${\pi}_{\exp}(f^t, h)$ to collect a trajectory $\cD_{h}^t = \tau_{h}^t$
        \STATE $L_h^t(f) = \eta \log \PP_f(\tau_h^t)$
        \ENDFOR
        \ENDFOR
    \end{algorithmic}
\end{algorithm}

Similar to the model-based MDP case (cf. Definition~\ref{def:omega:mdp}), we introduce the following quantity to describe how well the prior $p^0$ covers the true model $f^*$.

\begin{definition} \label{def:omega:psr}
    Given $c > 0$ and prior $p^0$, we define
    \$
    \omega(c, p^0) = \inf_{\epsilon > 0} \big[c \epsilon - \log p^0\big(\cH(\epsilon) \big) \big],
    \$
    where $\cH(\epsilon) = \{f \in \cH : V^* - V_f \le \epsilon; \sup_{\pi} \mathrm{KL}(\PP_{f^*}^{\pi}(\tau_H) \| \PP_{f}^{\pi}(\tau_H) ) \le \epsilon \}$. 
\end{definition}

If $\cH$ is a finite model class containing the true model $f^*$ and we choose prior $p^0$ as the uniform distribution over $\cH$, then we have $\omega(c, p^0) \le \log |\cH|$ for any $c > 0$. 

Now we provide a more general result for PSR with low $\GEC(\cH, \ell, \Pi_{\exp})$, which subsumes the generalized regular PSR as a special case (Lemma \ref{lemma:general:regular:psr}).  Here we choose the hypothesis class $\cH$ as in Example~\ref{example:hyp:model}, discrepancy function $\ell$ as in Example~\ref{example:ell:psr}, exploration policy class $\Pi_{\exp}$ as in  Example~\ref{example:policy:psr}, and the burn-in cost in Definition~\ref{def:gec} as $\sqrt{dHT}$.

\begin{theorem}[PSR with Low GEC] \label{thm:psr}
    Choose $\eta = 1/2$ and $\gamma = 2\sqrt{\omega(HT, p^0)T/d}$ in Algorithm~\ref{alg:psr}. For PSR with generalized eluder coefficient $d_{\GEC} = \GEC(\cH, \ell, \Pi_{\exp})$,  where $\cH$, $\ell$, and $\Pi_{\exp}$ are specified in Examples~\ref{example:hyp:model}, \ref{example:ell:psr}, and \ref{example:policy:psr}, the regret of Algorithm~\ref{alg:psr} satisfies
    \$
    \EE[\mathrm{\reg}(T)] \le 2 \sqrt{d_{\GEC} \cdot HT \cdot (H + \omega(HT, p^0) )}.
    \$
    If the model class $\cH$ is finite, we further have
    \$
    \EE[\mathrm{\reg}(T)] \le 2 \sqrt{d_{\GEC} \cdot HT \cdot (H + \log|\cH|)}.
    \$
\end{theorem}

\begin{proof}
    See Appendix~\ref{appendix:pf:psr} for a detailed proof.
\end{proof}

By establishing a $\sqrt{T}$-regret bound that only scales polynomially in generalized eluder coefficient $d_{\GEC}$, episode length $H$, and $\omega$ defined in Definition~\ref{def:omega:psr},  Theorem~\ref{thm:psr} states that our algorithm (Algorithm~\ref{alg:psr})  solves the PSR with low GEC efficiently. To illustrate our theory more, we instantiate this theorem to the generalized regular PSR defined in Definition~\ref{def:general:regular:psr}.

\begin{theorem}[Generalized Regular PSR] \label{thm:general:psr}
     Choose $\eta = 1/2$ and $\gamma = 2\sqrt{\omega(HT, p^0)T/d}$ in Algorithm~\ref{alg:psr}. For $\alpha$-generalized regular PSR, the regret of Algorithms~\ref{alg:psr} satisfies
     \$
     \EE[\mathrm{\reg}(T)] \le \cO \biggl( \frac{A U_A^2 H \sqrt{d_{\psr} \cdot AT (H + \omega(HT, p^0) ) \cdot \iota}}{\alpha^2} \biggr),
     \$
      where $\iota = 2 \log\Big(1 + \frac{4 d_{\psr} A^2U_A^2  \delta_{\psr}^2 T }{ \alpha^4}\Big)$. If the model class is finite, we further have
      \$
      \EE[\mathrm{\reg}(T)] \le \cO \biggl( \frac{A U_A^2 H \sqrt{d_{\psr} \cdot AT (H + \log|\cH| ) \cdot \iota}}{\alpha^2} \biggr).
      \$
\end{theorem}

\begin{proof}
    This result is implied by Lemma~\ref{lemma:general:regular:psr} and Theorem~\ref{thm:psr}.
\end{proof}
    
The regret bound in Theorem~\ref{thm:general:psr} is polynomial in the number of actions $A$, the number of core action sequences $U_A$, PSR rank $d_{\psr}$, the episode length $H$, and  $\omega$ specified in Definition~\ref{def:omega:psr}. Notably, this result only depends on the number of core action sequences instead of the number of core tests. 
We will provide a more in-depth discussion about   $\log|\cH|$ and $\iota$  in Appendices~\ref{appendix:log:H} and \ref{appendix:delta}, respectively. Since we have shown in Appendix~\ref{sec:example:general:psr} that generalized regular PSR  captures a rich class of POMDPs, Theorem~\ref{thm:general:psr} implies a sequence of corollaries by specializing to these special cases. We defer these corollaries to  Appendix~\ref{appendix:cor} for brevity. Finally, our regret upper bound is independent of the size of the observation space, which allows us to tackle the large observation space setting. See Appendix~\ref{appendix:large:obs} for details.

\subsubsection{PO-bilinear Class} 
Recall that $\cZ_h = (\cO \times \cA)^{\min\{h, M\}}$ and $\bar{\cZ}_h = \cZ_{h-1} \times \cO$. In PO-bilinear class, we cannot impose a natural completeness assumption like Assumption~\ref{assu:discrepancy} in MDPs. In specific, given a tuple $(\bar{z}_h, a_h, r_h, o_{h+1}) \in \bar{\cZ}_h \times \cA \times \RR \times \cO$, the conditional expectation $\EE_{o_{h+1} \mid \bar{z}_h, a_h}$ is not well-defined since $o_{h+1}$ depends on the full history instead of the most recent $M$ steps. To this end, we use the trajectory aggregation technique to solve the PO-bilinear class without any additional assumption. Here the trajectory aggregation technique means that the algorithm executes the exploration policy repeatedly to collect a batch of trajectories for each iteration and thus achieves a better estimation by the sample mean.

At the beginning of $t$-th episode, given loss functions $\{L_h^{1:t-1}\}_{h \in [H]}$, the learner determines a distribution $p^t$ over $\cH = \Pi \times \cG$ as Line~\ref{line:posterior} of Algorithm~\ref{alg:generic}. Then, the learner samples a $f^t \sim p^t$ and executes $\pi_{\exp}(f^t,h)$ defined in \eqref{eq:def:explore:pobilinear} for $\nb$ times and collect samples 
\# \label{eq:def:cD:pobilinear}
\cD_{h}^t = \{ \zeta_{i, h}^t = (\bar{z}_{i,h}^t, a_{i,h}^t, r_{i,h}^t, o_{i,h+1}^t)\}_{i \in [\nb]},
\# 
where $\bar{z}_{i,h}^t \in \cZ_h$. 
Finally, the loss function is calculated by
\# \label{eq:def:L:pobilinear}
L_h^{1:t}(f) = -\eta \sum_{s=1}^{t}\Big( \frac{1}{\nb} \sum_{i = 1}^{\nb} l(f, \zeta_{i,h}^s) \Big)^2,
\#
where $l$ is defined in \eqref{eq:def:loss:pobilinear} and $\eta$ is a tuning parameter. The pseudocode is given in Algorithm~\ref{alg:po:bilinear}.

\begin{algorithm}[H]
    \caption{GPS-IDM (PO-Bilinear Class Version)}
    \begin{algorithmic}[1] \label{alg:po:bilinear}
    \STATE {\textbf{Input: }} Hypothesis class $\cH = \Pi \times \cG$, prior $p^0$, parameters $\eta$ and $\gamma$
        \FOR{$t = 1, \cdots, T$}
        \STATE Define $p^t(f) \propto p^0(f)  \cdot \exp(\gamma V_{f} + \sum_{h = 1}^{H } L^{1:t-1}_h(f))$ 
        \STATE Sample $f^t \sim p^t$
        \FOR{$h = 1, \cdots, H$}
        \STATE Execute ${\pi}_{\exp}(f^t, h)$ for $\nb$ times to collect samples $\cD_{h}^t$ (cf. \eqref{eq:def:cD:pobilinear})
        
        \STATE Calculate $L_h^{1:t}(f)$ as in \eqref{eq:def:L:pobilinear}
        \ENDFOR
        \ENDFOR
    \end{algorithmic}
\end{algorithm}

For ease of presentation, we consider the case where $\cH = \Pi \times \cG$ is finite and choose a uniform prior over $\cH$. This could be readily generalized to the infinite class setting (cf. Appendix~\ref{appendix:log:H}). Then we provide the following theoretical result for PO-bilinear class.

\begin{theorem}[PO-bilinear Class] \label{thm:pobilinear}
Suppose the learning process consists of $K$ episodes in total and $K = \nb T H$. We set $\nb=(A^2\iota/(d_{\GEC} H ))^{1/3}K^{2/3}$, $T = (d_{\GEC}/ (\iota A^2 H^2))^{1/3}K^{1/3}$, and $\gamma = \eta = K^{1/3} \log(|\cG| |\Pi|)$, where $\iota = \log (|\cG| |\Pi| HK^2)$. Then the regret of Algorithm~\ref{alg:po:bilinear} satisfies:
$$
\EE [\reg(K)] \leq \cO((d_{\GEC}\cdot AHK)^{2/3} \cdot \iota^{1/3}),
$$
where $d_{\GEC} = \GEC(\cH, l^2, \Pi_{\exp}, 1/(HK))$.
\end{theorem}

\begin{proof}
    See Appendix~\ref{appendix:pf:pobilinear} for a detailed proof.
\end{proof}


\section{Conclusion}
In this work, we introduce a novel complexity measure for  online interactive decision making named generalized eluder coefficient (GEC). 
GEC characterizes the  challenge of exploration by comparing the error of predicting the updated policy with the in-sample training error evaluated on the historical dataset. 
When specialized to particular examples of MDP, POMDP, PSR models, GEC can be bounded by other complexity measures such as Bellman eluder dimension \citep{jin2021bellman}, bilinear class \citep{du2021bilinear}, witness rank \citep{sun2019model}, PSR rank \citep{zhan2022pac}, and rank of PO-bilinear class \citep{uehara2022provably}. 
Furthermore, in addition to a new complexity measure, we propose a generic posterior sampling algorithm that can be implemented in both the model-free and model-based fashion. 
Moreover, this algorithm can be applied to both fully observable  and partially observable settings with provable sample efficiency guarantees. 
Compared to the vanilla posterior sampling algorithm, the proposed method uses an optimistic prior distribution that favors hypotheses with higher values. 
When the hypothesis space $\cH$ is finite, our posterior sampling achieves a $\sqrt{ d_{\mathrm{GEC}} \cdot H T \cdot \log | \cH| } $ regret, when $d_{\mathrm{GEC} } $ is GEC. 
When restricted to special MDP, POMDP, and PSR examples, such a regret bound matches those established in the literature. 
For future work, it would be interesting to understand the tightness of GEC by establishing a regret lower bound.  
In addition, for algorithm design, it would be interesting to develop posterior sampling algorithms that are provably efficient for multi-agent decision making systems with partial/private observations. 

\section*{Acknowledgements}

The authors would like to thank Yu Bai and Song Mei for pointing
out a technical issue in the first version regarding the $\ell_2$ eluder technique.

\bibliographystyle{ims}
\bibliography{graphbib}

\newpage
\appendix
\section{Notation Table and Additional Related Work}

\subsection{Notation Table} \label{appendix:notation}

\begin{tabular}{ll}
     \hline
     \bf{Notation}   & \bf{Explanation} \\
     \hline
     $\cH, \ell, \Pi_{\exp}$ & hypothesis class, discrepancy function, and exploration policy class\\
     $\cT_h, \cE_h$ & Bellman operator and Bellman residual defined in \eqref{eq:Bellman:equation} and \eqref{eq:def:residual}, respectively \\
     $\tau_h, \tau_{h:h'},\tau_h^{h+m}$ & $\tau_h = (o_{1:h}, a_{1:h})$, $\tau_{h:h'} = (o_{h:h'}, a_{h:h'}),\tau_h^{h+m}=(o_{h:h+m},a_{h:h+m-1})$ \\
     $\cU_h$ & core test set at step $h$ \\ 
     $\cU_{A,h}$ & action sequences in $\cU_h$ \\
     $U_A$   & $\max_{h} |\cU_{A,h}|$ \\
     $\DD_h$   & system dynamics \\
     $\bar{\DD}_h$ & row-wise sub-matrix of $\DD_h$ and its rows are indexed by tests in $\cU_{h+1}$ \\
     $d_{\psr, h}$, $d_{\psr}$ & $d_{\psr} = \max_{h} d_{\psr,h}$, where $d_{\psr,h} = \rank(\DD_h)$ (cf. Definition \ref{def:psr:rank})  \\
     $\KK_h$, $\VV_h$ & $\bar{\DD}_h = \KK_h \VV_h$, where $\KK_h$ is the core matrix at step $h$ (cf. \eqref{eq:k2}) \\
     $\delta_{\psr}$ & see \eqref{eq:def:delta} \\
     $\Mb_h(o, a)$ & observable operator matrix at step $h$ \\
     $\bfm(\tau_{h:H})$ & linear weights $\bfm(\tau_{h:H}) = \Mb_{H}(o_{H}, a_{H}) \cdots \Mb_h(o_h, a_h)$ (cf. \eqref{eq:linear:weight}) \\
     $\bar{\bfq}(\tau_h)$ & predictive state $[\PP(u \mid \tau_h)]_{u \in \cU_{h+1}}$\\
     $\bfq(\tau_h)$ & $[\PP(u, \tau_h)]_{u \in \cU_{h+1}} = \Mb_{h}(o_h, a_h) \cdots \Mb_1(o_1, a_1) \bfq_0$, where $\bfq_0 = [\PP(u)]_{u \in \cU_1}$ \\
     $\kappa, \kappa^c, \kappa^r$ & see Definition~\ref{def:kappa} \\
     $\omega$ & see Definition \ref{def:omega:mdp} (MDP version) and Definition \ref{def:omega:psr} (PSR version)\\
     \hline
\end{tabular}

\subsection{Comparison with Recent Independent Works} \label{appendix:concurrent}

During preparing this draft, we noticed several very recent independent works on fully observable RL \citep{chen2022general,chen2022unified} and partially observable RL \citep{chen2022partially,liu2022optimistic}. 

\paragraph{Comparison with \citet{chen2022unified,chen2022general}} 

Both our work and \citet{chen2022unified,chen2022general} can solve most known tractable fully observable RL problems, such as RL problems with low Bellman eluder dimension and bilinear class. But \citet{chen2022unified,chen2022general} do not consider  the more challenging partially observable RL problems. 
In contrast, our work solves a rich class of known tractable partially observable RL problems, including the generalized regular PSR and PO-bilinear class. Meanwhile,  most of the algorithms in \citet{chen2022unified,chen2022general} are optimism-based algorithms,  
and \citet{chen2022unified} also analyze a posterior sampling algorithm that is a variant of MOPS proposed by \citet{agarwal2022model}. In comparison, MOPS is a model-based posterior sampling algorithm, while our framework can further capture model-free methods. In terms of algorithmic ideas, both MOPS and our algorithm are motivated by \citet{zhang2022feel} and adopt optimistic modifications in the posterior distribution to obtain the frequentist regret guarantee. The main difference lies in the choices of complexity measure. The MOPS algorithm studied in \citet{agarwal2022model,chen2022unified} is based on the ``decoupling coefficeint'' proposed by \citet{zhang2022feel}. The main purpose of the decoupling coefficient is to decouple the model used for both data generalization and Bellman residual evaluations into two independent models. In this way, the online learning technique can be applied. Accordingly, they add an extra optimistic term at every step in the likelihood function to encourage exploration.
On the other hand, GEC serves to explicitly reduce the prediction error to the training error in historical data, which is more related to the eluder dimension \citep{russo2013eluder} and the corresponding optimistic modification is adopted for the prior distribution.

\paragraph{Comparison with \citet{chen2022partially}}

In terms of complexity measure, \citet{chen2022partially} propose a notion of \emph{B-stability} and they prove that PSR with B-stability is tractable. PSR with B-stability is similar to the generalized regular PSR in our paper and 
we can show that our complexity measure GEC captures B-stability in the sense that GEC can be upper bounded by the rank of PSR multiplied by some problem-specific quantities when B-stability holds. 

\begin{proposition}[B-Stable PSR $\subset$ PSR with Low GEC] \label{lemma:b:stable}
    For any $\Lambda_B$-stable PSR, its GEC is upper bounded by $\tilde{\cO}( \Lambda_B^2 \cdot  d_{\psr} \cdot  A U_A H)$.
\end{proposition}

\begin{proof}[Proof of Proposition \ref{lemma:b:stable}]
    By the same proof in \citep[][Appendix E]{chen2022partially}, it holds that
    \$
    \sum_{t = 1}^T V_{f^t} - V^{\pi_{f^t}} \le \biggl[d_{\GEC} \sum_{t=1}^T \sum_{s = 1}^{t - 1} \sum_{h=0}^{H-1}  D_H^2\Big(\PP_{f^t}^{\pi_{\exp}(f^s,h)}, \PP_{f^*}^{\pi_{\exp}(f^s,h)} \Big)\biggr]^{1/2} + \sqrt{d_{\GEC} H T},
    \$
    where $d_{\GEC} = \tilde{\cO}( \Lambda_B^2 \cdot d_{\psr} \cdot  A U_A H)$, which concludes our proof.
\end{proof}
They also show that B-stable PSR also permits a low DEC, which is the complexity measure proposed by \citet{foster2021statistical}. Moreover, our framework can also capture the PO-bilinear class \citep{uehara2022provably}, while they seem cannot.

In terms of algorithmic design, \citet{chen2022partially}  analyze the optimism-based OMLE algorithm proposed by \citet{liu2022partially} and explorative E2D algorithm \citep{foster2021statistical,chen2022unified}, which are different from our proposed posterior sampling algorithm. 
They also provide analysis for MOPS \citep{agarwal2022model} and the differences have been explained in the comparison with \citet{chen2022unified}. 

Finally, we remark that \citet{chen2022partially} develop a similar generalized $\ell_2$ eluder technique (cf. Appendix \ref{appendix:eluder}). Based on this technique, they establish a sharper bound for OMLE. In contrast, we utilize this technique to calculate the GEC for partially observable decision making and further analyze the posterior sampling algorithms. We also thank the authors of \citet{chen2022partially} for pointing out an $\ell_2$ eluder technical issue in our initial arXiv version  
and our fix is motivated by the counterpart of \citet{chen2022partially}. See Remark~\ref{remark:issue} for details.

\paragraph{Comparison with \citet{liu2022optimistic}}

The work of \citet{liu2022optimistic} focuses on the optimism-based algorithm OMLE \citep{liu2022partially} and provides an analysis under the generalized eluder type condition. 
Specifically, the generalized eluder type condition assumes that there exists a positive number  $d$ such that for any $(T, \Delta) \in \NN \times \RR^+$, and
\$
\forall t \in [T], \, \sum_{h = 1}^H \sum_{s = 1}^{t - 1} \mathrm{TV} \Big(\PP_{f^t}^{\pi_{\exp}(f^{s}, h)}, \PP_{f^*}^{\pi_{\exp}(f^{s}, h)} \Big)^2 \le \Delta \Rightarrow \sum_{t = 1}^T \mathrm{TV} \Big( \PP_{f^t}^{\pi_{f^t}}, \PP_{f^*}^{\pi_{f^t}} \Big) \le \tilde{\cO}\big(\sqrt{d \Delta |\Pi_{\exp}| K}\big).
\$
Here ``$\Rightarrow$'' means the former argument implies the latter. 
This generalized eluder-type condition is not suitable for the posterior sampling algorithm since we cannot find a small \emph{fixed} $\Delta$ such that $$\sum_{h = 1}^H \sum_{s = 1}^{t - 1} \mathrm{TV} \Big(\PP_{f^t}^{\pi_{\exp}(f^{s}, h)}, \PP_{f^*}^{\pi_{\exp}(f^{s}, h)} \Big)^2 \le \Delta\quad \textrm{for any}~t \in [T].$$  So their $\ell_1$ eluder technique seems not applicable to posterior sampling algorithms. To tackle this challenge, we develop a new $\ell_2$ eluder technique. See Appendix~\ref{appendix:eluder} for more details.

We also remark that \citet{liu2022optimistic} identify a tractable class of PSR (``well-conditioned PSR'' in  \citet{liu2022optimistic}), which is equivalent to our generalized regular PSR.  Both our work and \citet{liu2022optimistic} show that weakly revealing POMDPs (or the equivalent tabular observable POMDPs \citep{golowich2022learning,golowich2022planning}) and decodable POMDPs \citep{du2019provably,efroni2022provable} are special cases of generalized regular/well-conditioned PSR. Besides these two examples, they have some instances (e.g., observable POMDPs with continuous observations) not studied by us. But we also study the latent MDPs \citep{kwon2021rl}, low-rank POMDPs \citep{wang2022embed}, and regular PSRs \citep{zhan2022pac}, which are not discussed in their work. More importantly, our work can capture the PO-bilinear class \citep{uehara2022provably}, a rich class of known tractable partially observable RL problems, while they seem cannot.

Finally, both our work and \citet{liu2022optimistic} can tackle the fully observable interactive decision making. 
Due to using the usage of maximum likelihood estimation, their algorithm is model-based and thus seems only applicable to problems where directly estimating the transition models is feasible, e.g., MDPs with low witness rank.   In comparison, our generic algorithm can be either model-based or model-free and is capable of solving a more rich class of MDP problems such as low Bellman eluder dimension problems and bilinear class. 
\section{Proofs for Examples of Fully Observable Interactive Decision Making}

\subsection{Bellman Eluder Dimension} \label{appendix:BE}
In this subsection, we prove the Lemma~\ref{lem:reduction_bed_dc}. The proof follows the reduction of decoupling coefficient in \citet{dann2021provably} but with modification to further include the V-type Bellman eluder dimension and to deal with the different definitions of generalized  eluder coefficient. Our proof relies on the following lemma.

\begin{lemma} \label{lem:bed_proof} For any sequence of positive reals $x_1,\cdots,x_n$, we have
    $$
        f(x):=\frac{\sum_{i=1}^{n} x_{i}}{\sqrt{\sum_{i=1}^{n} i x_{i}^{2}}} \leq \sqrt{1+\log (n)}.
    $$
\end{lemma}

\begin{proof}
    See \citet{dann2021provably} for a detailed proof.
\end{proof}

\begin{proof}[Proof of Lemma~\ref{lem:reduction_bed_dc}]
    Let $E_\epsilon = \mathrm{dim_{BE}}(\Phi, \Pi, \epsilon)$. We consider a fixed $h \in [H]$. We first introduce some short-hand notations. We denote $\hat{\epsilon}_h^{st} = |\EE_{\mu_s} \phi_t|$ and $\epsilon_h^{st} = \hat{\epsilon}_h^{st} \mathbf{1}(\hat{\epsilon}_h^{st} > \epsilon)$ where $\phi_t \in \Phi$ and $\mu_t \in \Pi$. The proof proceeds as follows. We initialize $T$ empty buckets $B^0_h,\cdots, B^{T-1}_h$ and go through $\epsilon^{tt}_h$ one by one for $t \in [T]$. If $\epsilon^{tt}_h = 0$, we discard it. Otherwise, we go through the buckets from $0$ in an increasing order with the following rule. At bucket $i$,
    \begin{itemize}
        \item if $\sum_{s \leq t-1, s \in B_h^i} (\epsilon^{st}_h)^2 < (\epsilon^{tt}_h)^2$, we add $t$ into $B_h^i$;
        \item otherwise, we continue with the next bucket.
    \end{itemize}
    We denote the index of bucket that each non-zero timestep ends up in as $b^t_h$. As $\epsilon^{tt}_h$ skips the bucket $0,\cdots,b^t_{h}-1$, by construction, we have
    \begin{equation} \label{eqn:be_reduction0}
    \begin{aligned}
            \sum_{t=1}^{T} \sum_{s=1}^{t-1}\left(\epsilon_h^{st}\right)^{2} \geq \sum_{t=1}^T \sum_{0\leq i \leq b^t_{h}-1} \sum_{s \leq t-1, s \in B_h^i} \left(\epsilon^{st}_h\right)^2 \geq \sum_{t=1}^{T} b^{t}_{h}\left(\epsilon^{t t}_{h}\right)^{2},
        \end{aligned}
    \end{equation}
    where the first inequality is because $\{s \in B_h^i: s \leq t-1 \text{ and } 0 \leq i \leq b_h^t-1\}$ is a subset of $[t-1]$ and the summation terms are non-negative. The second inequality is because $\epsilon_h^{tt}$ skips bucket $i$ means that $(\epsilon_h^{tt}) \leq \sum_{s\leq t-1, s \in B_h^i} (\epsilon_h^{st})^2$ by construction. Note by the definition of eluder dimension, for the measures in $B_h^i$, say, $\{\mu_{t_i}: i = 1,\cdots,m\}$, $\mu_{t_j}$ is $\epsilon$-independent from all the predecessors $\mu_{t_1},\cdots,\mu_{t_{j-1}}$. Therefore, the size of each bucket cannot exceed the Bellman eluder dimension $E_\epsilon = \mathrm{dim_{BE}}(\cF, \Pi, \epsilon)$. By Jensen's inequality, we can obtain that
    \begin{equation} \label{eqn:be_reduction1}
        \begin{aligned}
            \sum_{t=1}^{T} b^{t}_{h}\left(\epsilon^{t t}_{h}\right)^{2} & =  \sum_{i=1}^{T-1} i \sum_{s \in B_h^i} (\epsilon^{ss}_h)^2 \geq \sum_{i=1}^{T-1} i |B_h^i| \bigg(\sum_{s \in B_h^i} \frac{\epsilon^{ss}_h}{|B_h^i|}\bigg)^2\geq \sum_{i=1}^{T-1} i E_\epsilon \bigg(\sum_{s \in B_h^i} \frac{\epsilon^{ss}_h}{E_\epsilon}\bigg)^2,
        \end{aligned}
    \end{equation}
    where the first equality uses an equivalent way of summation, and the last inequality uses $|B_h^i| \leq E_\epsilon$. By Lemma~\ref{lem:bed_proof} with $x_i = \sum_{s \in B_h^i} \epsilon^{ss}_h$, we know that
    \begin{equation}
        \label{eqn:be_reduction2}
        \begin{aligned}
            \sum_{i=1}^{T-1} E_\epsilon i\bigg(\sum_{s \in B^{i}_{h}} \frac{\epsilon_h^{ss}}{E_{\epsilon}}\bigg)^{2} & =\frac{1}{E_{\epsilon}} \sum_{i=1}^{T-1} i\bigg(\sum_{s \in B^{i}_{h}} \epsilon_h^{ss}\bigg)^{2} \geq \frac{1}{E_{\epsilon}(1+\log (T))}\bigg(\sum_{i=1}^{T-1} \sum_{s \in B_h^i} \epsilon_h^{ss}\bigg)^{2} \\
            & = \frac{1}{E_{\epsilon}(1+\log (T))}\bigg(\sum_{s \in [T] \backslash B^{0}_{h}} \epsilon_h^{ss}\bigg)^{2}.
        \end{aligned}
    \end{equation}
    To summarize, by \eqref{eqn:be_reduction0}, \eqref{eqn:be_reduction1} and \eqref{eqn:be_reduction2}, we have proved that
    $$\sum_{t=1}^{T} \sum_{s=1}^{t-1}\left(\epsilon_h^{st}\right)^{2} \geq  \frac{1}{E_{\epsilon}(1+\log (T))}\bigg(\sum_{s \in[T] \backslash B^{0}_{h}} \epsilon_h^{ss}\bigg)^{2}.$$ 
    It follows that
    $$
        \begin{aligned}
            \sum_{h=1}^{H} \sum_{t=1}^{T} \hat{\epsilon}^{t t}_{h} & \leq HT\epsilon +\sum_{h=1}^{H} \sum_{t=1}^{T} \epsilon^{t t}_{h} \leq HT \epsilon + \min\{E_\epsilon, T\} H +\sum_{h=1}^{H} \sum_{s \in[T] \backslash B^{0}_{h}} \epsilon_h^{ss} \\
& \leq HT \epsilon + \min\{HE_\epsilon, HT\}+\sum_{h=1}^{H} \Big[E_\epsilon (1 + \log (T))  \sum_{t=1}^{T} \sum_{s=1}^{t-1}\left(\epsilon_h^{st}\right)^{2}\Big]^{1/2}                        \\& \leq HT \epsilon + \min\{HE_\epsilon, HT\} + \Big[ E_\epsilon H ({1 + \log (T)})  \sum_{t=1}^{T} \sum_{h=1}^H\sum_{s=1}^{t-1}\left(\epsilon_h^{st}\right)^{2}\Big]^{1/2},
        \end{aligned}
    $$
    where the second inequality follows from $|B^0_h| \leq E_\epsilon$ and the last inequality follows from Jensen's inequality. For Q-type Bellman eluder dimension, we take $\phi_t = (f_h^t-\cT_h f_{h+1}^t)(\cdot,\cdot) \in (I-\cT_h)\cF =: \{f_h - \cT_h f_{h+1}: f \in \cF\}$ and $\mu_t = \mathbb{P}^{\pi^t}(x_h = \cdot, a_h = \cdot)$. Then we obtain
      \# \label{eq:002}
            \sum_{t=1}^T V_{f^t} - V^{\pi_{f^t}} & = \sum_{h=1}^H \sum_{t = 1}^T \EE_{\pi_{f^t}} [\cE_h(f^t, x_h, a_h)] \notag\\
            &\leq \Big[2d_QH\log(T) \sum_{t=1}^T \sum_{h=1}^H \sum_{s=1}^{t-1} \EE_{\pi_{f^s}}(\cE_h(f^t, x_h, a_h))^2\Big]^{1/2} + \min \{ T, d_Q\} H + HT \epsilon,      
    \#
    For V-type Bellman eluder dimension, we take $\phi_t(x) = (f_h^t - \cT_h f_{h+1}^t)(x, \pi_{f_h^t}(x)) \in (I-\cT_h) V_{\cF}:= \{(f_h - \cT_h f_{h+1})(\cdot, \pi_{f_h}(\cdot)): f \in \cF\}$ where $\pi_{f_h}(x) \in \cA$ is the greedy action in $f_h$, and we take $\mu_t = \mathbb{P}^{\pi^t}(x_h = \cdot)$. Then we have
    \# \label{eq:003}
    \sum_{t=1}^T V_{f^t} - V^{\pi_{f^t}} & = \sum_{h=1}^H \sum_{t = 1}^T \EE_{\pi_{f^t}} [\cE_h(f^t, x_h, a_h)] \notag \\ 
    & \leq \Big[2d_VH\cdot\log(T) \sum_{t=1}^T \sum_{h=1}^H\sum_{s=1}^{t-1} (\EE_{\pi_{f^s}}\cE_h(f^t, x_h, \pi_{f^t}(x_h)))^2\Big]^{1/2} + \min \{ T, d_V\} H + HT \epsilon \notag \\
    & \leq \Big[2d_VH\cdot\log(T) \sum_{t=1}^T \sum_{h=1}^H\sum_{s=1}^{t-1} \EE_{x_h \sim \pi_{f^s}, a_h \sim \pi_{f^t}} (\cE_h(f^t, x_h, a_h)))^2\Big]^{1/2} + \min \{ T, d_V\} H + HT \epsilon \notag \\
    & \leq \Big[2d_V A H \cdot \log(T) \sum_{t=1}^T \sum_{h=1}^H\sum_{s=1}^{t-1} \EE_{\pi_{\exp}({f^s},h)}(\cE_h(f^t, x_h, a_h))^2\Big]^{1/2}+ \min \{ T, d_V\} H + HT \epsilon,
    \#
    where the second inequality uses Jensen's inequality and the last inequality is implied by importance sampling. Combining \eqref{eq:002} and \eqref{eq:003}, we conclude the proof of Lemma~\ref{lem:reduction_bed_dc}.
\end{proof}

\subsection{Bilinear Class} \label{appendix:bilinear}

\begin{proof}[Proof of Lemma \ref{lemma:relationship:bilinear}]
    Let $\Sigma_{t;h} = \epsilon I + \sum_{s=1}^{t-1} X_h(f^s) X_h(f^s)^\top$. We first recall the definition of information gain (Definition~\ref{def:info:gain}):
    $$
    \gamma_T(\epsilon, \cX_h) = \max_{x_1, \cdots, x_{T} \in \cX_h} \log \det\left(I + \frac{1}{\epsilon} \sum_{t=1}^T x_tx_t^\top\right),
    $$
    where $\cX_h = \{X_h(f): f \in \cH\}$. Moreover, for any $\{x_t\}_{t=1}^T$, it follows that
    \begin{equation}\label{eqn:bilinear_coro}
    \sum_{t=1}^{T} \min\left\{\norm{x_t}_{\Sigma_{t}^{-1}}^2, 1\right\} \leq 2\log \det\left(I + \frac{1}{\epsilon}\sum_{t=1}^T x_tx_t^\top\right) \leq 2 \gamma_T(\epsilon, \cX_h),
    \end{equation}
    where we use Lemma~\ref{lemma:elliptical:potential} in the first inequality.
    By the performance difference lemma \citep{kakade2002approximate, jiang@2017}, we have
    \begin{equation}
        \begin{aligned}
             & \sum_{t = 1}^T V_{f^t} - V^{\pi^t} = \sum_{t=1}^T\sum_{h=1}^{H} \EE_{\pi_{f^t}} \left[\cE_h\left(f^t, x^t_h, a^t_h\right) \right]\\
             &= \sum_{t=1}^T\sum_{h=1}^{H} \min\left\{|\dotprod{W_h(f^t) - W_h(f^*), X_h(f^t)}|, 1\right\} \\
             & = \sum_{t=1}^T\sum_{h=1}^{H} \min\left\{|\dotprod{W_h(f^t) - W_h(f^*), X_h(f^t)}|, 1\right\} \cdot \Big(\mathbf{1}\{\left\|X_{h}\left(f^{t}\right)\right\|_{\Sigma_{t ; h}^{-1}} \leq 1\} + \mathbf{1}\{\left\|X_{h}\left(f^{t}\right)\right\|_{\Sigma_{t ; h}^{-1}} > 1\}\Big),
        \end{aligned}
    \end{equation}
where we further use the first condition of bilinear class in the second equality. We now proceed by invoking \eqref{eqn:bilinear_coro} to obtain that 
\# \label{eqn:bilinear_tmp}
            &\sum_{t = 1}^T V_{f^t} - V^{\pi^t} \notag \\
            &\leq \sum_{t=1}^T\sum_{h=1}^H \big(\left\|W_{h}\left(f^{t}\right)-W_{h}\left(f^{\star}\right)\right\|_{\Sigma_{t; h}} \cdot \left\|X_{h}\left(f^{t}\right)\right\|_{\Sigma_{t ; h}^{-1}} \cdot \mathbf{1}\{\left\|X_{h}\left(f^{t}\right)\right\|_{\Sigma_{t ; h}^{-1}} \leq 1\}\big)+  \min\left\{2\gamma_T(\epsilon, \cX), HT\right\}, \notag \\
             & \leq \sum_{t=1}^T\sum_{h=1}^H \underbrace{\left\|W_{h}\left(f^{t}\right)-W_{h}\left(f^{\star}\right)\right\|_{\Sigma_{t; h}} \cdot \min\big\{\left\|X_{h}\left(f^{t}\right)\right\|_{\Sigma_{t ; h}^{-1}}, 1 \big\}}_{\displaystyle \mathrm{(A)}_{t,h}} +  \min\left\{2\gamma_T(\epsilon, \cX), HT\right\},
\#  
where we also use the Cauchy-Schwarz inequality in the first inequality. To bound the first term about $\mathrm{(A)}_{t,h}$, we first expand $\left\|W_{h}\left(f^{t}\right)-W_{h}\left(f^{\star}\right)\right\|_{\Sigma_{t; h}}$ by the definition of $\Sigma_{t;h}$:
    $$
    \begin{aligned}
       \left\|W_{h}\left(f^{t}\right)-W_{h}\left(f^{\star}\right)\right\|_{\Sigma_{t; h}} &=\Big[\epsilon \cdot \left\|W_{h}\left(f^{t}\right)-W_{h}\left(f^{*}\right)\right\|_{2}^{2}+ \sum_{s=1}^{t-1} |\dotprod{W_h(f^t) - W_h(f^*), X_h(f^s)}|^2\Big]^{1/2}\\
       &\leq \sqrt{\epsilon} + \Big[\sum_{s=1}^{t-1} |\dotprod{W_h(f^t) - W_h(f^*), X_h(f^s)}|^2\Big]^{1/2},
    \end{aligned}
    $$
    where we use $\norm{W_{h}\left(f^{t}\right)-W_{h}\left(f^{*}\right)} \leq 1$. It follows that
    $$
        \begin{aligned}
              \sum_{t=1}^T \sum_{h=1}^H \mathrm{(A)}_{t,h} &\leq \sum_{t=1}^T\sum_{h=1}^H\bigg(\sqrt{\epsilon  } + \Big[\sum_{s=1}^{t-1} |\dotprod{W_h(f^t) - W_h(f^*), X_h(f^s)}|^2\Big]^{1/2}\bigg) \cdot \min\big\{\left\|X_{h}\left(f^{t}\right)\right\|_{\Sigma_{t ; h}^{-1}}, 1\big\}\\
             &  \leq \Big[\sum_{t=1}^T\sum_{h=1}^H {\epsilon  }\Big]^{1/2} \cdot \Big[\sum_{t=1}^T \sum_{h=1}^H \min\big\{\left\|X_{h}\left(f^{t}\right)\right\|^2_{\Sigma_{t ; h}^{-1}}, 1\big\}\Big]^{1/2}\\
             &  \quad  + \Big[\sum_{t=1}^T \sum_{h=1}^H \sum_{s=1}^{t-1} |\dotprod{W_h(f^t) - W_h(f^*), X_h(f^s)}|^2\Big]^{1/2} \cdot \Big[\sum_{t=1}^T \sum_{h=1}^H \min\big\{\left\|X_{h}\left(f^{t}\right)\right\|^2_{\Sigma_{t ; h}^{-1}}, 1\big\}\Big]^{1/2},
        \end{aligned}
    $$
    where the second inequality uses Cauchy-Schwarz inequality. We again invoke \eqref{eqn:bilinear_coro} to obtain that  
    $$\begin{aligned}
        &\sum_{t=1}^T \sum_{h=1}^H \mathrm{(A)}_{t,h} \\
        &\qquad  \leq \sqrt{HT\epsilon \cdot \min\{2\gamma_T(\epsilon,\cX), HT\}} + \Big[2\gamma_T(\epsilon,\cX)\sum_{t=1}^T \sum_{h=1}^H \sum_{s=1}^{t-1} |\dotprod{W_h(f^t) - W_h(f^*), X_h(f^s)}|^2\Big]^{1/2},\\
        &\qquad  \leq \sqrt{HT\epsilon \cdot \min\{2\gamma_T(\epsilon,\cX), HT\}} + \Big[2\gamma_T(\epsilon,\cX)\sum_{t=1}^T \sum_{h=1}^H \sum_{s=1}^{t-1} \left|\EE_{x_h \sim \pi_{f^s}, a_h \sim \tilde{\pi}} l_{f^s}(f^t, \zeta_h)\right|^2\Big]^{1/2},
    \end{aligned}
    $$
where the last inequality follows from the second condition of bilinear class. We proceed by plugging the bound of $\mathrm{A}_{t,h}$ into \eqref{eqn:bilinear_tmp}. For Q-type problem, i.e., $\tilde{\pi} = \pi_{f^s}$, we have
    $$
        \begin{aligned}
             & \sum_{t=1}^T\sum_{h=1}^{H} \EE_{\pi_{f^t}} \left[\cE_h\left(f^t, x^t_h, a^t_h\right)\right]\\
             & \leq \Big[2\gamma_T(\epsilon,\cX)\sum_{t=1}^T \sum_{h=1}^H \sum_{s=1}^{t-1} |\EE_{{\pi}_{f^s}} l_{f^s}(f^t, \zeta_h)|^2\Big]^{1/2} +  \min\left\{2\gamma_T(\epsilon, \cX), HT\right\} + \sqrt{HT\epsilon \cdot \min\{2\gamma_T(\epsilon,\cX), HT\}} \\
             & \leq \Big[2\gamma_T(\epsilon,\cX)\sum_{t=1}^T \sum_{h=1}^H \sum_{s=1}^{t-1} |\EE_{{\pi}_{f^s}} l_{f^s}(f^t, \zeta_h)|^2\Big]^{1/2} + 2 \min\left\{2\gamma_T(\epsilon, \cX), HT\right\} + HT\epsilon , \\ 
             &\leq \Big[2\gamma_T(\epsilon,\cX)\sum_{t=1}^T \sum_{h=1}^H \sum_{s=1}^{t-1} \EE_{{\pi}_{f^s}} | \EE_{x_{h+1}| x_h, a_h} l_{f^s}(f^t, \zeta_h)|^2\Big]^{1/2} + 2 \min\left\{2\gamma_T(\epsilon, \cX), HT\right\} + HT\epsilon, 
        \end{aligned}
    $$
        where the second inequality is implied by the AM-GM inequality and the last inequality uses Jensen's inequality. Therefore, for Q-type bilinear class, the GEC is upper bounded by $2\gamma_T(\epsilon, \cX)$ and $\pi_{\exp}(f^s,h) = \pi_{f^s}$. For V-type problem, i.e., $\tilde{\pi} = f^t$, we have
            $$
        \begin{aligned}
             & \sum_{t=1}^T\sum_{h=1}^{H} \EE_{\pi_{f^t}} \left[\cE_h\left(f^t, x^t_h, a^t_h\right)\right]\\
             & \qquad \leq \Big[2\gamma_T(\epsilon,\cX)\sum_{t=1}^T \sum_{h=1}^H \sum_{s=1}^{t-1} |\EE_{x_h \sim {\pi}_{f^s}, a_h \sim \pi_{f^t}} l_{f^s}(f^t, \zeta_h)|^2\Big]^{1/2} + 2 \min\left\{2\gamma_T(\epsilon, \cX), HT\right\} + HT\epsilon, \\ 
             & \qquad \leq \Big[2A \gamma_T(\epsilon,\cX) \sum_{t=1}^T \sum_{h=1}^H \sum_{s=1}^{t-1} \EE_{x_h \sim {\pi}_{f^s}, a_h \sim \mathrm{Unif}(\cA)}| \EE_{x_{h+1}| x_h, a_h} l_{f^s}(f^t, \zeta_h)|^2\Big]^{1/2} \\
             & \qquad \qquad + 2 \min\left\{2\gamma_T(\epsilon, \cX), HT\right\} + HT\epsilon, 
             & 
        \end{aligned}
    $$
    where the last inequality is from importance sampling. Therefore, for V-type bilinear class, if we choose $\pi_{\exp}(f^s,h) = \pi_{f^s} \circ_{h} \mathrm{Unif}(\cA)$, the GEC is upper bounded by $2A \cdot \gamma_T(\epsilon, \cX)$. This finishes the proof of Lemma~\ref{lemma:relationship:bilinear}.
\end{proof}

\subsection{Witness Rank} \label{appendix:witness:rank}
\begin{proof}[Proof of Lemma~\ref{lem:witness_rank}] 
To begin with, we note that for any $(h,x,a,g) \in [H] \times \cS \times \cA \times \cH$, we have
\begin{equation} \label{eqn:witness_hellin}
\begin{aligned}
& \max_{v \in \cV} [(\EE_{x' \sim \PP_{h, g}(\cdot|x_h,a_h)} v(x_h,a_h, x') - \EE_{x' \sim \PP_{h,f^*}(\cdot|x_h,a_h)} v(x_h,a_h, x'))^2] \\
&\qquad \leq \mathrm{TV}\big(\PP_{h,g}(\cdot|x_h,a_h), \PP_{h,f^*}(\cdot|x_h,a_h)\big)^2 \leq 2D_H^2\big(\PP_{h,g}(\cdot|x_h,a_h), \PP_{h,f^*}(\cdot|x_h,a_h)\big),
\end{aligned}
\end{equation}
where $\TV(\cdot,\cdot)$ is the total variation distance and $D_H^2(\cdot, \cdot)$ is the Hellinger divergence. In the following, we first establish the result for the Q-type witness rank. We recall the two conditions of witness rank in Definition~\ref{def:witness}:
$$
\begin{aligned}
     &\max_{v \in \cV_h} \EE_{x_h \sim \pi_f, a_h \sim \tilde{\pi}}[\EE_{x' \sim \PP_{h, g}(\cdot|x_h,a_h)} v(x_h,a_h, x') - \EE_{x' \sim \PP_{h, f^*}(\cdot|x_h,a_h)} v(x_h,a_h, x')] \geq \inner{W_h(g)}{X_h(f)}\\
     & \kappa_{\mathrm{wit}} \cdot \EE_{x_h \sim \pi_f, a_h \sim \tilde{\pi}}[\EE_{x' \sim \PP_{h, g}(\cdot|x_h,a_h)} V_{h+1,g}(x') - \EE_{x' \sim \PP_{h, f^*}(\cdot|x_h,a_h)} V_{h+1, g}(x')] \leq \inner{W_h(g)}{X_h(f)},
\end{aligned}
$$
where $\tilde{\pi}$ is either $\pi_f$ (Q-type) or $\pi_g$ (V-type). According to the first condition, we have
$$
\max_{v \in \cV_h} \EE_{x_h,a_h \sim \pi_f}[\EE_{x' \sim \PP_{h, g} (\cdot|x_h,a_h)} v(x_h,a_h, x') - \EE_{x' \sim \PP_{h,f^*}(\cdot|x_h,a_h)} v(x_h,a_h, x')] \geq \inner{W_h(g)}{X_h(f)}. 
$$
Moreover, by the Bellman equation, we know that $Q_{h,f}(x_h,a_h) = r_h(x_h,a_h) + \EE_{x' \sim \PP_{h, f}(\cdot|x_h,a_h)} V_{h+1, f}(x')$ for any $f \in \cH$. Combining this with the second condition of witness rank, we know that 
\begin{equation}\label{eqn:witness_bellman}
|\EE_{\pi_f}[Q_{h,f}(x_h,a_h) - r_h(x_h,a_h) - \EE_{x' \sim \PP_{h,f^*}(\cdot|x_h,a_h)} V_{h+1,f}(x')]| \leq \frac{1}{\kappa_{\mathrm{wit}}} \inner{W_h(f)}{X_h(f)}.
\end{equation}
Let $\Sigma_{t ; h} = \epsilon I + \sum_{s=1}^{t-1} X_h(f^s) X_h(f^s)^\top$. We can apply the elliptical potential lemma (Lemma~\ref{lemma:elliptical:potential}) to obtain that 
\begin{equation} \label{eqn:witness_potential}
    \sum_{t=1}^T \min\{\norm{X_h(f^t)}_{\Sigma_{t ; h}^{-1}}^2,1\} \leq {2d\log\left(\frac{\epsilon + T}{\epsilon}\right)} := d(\epsilon).
\end{equation}
Then, by the performance difference lemma \citep{kakade2002approximate, jiang@2017}, it follows that
\begin{equation}\label{eqn:witness_performance}
\begin{aligned}
     &\sum_{t=1}^T V_{f^t} - V^{\pi_{f^t}} = \sum_{t=1}^T\sum_{h=1}^{H} \EE_{\pi_{f^t}} \left[\cE_h\left(f^t, x_h, a_h\right) \right]\\
     &\leq  \sum_{t=1}^T \sum_{h=1}^H \min\{\frac{1}{\kappa_{\mathrm{wit}}} |\inner{W_h(f^t) }{X_h(f^t)}|, 1\}\\
    &=  \sum_{t=1}^T \sum_{h=1}^H \min\{\frac{1}{\kappa_{\mathrm{wit}}} |\inner{W_h(f^t) }{X_h(f^t)}|, 1\} \cdot \Big(\mathbf{1}\{\left\|X_{h}\left(f^{t}\right)\right\|_{\Sigma_{t ; h}^{-1}} \leq 1\} + \mathbf{1}\{\left\|X_{h}\left(f^{t}\right)\right\|_{\Sigma_{t ; h}^{-1}} > 1\}\Big),
\end{aligned}
\end{equation}
where the first inequality follows from \eqref{eqn:witness_bellman}. We proceed by invoking \eqref{eqn:witness_potential} and Cauchy-Schwarz inequality to obtain that
$$
\begin{aligned}
    &\sum_{t=1}^T V_{f^t} - V^{\pi_{f^t}}\\
    &\leq \sum_{t=1}^T \sum_{h=1}^H \frac{1}{\kappa_{\mathrm{wit}}} \Big(\norm{W_h(f^t)}_{\Sigma_{t;h}} \cdot \norm{X_h(f^t)}_{\Sigma_{t;h}^{-1}} \cdot \mathbf{1}\{\left\|X_{h}\left(f^{t}\right)\right\|_{\Sigma_{t ; h}^{-1}} \leq 1 \}\Big) +  \min \{d(\epsilon)H, HT\}\\     
    &\leq \sum_{t=1}^T\sum_{h=1}^H \underbrace{\frac{1}{\kappa_{\mathrm{wit}}}\norm{W_h(f^t) }_{\Sigma_{t ; h}} \cdot \min \{\norm{X_h(f^t)}_{\Sigma_{t ; h}^{-1}}, 1\}}_{\displaystyle \mathrm{(A)}_{t,h}} + \min\{d(\epsilon)H, HT\},
\end{aligned}
$$
We now expand $\mathrm{(A)}_{t,h}$ by the definition of $\Sigma_{t;h}$:
$$
\begin{aligned}
      \mathrm{(A)}_{t,h} &= \frac{1}{\kappa_{\mathrm{wit}}} \Big[\epsilon \cdot \norm{W_h(f^t)}^2_2+\sum_{s=1}^{t-1}\left|\inner{W_h(f^t)}{X_h(f^s)}\right|^2\Big]^{1/2} \cdot \min \{\norm{X_h(f^t)}_{\Sigma_{t ; h}^{-1}}, 1\}\\
      &\leq \frac{1}{\kappa_{\mathrm{wit}}} \Big(\sqrt{\epsilon} +  \Big[\sum_{s=1}^{t-1}\left|\inner{W_h(f^t)}{X_h(f^s)}\right|^2\Big]^{1/2} \Big) \cdot \min \{\norm{X_h(f^t)}_{\Sigma_{t ; h}^{-1}}, 1\},
     \end{aligned}
$$
where we use $\sqrt{a + b} \leq \sqrt{a} + \sqrt{b}$. Summing over $t\in [T]$ and $h \in [T]$, it follows that 
$$
\begin{aligned}
     &\sum_{t=1}^T \sum_{h=1}^H \mathrm{(A)}_{t,h} \leq \frac{1}{\kappa_{\mathrm{wit}}} \sum_{t=1}^T \sum_{h=1}^H \Big(\sqrt{\epsilon} +  \Big[\sum_{s=1}^{t-1}\left|\inner{W_h(f^t)}{X_h(f^s)}\right|^2\Big]^{1/2} \Big) \cdot \min \{\norm{X_h(f^t)}_{\Sigma_{t ; h}^{-1}}, 1\},\\
     &\leq \frac{1}{\kappa_{\mathrm{wit}}}\sqrt{HT\epsilon \cdot \min\{d(\epsilon)H, HT\}} + \frac{1}{\kappa_{\mathrm{wit}}}\Big[d(\epsilon)H\sum_{t=1}^T \sum_{h=1}^H \sum_{s=1}^{t-1} |\dotprod{W_h(f^t), X_h(f^s)}|^2\Big]^{1/2},\\
      &\leq \frac{1}{\kappa_{\mathrm{wit}}}\Big[d(\epsilon)H \sum_{t=1}^T \sum_{h=1}^H \sum_{s=1}^{t-1}|\max_{v \in \cV_h} \EE_{x_h^s, a_h^s \sim \pi_{f^s}}[\EE_{x' \sim \PP_{h, f^t}(\cdot|x_h,a_h)} v(x_h,a_h, x') - \EE_{x' \sim \PP_{h,f^*}(\cdot|x_h,a_h)} v(x_h,a_h, x')]|^2\Big]^{1/2} \\
     &\qquad +\frac{1}{\kappa_{\mathrm{wit}}} \sqrt{HT \epsilon \cdot \min\{d(\epsilon)H, HT\}},
\end{aligned}
$$
where the second inequality follows from Cauchy-Schwarz inequality, and the last inequality is due to the first condition of witness rank.  Together with  \eqref{eqn:witness_hellin}, we further obtain 
\$
\sum_{t=1}^T \sum_{h=1}^H \mathrm{(A)}_{t,h} & \le \Big[4dH\cdot\log(\frac{\epsilon+T}{\epsilon})/\kappa_{\mathrm{wit}}^2 \sum_{t=1}^T \sum_{h=1}^H \sum_{s=1}^{t-1} \EE_{x_h^s, a_h^s \sim \pi_{f^s}}  D_H^2\big(\PP_{h,f^t}(\cdot|x_h^s,a_h^s), \PP_{h,f^*}(\cdot|x_h^s,a_h^s) \big)\Big]^{1/2} \\
& \qquad + 2 \min \{4dH\cdot\log(\frac{\epsilon+T}{\epsilon})/\kappa_{\mathrm{wit}}^2, HT\}+HT\epsilon/\kappa_{\mathrm{wit}}^2 .
\$
We now rescale $\epsilon$ to be $\epsilon' = \epsilon/\kappa_{\mathrm{wit}}^2$ and return to \eqref{eqn:witness_performance}:
$$
\begin{aligned}
    &\sum_{t=1}^T V_{f^t} - V^{\pi_{f^t}}\\
    &\qquad \leq \Big[4dH\cdot\log(\frac{\epsilon'\kappa_{\mathrm{wit}}^2 +T}{\epsilon'\kappa_{\mathrm{wit}}^2})/\kappa_{\mathrm{wit}}^2 \sum_{t=1}^T \sum_{h=1}^H \sum_{s=1}^{t-1} \EE_{x_h^s, a_h^s \sim \pi_{f^s}}  D_H^2\big(\PP_{h,f^t}(\cdot|x_h^s,a_h^s), \PP_{h,f^*}(\cdot|x_h^s,a_h^s) \big)\Big]^{1/2} \\
     & \qquad\qquad + 2 \min \{4dH\cdot\log(\frac{\epsilon'\kappa_{\mathrm{wit}}^2+T}{\epsilon'\kappa_{\mathrm{wit}}^2})/\kappa_{\mathrm{wit}}^2, HT\}+HT\epsilon',
\end{aligned}
$$
which implies that for Q-type witness rank, the $\GEC$ is upper bounded by $4dH\cdot\log(1+\frac{T}{\epsilon'\kappa_{\mathrm{wit}}^2})/\kappa_{\mathrm{wit}}^2$ where $d$ is the Q-type witness rank. For V-type witness rank, following the same procedure, we can obtain that 
$$
\begin{aligned}
     &\sum_{t=1}^T V_{f^t} - V^{\pi_{f^t}} \leq  \sum_{t=1}^T \sum_{h=1}^H \min\{\frac{1}{\kappa_{\mathrm{wit}}} |\inner{W_h(f^t)}{X_h(f^t)}|, 1\}\\
     &\leq \Big[2d(\epsilon)H/\kappa_{\mathrm{wit}}^2 \sum_{t=1}^T \sum_{h=1}^H \sum_{s=1}^{t-1} \EE_{x_h^s\sim \pi_{f^s}, a_h^s \sim \pi_{f^t}}  D_H^2\big(\PP_{h,f^t}(\cdot|x_h^s,a_h^s), \PP_{h,f^*}(\cdot|x_h^s,a_h^s) \big)\Big]^{1/2}\\
     &\qquad + 2 \min \{d(\epsilon)H/\kappa_{\mathrm{wit}}^2, HT\}+HT\epsilon\\
     &\leq \Big[4dAH\cdot \log(\frac{\epsilon'\kappa_{\mathrm{wit}}^2+T}{\epsilon'\kappa_{\mathrm{wit}}^2})/\kappa_{\mathrm{wit}}^2 \sum_{t=1}^T \sum_{h=1}^H \sum_{s=1}^{t-1} \EE_{x_h^s\sim \pi_{f^s}, a_h^s \sim \mathrm{Unif}(\cA)} D_H^2 \big(\PP_{h,f^t}(\cdot|x_h^s,a_h^s), \PP_{h,f^*}(\cdot|x_h^s,a_h^s) \big)\Big]^{1/2}\\
     &\qquad + 2 \min \{4dAH\cdot\log(\frac{\epsilon'\kappa_{\mathrm{wit}}^2+T}{\epsilon'\kappa_{\mathrm{wit}}^2})/\kappa_{\mathrm{wit}}^2, HT\}+HT\epsilon',
\end{aligned}
$$
where we rescale $\epsilon$ to be $\epsilon' = \epsilon/\kappa_{\mathrm{wit}}^2$ and use importance sampling in the last step. Hence, for V-type witness rank, the $\GEC$ is upper bounded by $4dAH\cdot\log(1+\frac{T}{\epsilon'\kappa_{\mathrm{wit}}^2})/\kappa_{\mathrm{wit}}^2$. Combining results for Q-type and V-type witness rank, we conclude the proof of Lemma~\ref{lem:witness_rank}.
\end{proof}

\subsection{Generalized Completeness} \label{appendix:completeness}
The generalized completeness assumption is a generalization of the Bellman completeness. The main motivation to introduce the completeness assumption is to cancel the sampling variance from the state transition. In specific, we have
$$
\begin{aligned}
\EE_{x_{h+1}|x_h,a_h} l_{f'}((g_h, f_{h+1}), \zeta_h)^2 &= (\EE_{x_{h+1}|x_h,a_h} l_{f'}((g_h, f_{h+1}), \zeta_h))^2\\
&\quad + \EE_{x_{h+1}|x_h,a_h} (l_{f'}((g_h, f_{h+1}), \zeta_h) - \EE_{x_{h+1}|x_h,a_h} l_{f'}((g_h, f_{h+1}), \zeta_h))^2,
\end{aligned}
$$
where the second term is the conditional sampling variance. Due to the special denominator in \eqref{eqn:posterior_model_free_mdp}, we can replace $l_{f'}((g_h, f_{h+1}), \zeta_h)^2$ with $l_{f'}((g_h, f_{h+1}), \zeta_h)^2 - l_{f'}((\cP_h f_{h+1}, f_{h+1}), \zeta_h)^2$ in the theoretical analysis where $\cP_h f_{h+1} \in \cH_h$. Thus, as long as the expectation of the second term is equal to the conditional variance, we can cancel the variance in the theoretical analysis. To this end, we require
$$
l_{f'}((g_h, f_{h+1}), \zeta_h) - \EE_{x_{h+1}|x_h,a_h} l_{f'}((g_h, f_{h+1}), \zeta_h) = l_{f'}((\cP_h f_{h+1}, f_{h+1}), \zeta_h).
$$
In this subsection, we show that the linear Bellman complete MDP and linear mixture MDP satisfy this condition naturally.

\paragraph{Linear Bellman Completeness.} We consider a hypothesis class $\cH$ of linear functions with respect to some known feature $\phi:\cS \times \cA \to \cV$ where $\cV$ is a Hilbert space.

\begin{example}[Linear Bellman Complete] We say $\cH$ is linear Bellman complete with respect to the underlying MDP if $f^* = (\theta_{1,f^*}, \cdots, \theta_{H, f^*}) \in \cH$ and there exists $\cT_h: \cV \to \cV$ such that for all $f=(\theta_{1,f},\cdots,\theta_{H,f}) \in \cH$ and all $h \in [H]$, we have
\begin{equation} \label{eqn:b_complete}
\cT_h(\theta_{h+1,f})^\top \phi(x,a) = R_h(x,a) + \EE_{x' \sim \PP_{h, f^*}(\cdot|x,a)} \max_{a'\in\cA} \theta^\top_{h+1,f}\phi(x',a').
\end{equation}
\end{example}
Let $g_h \in \cH_h$ and $f_{h+1} \in \cH_{h+1}$, respectively. We consider the following loss function.
$$l_{f'}((g_h, f_{h+1}), \zeta_h) := \theta^\top_{h,g} \phi(x_h,a_h) - r_h - \max_{a' \in \cA} \theta_{h+1,f}^\top \phi(x_{h+1},a').$$
It follows that 
$$
\begin{aligned}
&l_{f'}((g_h, f_{h+1}), \zeta_h) - \EE_{x_{h+1} \sim \PP_{h, f^*}(\cdot|x_h,a_h)} l_{f'}((g_h, f_{h+1}), \zeta_h)\\
& \qquad = \Big( \theta^\top_{h,g} \phi(x_h,a_h) - r_h - \max_{a' \in \cA} \theta_{h+1,f}^\top \phi(x_{h+1},a') \Big) \\
& \qquad \qquad - \Big(\theta^\top_{h,g} \phi(x_h,a_h) - r_h - \EE_{x_{h+1} \sim \PP_{h, f^*}(\cdot|x_h,a_h)}\max_{a' \in \cA} \theta_{h+1,f}^\top \phi(x_{h+1},a')\Big)\\
& \qquad = \cT_h(\theta_{h+1, f})^\top \phi(x_h,a_h) - r_h - \max_{a' \in \cA} \theta^\top_{h+1,f} \phi(x_{h+1}, a')\\
& \qquad = l_{f'}((\cT_h(\theta_{h+1, f}), f_{h+1}), \zeta_h).
\end{aligned}
$$
Therefore, we conclude that $\cP_h = \cT_h$ defined in \eqref{eqn:b_complete}.

\paragraph{Linear Mixture MDP.} 

Linear mixture MDP \citep{ayoub2020model,cai2020provably,modi2020sample} is defined as follows.

\begin{definition} \label{def:linear_mixture} We say an MDP is a liner mixture model if there exists (known) feature $\phi: \cS \times \cA \times \cS \to \cV$ and $\psi: \cS \times \cA \to \cV$; and (unknown) $\theta^* \in \cV$ for some Hilbert space $\cV$ and for all $h \in [H]$ and $(x,a,x') \in \cS \times \cA\times \cS$, we have
    \begin{equation}
        \PP_{h}(x^{\prime} \mid x, a)=\langle\theta_{h}^{*}, \phi(x, a, x^{\prime})\rangle, \qquad  r_h(x, a)=\langle\theta_{h}^{*}, \psi(x, a)\rangle.
    \end{equation}
\end{definition}

Suppose that $f = (\theta_{1,f}, \cdots, \theta_{H,f}) \in \cH$ so that $\theta_{h,f} \in \cV$ for all $h \in [H]$. By the definition of linear mixture MDP and Bellman equation, we know that 
$$
Q_{h,f}(x,a) = R_h(x,a) + \EE_{x' \sim \PP_{h,f}(\cdot|x_h,a_h)} V_{h+1,f}(x') =  \theta_{h,f}^\top \Big(\psi(x,a) + \sum_{x' \in \cS} \phi(x,a,x') V_{h+1,f}(x')\Big).
$$
With $\zeta_h = (x_h,a_h,r_h,x_{h+1})$, we adopt the following loss function:
$$
l_{f'}((g_h, f_{h+1}), \zeta_h) := \theta_{h,g}^\top \Big[\psi(x_h, a_h) +  \sum_{x' \in \cS} \phi(x_h,a_h,x') V_{h+1,f'}(x') \Big] - r_h - V_{h+1, f'}(x_{h+1}).
$$
Let $f^*=(\theta_{1,f^*}, \cdots, \theta_{H, f^*}) \in \cH$ be the true model. Then, we have
$$
\begin{aligned}
&l_{f'}((g_h, f_{h+1}), \zeta_h) - \EE_{x' \sim \PP_{h, f^*}(\cdot|x_h,a_h)} l_{f'}((g_h, f_{h+1}), \zeta_h)\\
& \qquad =\theta_{h,g}^\top \Big[\psi(x_h, a_h) +  \sum_{x' \in \cS} \phi(x_h,a_h,x') V_{h+1,f'}(x') \Big] - r_h - V_{h+1, f'}(x_{h+1}) \\
&\qquad \qquad - (\theta_{h,g} - \theta_{h, f^*})^\top \Big[\psi(x_h, a_h) +  \sum_{x' \in \cS} \phi(x_h,a_h,x') V_{h+1,f'}(x') \Big]\\
& \qquad = \theta_{h, f^*}^\top \Big[\psi(x_h, a_h) +  \sum_{x' \in \cS} \phi(x_h,a_h,x') V_{h+1,f'}(x') \Big] - r_h - V_{h+1,f'}(x_{h+1})\\
& \qquad = l_{f'}((f^*_h, f_{h+1}), \zeta_h),
\end{aligned}
$$
where the first equality is because the transition follows from the true model $f^*$. Therefore, we conclude that $(\cP_h f_{h+1}) = f^*_h$ for all $f_{h+1} \in \cH_{h+1}$.

\section{Proofs for Examples of Partially Observable Interactive Decision Making}

We will use $(\bfm, \Mb, \bfq, \bfq_0)$ and $(\bfm^t, \Mb^t, \bfq^t, \bfq^t_0)$ to denote operators for models $f^*$ and $f^t$, respectively. We also use the shorthand $\pi^t = \pi_{f^t}$. See Appendix~\ref{appendix:notation} for more explanations of notations.

\subsection{Generalized Regular PSR} \label{appendix:general:psr}

\begin{proof}[Proof of Lemma \ref{lemma:general:regular:psr}] Our proof consists of three steps. In the first step, we upper bound the prediction error by the out-of-sample operator errors. In the second step, we connect the training error with in-sample operator errors. Finally, we conclude the proof by bridging these two types of operator errors via new $\ell_2$ eluder techniques.

    \paragraph{Step 1: Prediction Error Decomposition}

    \begin{lemma}[Prediction Error Decomposition Lemma] \label{lemma:decomp:test}
        Let $\cH$ be a set of $\alpha$-generalized regular PSRs defined in Definition~\ref{def:general:regular:psr}.  For any $\{f^t, \pi^t = \pi_{f^t}\}_{t \in [T]}$, it holds that
        \$
        \sum_{t = 1}^T V_{f^t} - V_{f^*}^{\pi^t} &\le \frac{1}{\alpha} \sum_{t = 1}^T \Big( \sum_{h = 1}^H \sum_{\tau_h} \Big\|  \big( \Mb_h^t(o_h, a_h) - \Mb_h(o_h, a_h) \big) \bfq(\tau_{h-1}) \Big\|_1 \cdot \pi^t(\tau_h)
        + \big\|\bfq_0^t - \bfq_0 \big\|_1  \Big).
        \$
    \end{lemma}

    \begin{proof}[Proof of Lemma~\ref{lemma:decomp:test}]
        Fix $t \in [T]$, we have
        \$
        V_{f^t} - V_{f^*}^{\pi^t} &=  \sum_{\tau_H}  \big( \PP_{f^t}^{\pi^t}(\tau_H) - \PP_{f^*}^{\pi^t}(\tau_H) \big) \cdot r(\tau_H)  \\
        & \le \sum_{\tau_H} \big| \PP_{f^t}^{\pi^{t}}(\tau_H) - \PP_{f^*}^{\pi^t}(\tau_H) \big| \\
        & = \sum_{\tau_H} \Big| \Mb_{H}^t(o_H, a_H) \cdots \Mb_1^t(o_1, a_1) \bfq_0^t - \Mb_{H}(o_H, a_H) \cdots \Mb_1(o_1, a_1) \bfq_0 \Big| \cdot \pi^t(\tau_H),
        \$
        where $\Mb_H^t(o_H, a_H)$ is a vector because $|\cU_{H+1}|=1$, the first inequality uses the fact that $r(\tau_H) : =\sum_{h = 1}^H r_h(o_h, a_h) \le 1$ and the last equality uses \eqref{eq:prob:traj}. By the triangle inequality, we further obtain
        \#\label{eq:112}
        V_{f^t} - V_{f^*}^{\pi^t} &\le  \sum_{h = 1}^H \sum_{\tau_H} \Big| \Mb_H^t(o_H, a_H) \cdots \Mb_{h+1}^t(o_{h+1}, a_{h+1}) \big( \Mb_h^t(o_h, a_h) - \Mb_h(o_h, a_h) \big) \bfq(\tau_{h-1}) \Big| \cdot \pi^t(\tau_H) \notag\\
        & \qquad \quad  + \sum_{\tau_H} \Big| \Mb_H^t(o_H, a_H) \cdots \Mb_{1}^t(o_{1}, a_{1})  (\bfq_0^t - \bfq_0) \Big|  \cdot \pi^t(\tau_H)   \notag \\
        & =  \sum_{h = 1}^H \sum_{\tau_H} \Big| \bfm^t(o_{h+1:H}, a_{h+1:H}) \big( \Mb_h^t(o_h, a_h) - \Mb_h(o_h, a_h) \big) \bfq(\tau_{h-1}) \Big| \cdot \pi^t(\tau_H)  \notag \\
        & \qquad \quad  + \sum_{\tau_H} \Big| \bfm^t(o_{1:H}, a_{1:H})   (\bfq_0^t - \bfq_0) \Big| \cdot \pi^t(\tau_H)  ,
        \#
        where the last equality follows from the definition of $\bfm^t$ in \eqref{eq:def:m}. By the definition of the generalized regular PSR in Definition~\ref{def:general:regular:psr}, we further have
        \# \label{eq:113}
        & \sum_{h = 1}^H \sum_{\tau_H} \Big| \bfm^t(o_{h+1:H}, a_{h+1:H}) \big( \Mb_h^t(o_h, a_h) - \Mb_h(o_h, a_h) \big) \bfq(\tau_{h-1}) \Big| \cdot \pi^t(\tau_H) \notag \\
        & \qquad \le \frac{1}{\alpha} \sum_{h = 1}^H \sum_{\tau_h} \Big\Vert \big( \Mb_h^t(o_h, a_h) - \Mb_h(o_h, a_h) \big) \bfq(\tau_{h-1}) \Big\Vert_1 \cdot \pi^t(\tau_h),
        \#
        where the inequality uses Condition 1 in Definition~\ref{def:general:regular:psr} with $\bfx = (\Mb_h^t(o_h, a_h) - \Mb_h(o_h, a_h) ) \bfq(\tau_{h-1})$ for any $h \in [H]$. Similarly, we obtain
        \# \label{eq:114}
        \sum_{\tau_H} \Big| \bfm^t(o_{1:H}, a_{1:H})   (\bfq_0^t - \bfq_0) \Big| \cdot \pi^t(\tau_H) \le \frac{1}{\alpha} \big\| \bfq_0^t - \bfq_0 \big\|_1,
        \#
        where the inequality follows from Condition 1 in Definition~\ref{def:general:regular:psr} with $\bfx = \bfq_0^t - \bfq_0$.
        Plugging \eqref{eq:113} and \eqref{eq:114} into \eqref{eq:112} and then taking summation over $t \in [T]$ finish the proof of Lemma~\ref{lemma:decomp:test}.
    \end{proof}

    \paragraph{Step 2: Training Error Decomposition}

    \begin{lemma}[Training Error Decomposition Lemma] \label{lemma:decomp:train}
        For any $\{f^t, \pi^t = \pi_{f^t}\}_{t \in [T]}$, it holds that
        \$
        & \sum_{t = 1}^{T} \sum_{h = 1}^H \sum_{s = 1}^{t - 1} \sum_{\tau_{h-1} \sim \PP^{\pi^s}} \Big( \sum_{o_h \in \cO, a_h \in \cA} \Big\| \big( \Mb_h^t(o_h, a_h) - \Mb_h(o_h, a_h) \big) \bar{\bfq}(\tau_{h-1}) \Big\|_1 \Big)^2 \\
        & \qquad \lesssim  \frac{A^3 U_A^4}{\alpha^2} \sum_{t = 1}^T \sum_{s = 1}^{t - 1} \sum_{h = 0}^{H - 1} D_H^2 \Big(\PP_{f^t}^{\pi_{\exp}(f^s, h)}, \PP_{f^*}^{\pi_{\exp}(f^s, h)} \Big) ,
        \$
        where $\lesssim$ omits an absolute constant and $\pi_{\exp}(f^s, h) =\pi_{f^s} \circ_{h} \operatorname{Unif}(\mathcal{A}) \circ_{h+1} \operatorname{Unif}\left(\mathcal{U}_{A, h+1}\right) $ is the exploration policy defined in Example~\ref{example:policy:psr}.
    \end{lemma}

    \begin{proof}[Proof of Lemma \ref{lemma:decomp:train}]
         Fix $t \in [T]$. For any $h \in [H]$, using the inequality $(a + b)^2 \le 2a^2 + 2b^2$, we have
           \begin{equation}
            \begin{aligned} \label{eq:22116}
                 & \sum_{s = 1}^{t - 1} \sum_{\tau_{h-1} \sim \PP^{\pi^s}} \Big( \sum_{o_h \in \cO, a_h \in \cA} \Big\| \big(\Mb_h^t(o_h, a_h) - \Mb_h(o_h, a_h) \big) \bar{\bfq}(\tau_{h-1}) \Big\|_1   \Big)^2  \\
                 & \qquad \le \underbrace{2\sum_{s = 1}^{t-1} \sum_{\tau_{h-1} \sim \PP^{\pi^s}} \Big( \sum_{o_h \in \cO, a_h \in \cA} \Big\| \big(\Mb_h^t(o_h, a_h) \bar{\bfq}^t(\tau_{h-1}) - \Mb_h(o_h, a_h) \bar{\bfq}(\tau_{h-1}) \Big\|_1  \Big)^2}_{\displaystyle{\rm (i)}} \\
                 & \qquad \quad + \underbrace{2 \sum_{s = 1}^{t - 1} \sum_{\tau_{h-1} \sim \PP^{\pi^s}} \Big( \sum_{o_h \in \cO, a_h \in \cA} \Big\| \Mb_h^t(o_h, a_h) \big( \bar{\bfq}^t(\tau_{h-1}) - \bar{\bfq}(\tau_{h-1}) \big) \Big\|_1 \Big)^2.}_{\displaystyle{\rm (ii)}}
            \end{aligned}
        \end{equation}

        \noindent{\bf Term (i).} We first consider the case $h \le H - 1$. Note that $\Mb_h^t(o_h, a_h) \bar{\bfq}^t(\tau_{h-1}) = [\PP_{f^t}(o_h, a_h, u \mid \tau_{h-1})]_{u \in \cU_{h+1}}$ and $\Mb_h(o_h, a_h) \bar{\bfq}(\tau_{h-1}) = [\PP_{f^*}(o_h, a_h, u \mid \tau_{h-1})]_{u \in \cU_{h+1}}$, we have
        \# \label{eq:22117}
        {\displaystyle{\rm (i)}} &=   2\sum_{s = 1}^{t-1} \sum_{\tau_{h-1} \sim \PP^{\pi^s}} \Big( \sum_{o_h \in \cO, a_h \in \cA} \sum_{(\bfo, \bfa) \in \cU_{h+1} } \big| \PP_{f^t}(o_h, \bfo \mid \tau_{h-1}, \mathrm{do}(a_h, \bfa)) - \PP_{f^*}(o_h, \bfo \mid \tau_{h-1}, \mathrm{do}(a_h, \bfa)) \big|   \Big)^2 \notag\\
        & \le 2 A^2 U_A^2 \sum_{s = 1}^{t-1} \sum_{\tau_{h-1} \sim \PP^{\pi^s}} \Big\| \PP_{f^t}^{\mathrm{Unif}(\cA) \circ_{h+1} \mathrm{Unif}(\cU_{A,h+1}) }(\cdot \mid \tau_{h-1}) - \PP_{f^*}^{\mathrm{Unif}(\cA) \circ_{h+1} \mathrm{Unif}(\cU_{A,h+1}) }(\cdot \mid \tau_{h-1})  \Big\|_1^2  \notag\\
        & \le 16A^2 U_A^2 \sum_{s = 1}^{t-1}  \sum_{\tau_{h-1} \sim \PP^{\pi^s}} D_H^2 \Big( \PP_{f^t}^{\mathrm{Unif}(\cA) \circ_{h+1} \mathrm{Unif}(\cU_{A,h+1}) }(\cdot \mid \tau_{h-1}), \PP_{f^*}^{\mathrm{Unif}(\cA) \circ_{h+1} \mathrm{Unif}(\cU_{A,h+1}) }(\cdot \mid \tau_{h-1})  \Big) \notag \\
        & \le 64 A^2 U_A^2 \sum_{s = 1}^{t-1} D_H^2 \Big( \PP_{f^t}^{\pi_{\exp}(f^s, h)}, \PP_{f^*}^{\pi_{\exp}(f^s, h)} \Big),
        \#
        where $(\PP_{f^t}^{\mathrm{Unif}(\cA) \circ_{h+1} \mathrm{Unif}(\cU_{A,h+1}) }(\cdot \mid \tau_{h-1}), \PP_{f^*}^{\mathrm{Unif}(\cA) \circ_{h+1} \mathrm{Unif}(\cU_{A,h+1}) }(\cdot \mid \tau_{h-1}))$ are conditional distributions over $\tau_{h:H}$ and $(\PP_{f^t}^{\pi_{\exp}(f^s, h)}, \PP_{f^*}^{\pi_{\exp}(f^s, h)})$ are distributions over $\tau_{1:H}$.
        Here the first inequality uses importance sampling, the second inequality uses Lemma~\ref{lem:hellinger:1}, and the last inequality is implied by Lemma~\ref{lem:hellinger:1} and the fact that $\pi_{\exp}(f^s, h) =\pi_{f^s} \circ_{h} \operatorname{Unif}(\mathcal{A}) \circ_{h+1} \operatorname{Unif}\left(\mathcal{U}_{A, h+1}\right)$.
        For the case where $h = H$, we have 
        \# \label{eq:22118}
        {\displaystyle{\rm (i)}} 
        & =  2\sum_{s = 1}^{t-1} \sum_{\tau_{H-1} \sim \PP^{\pi^s}} \Big( \sum_{o_H \in \cO, a_H \in \cA}  \big| \PP_{f^t}(o_H \mid \tau_{H-1}, \mathrm{do}(a_H)) - \PP_{f^*}(o_H \mid \tau_{H-1}, \mathrm{do}(a_H)) \big|   \Big)^2 \notag \\
        & \le  2 A^2 \sum_{s = 1}^{t-1} \sum_{\tau_{H-1} \sim \PP^{\pi^s}} \Big\| \PP_{f^t}^{\mathrm{Unif}(\cA)}(\cdot \mid \tau_{H-1})  - \PP_{f^*}^{\mathrm{Unif}(\cA)}( \cdot \mid \tau_{H-1})  \Big\|_1^2 \notag \\
        & \le 64 A^2 \sum_{s = 1}^{t-1} D_H^2\Big( \PP_{f^t}^{\pi^s \circ_{H} \mathrm{Unif}(\cA)}, \PP_{f^*}^{\pi^s \circ_{H} \mathrm{Unif}(\cA)}  \Big) \notag \\
        & \le 64 A^3 U_A \sum_{s = 1}^{t-1} D_H^2\Big( \PP_{f^t}^{\pi_{\exp}(f^s, H - 1)}, \PP_{f^*}^{\pi_{\exp}(f^s, H - 1)}  \Big),
        \#
        where the first equality uses the facts that $\Mb_h^t(o_H, a_H) \bar{\bfq}^t(\tau_{H-1}) = \PP_{f^t}(o_H \mid \tau_{H-1}, \mathrm{do}(o_H))$ and $\Mb_h(o_H, a_H) \bar{\bfq}(\tau_{H-1}) = \PP_{f^*}(o_H \mid \tau_{H-1}, \mathrm{do}(o_H))$, the second inequality can be obtained by the same derivation of \eqref{eq:22117}, and the last inequality uses importance sampling and the fact that $\pi_{\exp}(f^s, H - 1) =\pi_{f^s} \circ_{H - 1} \operatorname{Unif}(\mathcal{A}) \circ_{H} \operatorname{Unif}\left(\mathcal{U}_{A, H}\right)$.

        \noindent{\bf Term (ii).} By the definition of generalized regular PSR in Definition~\ref{def:general:regular:psr}, we have
        \$
        {\displaystyle{\rm (ii)}} & \le \frac{2 A^2 U_A^2}{\alpha^2} \sum_{s = 1}^{t - 1} \sum_{\tau_{h-1} \sim \PP^{\pi^s}}  \big\|  \bar{\bfq}^t(\tau_{h-1}) - \bar{\bfq}(\tau_{h-1})  \big\|_1^2  \\
        & =  \frac{2 A^2 U_A^2}{\alpha^2} \sum_{s = 1}^{t - 1} \sum_{\tau_{h-1} \sim \PP^{\pi^s}} \Big( \sum_{(\bfo, \bfa) \in \cU_{h} } \big| \PP_{f^t}(\bfo \mid \tau_{h-1}, \mathrm{do}(\bfa)) - \PP_{f^*}(\bfo \mid \tau_{h-1}, \mathrm{do}(\bfa)) \big|  \Big)^2  \\
        & \le \frac{2 A^2 U_A^4}{\alpha^2} \sum_{s = 1}^{t - 1} \sum_{\tau_{h-1} \sim \PP^{\pi^s}} \Big\|  \PP_{f^t}^{\mathrm{Unif}(\cU_{A, h}) } (\cdot \mid \tau_{h-1}) - \PP_{f^*}^{\mathrm{Unif}(\cU_{A, h})}(\cdot \mid \tau_{h-1})   \Big\|_1^2 ,
        \$
        where the first equality uses the probabilistic meanings of $\bar{\bfq}^t(\tau_{h-1})$ and $\bar{\bfq}(\tau_{h-1})$ and the last inequality follows from importance sampling. By the same derivation of \eqref{eq:22117}, we further obtain
        \# \label{eq:22119}
        {\displaystyle{\rm (ii)}} & \le \frac{64 A^2 U_A^4}{\alpha^2} \sum_{s = 1}^{t - 1} D_H^2 \Big( \PP_{f^t}^{ \pi^s \circ_{h} \mathrm{Unif}(\cU_{A, h}) }, \PP_{f^*}^{\pi^s \circ_{h} \mathrm{Unif}(\cU_{A, h})} \Big) \notag\\
        & \le \frac{64 A^3 U_A^4}{\alpha^2} \sum_{s = 1}^{t - 1} D_H^2 \Big( \PP_{f^t}^{\pi_{\exp}(f^s, h-1)}, \PP_{f^*}^{\pi_{\exp}(f^s, h-1)} \Big) ,
        \#
        where the last inequality uses importance sampling.

        Plugging \eqref{eq:22117}, \eqref{eq:22118}, and \eqref{eq:22119} into \eqref{eq:22116} and taking summation over $(t, h) \in [T] \times [H]$  give that
         \$
        & \sum_{t = 1}^{T} \sum_{h = 1}^H \sum_{s = 1}^{t - 1} \sum_{\tau_{h-1} \sim \PP^{\pi^s}} \Big( \sum_{o_h \in \cO, a_h \in \cA} \Big\| \big( \Mb_h^t(o_h, a_h) - \Mb_h(o_h, a_h) \big) \bar{\bfq}(\tau_{h-1}) \Big\|_1 \Big)^2 \\
        & \qquad \lesssim  \frac{A^3 U_A^4}{\alpha^2} \sum_{t = 1}^T \sum_{s = 1}^{t - 1} \sum_{h = 0}^{H - 1} D_H^2 \Big(\PP_{f^t}^{\pi_{\exp}(f^s, h)}, \PP_{f^*}^{\pi_{\exp}(f^s, h)} \Big) ,
        \$
        which concludes the proof of Lemma~\ref{lemma:decomp:train}.
    \end{proof}

    \paragraph{Step 3: Bridging Prediction Error and Training Error via New $\ell_2$ Eluder Techniques}

    \vspace{1pt}
     Before our proof, we restate the definitions of $(\KK_{h}, \VV_{h})$ in \eqref{eq:k2}. Since $\bar{\DD}_h \in \RR^{|\cU_{h+1}| \times O^h A^h}$ has rank $d_{\psr, h}$, we know there exists $\KK_{h} = [\bar{\bfq}(\tau_h^1), \cdots, \bar{\bfq}(\tau_h^{d_{\psr,h}})] \in \RR^{|\cU_{h+1}| \times d_{\psr, h}}$ with $\rank({\KK}_h) = d_{\psr, h}$. This further implies that each column of $\bar{\DD}_h$ is a linear combination of the columns in $\KK_h$. Equivalently, for any length-$h$ history $\tau_h$, there exists a vector $\bfv_{\tau_h} \in \RR^{d_{\psr,h}}$ such that 
    \# \label{eq:k11}
    \bar{\bfq}({\tau_h}) = \KK_h \bfv_{\tau_h}.
     \#
     Let $\VV_h = [\bfv_{\tau_h}]_{\tau_h \in (\cO \times \cA)^h} \in \RR^{d_{\psr,h} \times O^h A^h}$, \eqref{eq:k11} further implies that
     \# \label{eq:k22}
     \bar{\DD}_h = \KK_h \VV_h . 
     \#
     In the following, we choose $(\KK_{h}, \VV_{h})$ as
     \$
     (\KK_h, \VV_{h}) = \argmin_{(\KK_{h}, \VV_{h})} \big\{ \|\KK_h\|_1 \cdot \|\VV_h\|_1 \mid \bar{\DD}_h = \KK_h \VV_h, \KK_h \in 
   \RR^{|\cU_{h+1}| \times d_{\psr,h}}, \VV_h \in \RR^{ 
   d_{\psr, h} \times |\cO|^{h}|\cA|^h } \big\}.
     \$
     Together with the definition of $\delta_{\psr}$ in \eqref{eq:def:delta}, we have
     \# \label{eq:11800}
     \|\KK_h\|_1 \cdot \|\VV_{h} \|_1 \le \delta_{\psr}, \qquad \forall h \in [H].
     \#
    Back to our proof, by Step 1 (Lemma \ref{lemma:decomp:test}), we have for any $\{f^t, \pi^t = \pi_{f^t}\}_{t \in [T]}$
    \# \label{eq:119}
    &\sum_{t = 1}^T  V_{f^t} - V^{\pi^t}    \notag \\
    & \qquad \le \frac{1}{\alpha} \sum_{t = 1}^T \Big( \sum_{h = 1}^{H} \sum_{\tau_h} \Big\| \big( \Mb_h^t(o_h, a_h) - \Mb_h(o_h, a_h) \big) \bfq(\tau_{h-1}) \Big\|_1 \cdot \pi^t(\tau_h) + \big\| \bfq_0^t - \bfq_0 \|_1 \Big).
    \#
    Fix $h \in [H]$. Denote $\Xb_u$ as the $u$-th row of the matrix $\Xb$. For any $(t, u, o, a) \in [T] \times \cU_{h+1} \times \cO \times \cA$, let $w_{t,u,o,a} := [(\Mb_h^t(o,a) - \Mb_h(o,a)) \KK_{h-1}]_u$. Meanwhile, 
    we define $[x_{\tau_{h-1}}]_{\tau_{h-1} \in (\cO \times \cA)^{h-1} } := [ \KK_{h-1}^\dagger \bar{\bfq}(\tau_{h-1}) ]_{\tau_{h-1} \in (\cO \times \cA)^{h-1}}$. By the definitions of $\bfq$ and $\bar{\bfq}$, we know $\bfq(\tau_{h-1}) = [\PP(t_h, \tau_{h-1})]_{t_h \in \cU_h} = [\PP(t_h \mid \tau_{h-1}) \times \PP(\tau_{h-1})]_{t_h \in \cU_h} = \bar{\bfq}(\tau_{h-1}) \cdot \PP(\tau_{h-1})$. Hence, for any $t \in [T]$, we have
    \# \label{eq:1192}
    &   \sum_{\tau_h} \| ( \Mb_h^t(o_h, a_h) - \Mb_h(o_h, a_h) ) \bfq(\tau_{h-1}) \|_1 \cdot \pi^t(\tau_h) \notag\\
    & \qquad \le  \sum_{\tau_h} \| ( \Mb_h^t(o_h, a_h) - \Mb_h(o_h, a_h) ) \bar{\bfq}(\tau_{h-1}) \|_1 \cdot \PP^{\pi^t}(\tau_{h-1}) \notag \\
    & \qquad =  \sum_{\tau_{h-1} \sim \PP^{\pi^t}} \sum_{u \in \cU_{h+1}} \sum_{o \in \cO, a \in \cA} |w_{t, u, o, a}^\top x_{\tau_{h-1}}|.
    \#
    Plugging \eqref{eq:1192} into \eqref{eq:119}, together with the fact that $V_{f^t} - V^{\pi_{f^t}} \le 1$ for any $t \in [T]$, we have
    \# \label{eq:11921}
    \sum_{t = 1}^T V_{f^t} - V^{\pi^t} \le \frac{1}{\alpha} \sum_{t = 1}^T \bigg(\sum_{h = 1}^H \Big( \alpha \land \sum_{\tau_{h-1} \sim \PP^{\pi^t}} \sum_{u \in \cU_{h+1}} \sum_{o \in \cO, a \in \cA} |w_{t, u, o, a}^\top x_{\tau_{h-1}}| \Big) + \big\| \bfq_0^t - \bfq_0 \|_1 \bigg).
    \#
    Meanwhile, we have the following lemma.
    \begin{lemma} \label{lemma:120}
        It holds that
        \begin{align*}
            \left\{\begin{array}{l} \sum_{\tau_{h-1} \sim \PP^{\pi^t}} \|x_{ \tau_{h-1}}\|_2^2  \le \|\VV_{h-1}\|_1^2, \\ \sum_{u \in \cU_{h+1}, o \in \cO, a \in \cA } \|w_{t,u,o,a}\|_2 \le  \frac{2AU_A d_{\psr, h - 1}}{\alpha} \big\|\KK_{h-1} \big\|_1.
            \end{array}\right.
        \end{align*}
        Here the definitions of $(\VV_{h-1}, \KK_{h-1})$ can be founded in \eqref{eq:k22}.
    \end{lemma}

    \begin{proof}[Proof of Lemma~\ref{lemma:120}]
        By the definition of $x_{\tau_{h-1}}$ that $x_{\tau_{h-1}} = \KK_{h-1}^\dagger \bar{\bfq}(\tau_{h-1})$, we have
        \$
        \max_{\tau_{h-1}} \|x_{\tau_{h-1}}\|_2 =  \max_{\tau_{h-1}} \| \KK_{h-1}^\dagger \bar{\bfq}(\tau_{h-1}) \|_2 = \max_{\tau_{h-1}} \| \bfv_{\tau_{h-1}} \|_2 \le  \max_{\tau_{h-1}} \| \bfv_{\tau_{h-1}} \|_1 = \|\VV_{h-1}\|_1,
        \$
        where the second equality is implied by \eqref{eq:k11}, the second inequality is obtained by the fact that $\|\cdot\|_2 \le \|\cdot\|_1$, and the last equality is implied by the fact that $\VV_{h-1} = [\bfv_{\tau_{h-1}}]_{\tau_{h-1} \in (\cO \times \cA)^{h-1}}$. The above inequality yields
        \# \label{eq:9992}
        \sum_{\tau_{h-1} \sim \PP^{\pi^t}} \|x_{\tau_{h-1}}\|_2^2 \le \|\VV_{h-1}\|_1^2.
        \#
        For the term $\sum_{u \in \cU_{h+1}, o \in \cO, a \in \cA } \|w_{t,u,o,a}\|_2$ where $w_{t,u,o,a} := [(\Mb_h^t(o,a) - \Mb_h(o,a)) \KK_{h-1}]_u$, we have
        \# \label{eq:9993}
        \sum_{u \in \cU_{h+1}, o \in \cO, a \in \cA } \|w_{t,u,o,a}\|_2 &\le \sum_{u \in \cU_{h+1}, o \in \cO, a \in \cA } \|w_{t,u,o,a}\|_1 \notag \\
        & = \sum_{ o \in \cO, a \in \cA } \sum_{l = 1}^{d_{\psr, h-1}} \Big\|  \big( \Mb_h^t(o, a) - \Mb_h(o, a) \big) \KK_{h-1} \bfe_l \Big\|_1  \notag\\
        & \le \frac{2AU_A}{\alpha} \sum_{l = 1}^{d_{\psr, h-1}} \big\| \KK_{h-1} \bfe_l \big\|_1 \le \frac{2AU_A d_{\psr, h-1}}{\alpha} \big\|\KK_{h-1} \big\|_1,
        \#
        where the first inequality is implied by the fact that $\|\cdot\|_2 \le \|\cdot\|_1$, the second inequality follows from Condition 2 in Definition \ref{def:general:regular:psr}, and the last inequality uses that $\sum_{l = 1}^{d_{\psr, h-1}} \| \KK_{h-1} \bfe_l \|_1 \le \sum_{l = 1}^{d_{\psr, h-1}} \| \KK_{h-1} \|_1 \|\bfe_l\|_1 = d_{\psr, h-1} \| \KK_{h-1} \|_1 $.

        Putting \eqref{eq:9992} and \eqref{eq:9993} together, we conclude the proof of Lemma \ref{lemma:120}.
    \end{proof}

    Invoking Lemma \ref{lemma:l2:eluder} with $R = \alpha$ and $R_x^2 R_w^2 \le \frac{4A^2U_A^2 d_{\psr}^2 \delta_{\psr}^2}{\alpha^2} $ (Lemma~\ref{lemma:120} and \eqref{eq:11800}) gives
    \# \label{eq:122}
    & \sum_{t = 1}^T \alpha \land \sum_{\tau_{h-1} \sim \PP^{\pi^t}} \sum_{u \in \cU_{h+1}} \sum_{o \in \cO, a \in \cA} |w_{t,u,o,a}^\top x_{\tau_{h-1}}| \notag \\
    & \qquad \le \Big[2d_{\psr}  \Big( T + \sum_{t = 1}^{T} \sum_{s = 1}^{t-1} \sum_{\tau_{h-1} \sim \PP^{\pi^s}} \big(  \sum_{u \in \cU_{h+1}} \sum_{o, a}  |w_{t,u,o,a}^\top x_{\tau_{h-1}}| \big)^2 \Big) \cdot \log\Big(1 + \frac{4A^2U_A^2 d_{\psr} \delta_{\psr}^2 T }{ \alpha^4}\Big)\Big]^{1/2}  \notag\\
    & \qquad = \Big[ d_{\psr} \cdot \iota  \cdot  \Big( T + \sum_{t = 1}^{T} \sum_{s = 1}^{t - 1} \sum_{\tau_{h-1} \sim \PP^{\pi^s}} \Big( \sum_{o_h \in \cO, a_h \in \cA} \Big\| \big(\Mb_h^t(o_h, a_h) - \Mb_h(o_h, a_h) \big) \bar{\bfq}(\tau_{h-1}) \Big\|_1  \Big)^2 \Big)\Big]^{1/2} , 
    \#
    where we use facts that $ \sum_{u \in \cU_{h+1}} \sum_{o \in \cO, a \in \cA} |w_{t,u,o,a}^\top x_{\tau_{h-1}}| = \sum_{o \in \cO, a \in \cA} \| \big(\Mb_h^t(o, a) - \Mb_h(o, a) \big) \bar{\bfq}(\tau_{h-1}) \|_1$, $d_{\psr,h-1} \le d_{\psr}$ for any $h$, and $\iota := 2 \log\Big(1 + \frac{4 d_{\psr} A^2U_A^2  \delta_{\psr}^2 T }{ \alpha^4}\Big)$.

    For the term $\sum_{t = 1}^T \| \bfq_0^t - \bfq_0 \|_1$ in \eqref{eq:11921}, by the fact that $\|\bfq_0^1 - \bfq_0\|_1 \le 1$ and the simple inequality $2(t-1)\geq t$ when $t \geq 2$, we have
    \$
     \sum_{t = 1}^T \|\bfq_0^t - \bfq_0 \|_1 & \le  1 + \sum_{t = 1}^T  \Big[\frac{2(t-1)}{t} \cdot \|\bfq_0^t - \bfq_0 \|_1^2 \Big]^{1/2}.
    \$
    By Cauchy-Schwarz inequality, we further obtain
    \# \label{eq:1191}
    \sum_{t = 1}^T \|\bfq_0^t - \bfq_0 \|_1
    & \le 1 + \Big[\sum_{t = 1}^T \frac{2}{t}\Big]^{1/2} \cdot \Big[\sum_{t = 1}^T (t - 1) \cdot \|\bfq_0^t - \bfq_0 \|_1^2\Big]^{1/2} \notag  \\
    & \lesssim   1 + \Big[\sum_{t = 1}^T \frac{2}{t}\Big]^{1/2} \cdot \Big[\sum_{t = 1}^T (t - 1) \cdot D_H^2(\bfq_0^t, \bfq_0) \Big]^{1/2} \notag\\
    & \lesssim 1 + \Big[ U_A \cdot \log T \cdot \sum_{t = 1}^T \sum_{s = 1}^{t - 1} D_H^2 \Big( \PP_{f^t}^{\pi_{\exp}(f^s, 0)}, \PP_{f^*}^{\pi_{\exp}(f^s, 0)}\Big)\Big]^{1/2},
    \#
    where the second inequality follows from Lemma~\ref{lem:hellinger:1}, and the last inequality importance sampling and the fact that $\pi_{\exp}(f^s, 0) = \mathrm{Unif}(\cU_{A, 1})$ for any $s \in [t - 1]$. 

    Putting \eqref{eq:11921}, \eqref{eq:122}, \eqref{eq:1191}, and Lemma~\ref{lemma:decomp:train} together and taking summation over $h$, we obtain
    \$
    & \sum_{t = 1}^T V_{f^t} - V^{\pi^t} \\
    & \qquad \lesssim \sum_{h = 0}^{H - 1} \bigg[\Big(\frac{d_{\psr} A^3 U_A^4  \cdot \iota}{\alpha^4}\Big) \sum_{t = 1}^{T} \sum_{s = 1}^{t - 1} D_H^2 \Big( \PP_{f^t}^{\pi_{\exp}(f^s, h)}, \PP_{f^*}^{\pi_{\exp}(f^s, h)} \Big)\Big) \bigg]^{1/2} +  \frac{H\sqrt{d_{\psr} T \cdot \iota} }{\alpha}\\
    & \qquad \lesssim \bigg[ \Big(\frac{d_{\psr} A^3 U_A^4 H  \cdot \iota}{\alpha^4} \Big)  \sum_{t = 1}^{T} \sum_{h = 0}^{H - 1}\sum_{s = 1}^{t - 1} D_H^2 \Big( \PP_{f^t}^{\pi_{\exp}(f^s, h)}, \PP_{f^*}^{\pi_{\exp}(f^s, h)} \Big)\bigg]^{1/2} + \frac{H\sqrt{d_{\psr} T \cdot \iota} }{\alpha},
    \$
    where the last inequality follows from Cauchy-Schwarz inequality.
    Hence, by the definition of GEC in Definition~\ref{def:gec}, we obtain that the GEC for $\alpha$-generalized regular PSR is upper bounded by
    \$
    \cO\Big(\frac{d_{\psr}  A^3 U_A^4 H \cdot \iota}{\alpha^4} \Big),
    \$
    which concludes our proof of Lemma~\ref{lemma:general:regular:psr}.
\end{proof}

\begin{remark}
    In Step 3, we use a novel $\ell_2$ eluder technique to connect Steps 1 and 2. This is different from the $\ell_1$ technique adopted by \citet{liu2022partially,zhan2022pac}. See Appendix~\ref{appendix:eluder} for more details. We also adjust Step 2 (Lemma~\ref{lemma:decomp:train}) in this version due to some technical issue, which we will elaborate in Remark \ref{remark:issue}.
\end{remark}

\subsection{Weakly Revealing POMDP} \label{appendix:weakly:revealing}

For $m$-step weakly revealing POMDPs, we choose the core test set as
\# \label{eq:6331}
\{\cU_h = (\cO \times \cA)^{\min\{m-1, H-h\}} \times \cO\}_{h \in [H]}.
\#
Similar to \citet{liu2022partially}, we define observable operators in the real model $f^*$ as
\begin{equation}
    \begin{aligned} \label{eq:6332}
        \left\{\begin{array}{l}\mathbf{q}_{0}=\mathbb{M}_{1} \mu_{1} \in \mathbb{R}^{|\cU_1|} \\ \mathbf{M}_{h}(o, a)=\mathbb{M}_{h+1} \mathbb{T}_{h, a} \operatorname{diag}\left(\mathbb{O}_{h}(o \mid \cdot)\right) \mathbb{M}_{h}^{\dagger} \in \mathbb{R}^{|\cU_{h+1}| \times |\cU_h|} \quad h \le H - m \\
            \Mb_h(o, a) = [\mathbf{1}(t_h = (o, a, t_{h+1})]_{(t_{h+1}, t_h) \in \cU_{h+1} \times \cU_h} \in \RR^{|\cU_{h+1}| \times |\cU_h|} \quad h > H -m\end{array}\right. .
    \end{aligned}
\end{equation}
It is not difficult to verify these observable operators satisfy conditions.

\begin{proof}[Proof of Lemma \ref{lemma:weak:reveal}] For simplicity, we only verify two conditions in Definition~\ref{def:general:regular:psr} for $m$-step $\alpha$-weakly revealing POMDPs.
    \paragraph{Condition 1.} For the case $h > H - m$, we have
    \# \label{eq:6333}
    \max_{\pi} \sum_{\tau_{h:H}} | \bfm(o_{h:H}, a_{h:H}) \bfx | \cdot \pi(o_{h:H}, a_{h:H}) = \max_{\pi} \sum_{\tau_{h:H}} |\bfx(\tau_{h:H}) |  \cdot \pi(o_{h:H}, a_{h:H}) \le \|\bfx\|_1
    \#
    for any $\bfx \in \RR^{|\cO|^{H-h+1} |\cA|^{H-h+1}}$. For the case $h \le H - m$, we have
    \# \label{eq:6334}
    & \max_{\pi} \sum_{\tau_{h:H}} | \bfm(o_{h:H}, a_{h:H}) \bfx | \cdot \pi(o_{h:H}, a_{h:H}) \notag \\
    & \qquad = \max_{\pi} \sum_{\tau_{h:H}} | \bfm(o_{h:H}, a_{h:H}) \MM_h \MM_h^\dagger \bfx | \cdot \pi(o_{h:H}, a_{h:H})  \notag \\
    & \qquad \le  \max_{\pi} \sum_{\tau_{h:H}} \sum_{i = 1}^S  |\bfm(o_{h:H}, a_{h:H}) \MM_h \bfe_i | \cdot | \bfe_i^\top \MM_h^\dagger \bfx | \cdot \pi(o_{h:H}, a_{h:H}) .
    \#
    Note that for any $\pi$ we have
    \$
    \sum_{\tau_{h:H}} | \bfm(o_{h:H}, a_{h:H}) \MM_h \bfe_i | \cdot \pi(o_{h:H}, a_{h:H}) = \sum_{\tau_{h:H}} \PP^{\pi}(\tau_{h:H} \mid s_i) = 1,
    \$
    where $s_i$ denotes $i$-th state. Together with \eqref{eq:6334}, we further obtain that
    \# \label{eq:6335}
    &\max_{\pi} \sum_{\tau_{h:H}} | \bfm(o_{h:H}, a_{h:H}) \bfx | \cdot \pi(o_{h:H}, a_{h:H})   \notag \\ 
    & \qquad \le \sum_{i = 1}^S |\bfe_i^\top \MM_h^\dagger \bfx|  =  \|\MM_h^\dagger \bfx \|_1 \le \| \MM_h^\dagger \|_1 \cdot \|\bfx\|_1 \le \frac{\sqrt{S}}{\alpha} \|\bfx\|_1,
    \#
    where the last inequality uses $\|\MM_h^\dagger\|_1 \le \sqrt{S} \|\MM_h^\dagger\|_2 \le \sqrt{S}/\alpha$.

    \paragraph{Condition 2.} For the case $h > H - m$, by the definitions of observable operators in \eqref{eq:6332}, we have
    \# \label{eq:6336}
    \max_{\pi} \sum_{(o_h, a_h) \in \cO \times \cA} \big \| \Mb_h(o_h, a_h) \bfx \big\|_1 \cdot \pi(o_h, a_h) = \|\bfx\|_1.
    \#
    Then we consider the case $h \le H - m$. Similar to \eqref{eq:6334}, we have
    \# \label{eq:6337}
    & \max_{\pi} \sum_{(o_h, a_h) \in \cO \times \cA} \big \| \Mb_h(o_h, a_h) \bfx \big\|_1 \cdot \pi(o_h, a_h) \notag \\
    &\qquad \le \max_{\pi} \sum_{(o_h, a_h, t_{h+1}) \in \cO \times \cA \times \cU_{h+1}} \sum_{i = 1}^S \big |\bfe_{t_{h+1}}^\top  \Mb_h(o_h, a_h) \MM_h  \bfe_i \big| \cdot \big| \bfe_i^\top \MM_h^\dagger \bfx \big| \cdot \pi(o_h, a_h) \notag \\
    & \qquad \le |\cU_{A, h+1}| \cdot \big\|\MM_h^\dagger \bfx \big\|_1 \le \frac{\sqrt{S} \cdot |\cU_{A,h+1}|}{\alpha} \|\bfx\|_1,
    \#
    where the second inequality is implied by the fact that  $\sum_{(o_h, a_h, t_{h+1}) \in \cO \times \cA \times \cU_{h+1}}  |\bfe_{t_{h+1}}^\top \Mb_h(o_h, a_h) \MM_h \bfe_i | \cdot \pi(o_h, a_h) = \sum_{(o_h, a_h, t_{h+1}) \in \cO \times \cA \times \cU_{h+1}} \PP(o_h, t_{h+1} \mid s_i, a_h) \cdot \pi(o_h, a_h) \le |\cU_{A,h+1}|$ for any $\pi$, and the last inequality can be obtained by the same derivation of \eqref{eq:6335}.

    Putting \eqref{eq:6333}, \eqref{eq:6334}, \eqref{eq:6336}, and \eqref{eq:6337} together, we finish the proof of Lemma~\ref{lemma:weak:reveal}.
\end{proof}

\subsection{Latent MDP} \label{appendix:latent:mdp}
Let $\cM=\{(\cS,\cA,\TT_m,R,\nu_m,\omega_m)\}_{m=1}^M$ be a latent MDP we define in Definition \ref{def:lt}. As shown in \cite{zhan2022pac}, we can solve this latent MDP by solving an equivalent POMDP. We consider a POMDP $(\bar{\cS},\cO,\cA,H,\mu_1,\TT,\OO,R)$ with $\bar{\cS}=\cS\times[m]$ and $\cO=\cS$. The initial state of this POMDP in each episode is sampled from $\mu_1((s_1,m))=\omega_{m}\cdot \nu_{m}(s_1)$. Our observation is the first entry of the state in each stage. At the $h$-th stage, we choose an action $a_h$, obtain the reward $R_h(s_h,a_h)$. The next state is sampled from $\TT_{h,a_h}$, which is defined as
\#
\TT_{h,a_h}\bigl((s_{h+1},m_{h+1})\mid (s_{h},r_{h-1},m_{h})\bigr)&=\mbo_{\{m_{h+1}=m_{h}\}}\cdot  \TT_{h,m}(s_{h+1}\mid s_{h},a_h).\label{lt:T}
\#
Solving the latent MDP $\cM$ is equivalent with solving the POMDP $(\bar{\cS},\cO,\cA,H,\mu_1,\TT,\OO,R)$. The following lemma shows that, the POMDP $(\bar{\cS},\cO,\cA,H,\mu_1,\TT,\OO,R)$ is also a generalized regular PSR.
\begin{proof}[Proof of Lemma \ref{lemma:latent:mdp}]
    Let $\cU_{h}$ be the test set for $\cM$ in the $h$-th step. We define the augmented test set $\bar{\cU}_{h+1}$ by $\bar{\cU}_{h+1}=\{(s_{h},t_{h})\mid s_{h}\in\cS, t_{h}\in\cU_{h}\}$. We define the matrix $\MM_h\in \RR^{\vert\bar{\cU}_{h+1}\vert\times \bar{\cS}}$ by $[\MM_h]_{(s^\prime,t),(s,m)}=\PP(o_h=s^\prime,t\mid s_h=(s,m))=\mbo_{\{s=s^\prime\}}\cdot \PP_{h,m}(t\mid s)$. We then have
    \#
    \MM_h=\mathrm{diag}\bigl(L_h(s^1,\cU_{h+1}),L_h(s^2,\cU_{h+1}),\ldots ,L_h(s^S,\cU_{h+1})\bigr).\nonumber
    \#
    By Assumption \ref{asp:lmgr}, we have $\sigma_{MS}(\MM_h)\geq \alpha.$ We define observable operators in the real model $f^*$ as
    \begin{equation}
        \begin{aligned} \label{eq:6332:lt}
            \left\{\begin{array}{l}\mathbf{q}_{0}=\mathbb{M}_{1} \mu_{1} \in \mathbb{R}^{|\cU_1|} \\ \mathbf{M}_{h}(s, a)=\mathbb{M}_{h+1} \mathbb{T}_{h, a} \operatorname{diag}\left(\mathbb{O}_{h}(s \mid \cdot)\right) \mathbb{M}_{h}^{\dagger} \in \mathbb{R}^{|\cU_{h+1}| \times |\cU_h|} \quad h \le H - 1 \\
                \Mb_H(s, a) = [1]_{ t_H \in  \cU_h} \in \RR^{1 \times |\cU_H|} \quad\end{array}\right. .
        \end{aligned}
    \end{equation}
    It is not difficult to verify these observable operators satisfy conditions. Then we verify two conditions in Definition~\ref{def:general:regular:psr} for the POMDP that equivalents with the latent MDP $\cM$.
    \paragraph{Condition 1.} For the case $h = H $, we have
    \# \label{eq:6333lt}
    \max_{\pi} \sum_{(s_H,a_H)} | \bfm(s_H,a_H) \bfx | \cdot \pi(s_H,a_H) = \max_{\pi} \sum_{(s_H,a_H)} |\bfx(s_H,a_H) |  \cdot \pi(s_H,a_H) \le \|\bfx\|_1
    \#
    for any $\bfx \in \RR^{|\cO|\times |\cA|}$. For the case $h \le H - 1$, we have
    \# \label{eq:6334lt}
    & \max_{\pi} \sum_{(\tau_{h:H})} | \bfm(s_{h:H}, a_{h:H}) \bfx | \cdot \pi(s_{h:H}, a_{h:H}) \notag \\
    & \qquad = \max_{\pi} \sum_{\tau_{h:H}} | \bfm(s_{h:H}, a_{h:H}) \MM_h \MM_h^\dagger \bfx | \cdot \pi(s_{h:H}, a_{h:H})  \notag \\
    & \qquad \le  \max_{\pi} \sum_{\tau_{h:H}} \sum_{i = 1}^{\vert \bar{\cS}\vert} |\bfm(s_{h:H}, a_{h:H}) \MM_h \bfe_i | \cdot | \bfe_i^\top \MM_h^\dagger \bfx | \cdot \pi(s_{h:H}, a_{h:H}) .
    \#
    Note that for any $\pi$ we have
    \$
    \sum_{\tau_{h:H}} | \bfm(s_{h:H}, a_{h:H}) \MM_h \bfe_i | \cdot \pi(s_{h:H}, a_{h:H}) = \sum_{\tau_{h:H}} \PP^{\pi}(\tau_{h:H} \mid s_i) = 1,
    \$
    where $s_i$ denotes $i$-th state. Together with \eqref{eq:6334lt}, we further obtain that
    \# \label{eq:6335lt}
    \max_{\pi} \sum_{\tau_{h:H}} | \bfm(s_{h:H}, a_{h:H}) \bfx | \cdot \pi(s_{h:H}, a_{h:H}) \le \|\MM_h^\dagger \bfx \|_1 \le \| \MM_h^\dagger \|_1 \cdot \|\bfx\|_1 \le \frac{\sqrt{SM}}{\alpha} \|\bfx\|_1,
    \#
    where the last inequality uses $\|\MM_h^\dagger\|_1 \le \sqrt{SM} \cdot \|\MM_h^\dagger\|_2 \le \sqrt{SM}/\alpha$.

    \paragraph{Condition 2.} For the case $h=H$, by the definitions of observable operators in \eqref{eq:6332:lt}, we have
    \# \label{eq:6336lt}
    \max_{\pi} \sum_{(s_h, a_h) \in \cS \times \cA} \big \| \Mb_h(s_h, a_h) \bfx \big\|_1 \cdot \pi(s_h, a_h) = \|\bfx\|_1.
    \#
    Then we consider the case $h \le H - 1$. Similar to \eqref{eq:6334lt}, we have
    \# \label{eq:6337lt}
    & \max_{\pi} \sum_{(s_h, a_h) \in \cO \times \cA} \big \| \Mb_h(s_h, a_h) \bfx \big\|_1 \cdot \pi(s_h, a_h) \notag \\
    &\qquad \le \max_{\pi} \sum_{(s_h, a_h, t_{h+1}) \in \cS \times \cA \times \bar{\cU}_{h+1}} \sum_{i = 1}^S \big |\bfe_{t_{h+1}}^\top  \Mb_h(o_h, a_h) \MM_h \bfe_i \big| \cdot \big| \bfe_i^\top \MM_h^\dagger \bfx \big| \cdot \pi(s_h, a_h) \notag \\
    & \qquad \le |\bar{\cU}_{A, h+1}| \cdot \big\|\MM_h^\dagger \bfx \big\|_1 \le \frac{\sqrt{SM} \cdot |\bar{\cU}_{A,h+1}|}{\alpha} \|\bfx\|_1,
    \#
    where the second inequality is implied by the fact that  $\sum_{(s_h, a_h, t_{h+1}) \in \cS \times \cA \times \cU_{h+1}}  |\bfe_{t_{h+1}}^\top  \Mb_h(o_h, a_h) \MM_h \bfe_i | \cdot \pi(s_h, a_h) = \sum_{(s_h, a_h, t_{h+1}) \in \cS \times \cA \times \cU_{h+1}} \PP(s_h, t_{h+1} \mid \bar{s}_i, a_h) \cdot \pi(s_h, a_h) \le |\bar{\cU}_{A,h+1}|$ for any $\pi$, 
    and the last inequality can be obtained by the same derivation of \eqref{eq:6335lt}. 

    Putting \eqref{eq:6333lt}, \eqref{eq:6334lt}, \eqref{eq:6336lt}, and \eqref{eq:6337lt} together, we finish the proof of Lemma~\ref{lemma:latent:mdp}.
\end{proof}

\subsection{Decodable POMDP} \label{appendix:decodable}

Before our proof, we define the observable operators for decodable POMDPs. We choose the core test set as
\# \label{eq:6431}
\{\cU_h = (\cO \times \cA)^{\min\{m - 1, H - h\}} \times \cO\}_{h \in [H]}.
\#
When $h \le H - m$, for any $t_h = (o_h, a_h, \cdots, o_{h+m-1})$ and $t_{h+1} = (o'_{h+1}, a'_{h+1}, \cdots, o'_{h+m})$, we have
\$
\PP_h(t_{h+1} \mid t_h) &= \PP(o_{h+m} = o'_{h+m} \mid s_{h+m-1} = \phi^*_{h+m-1}(t_h), a_{h+m-1}) \\
& \qquad \times \mathbf{1}\{(o_{h+1:h+m-1}, a_{h+1:h+m-2}) = (o'_{h+1:h+m-1}, a'_{h+1:h+m-2})\}   \\
& = \sum_{s \in \cS} \OO_{h+1}(o'_{h+m} \mid s ) \TT_{h + m -1, a_{h+m-1}}( s \mid \phi_{h+m-1}^*(t_h) ) \\
& \qquad \times \mathbf{1}\{(o_{h+1:h+m-1}, a_{h+1:h+m-2}) = (o'_{h+1:h+m-1}, a'_{h+1:h+m-2})\}  .
\$
Moreover, we define
\$
\PP_{h,o,a}(t_{h+1} \mid t_h) = \PP_h(t_{h+1} \mid t_h) \cdot \mathbf{1}\{(o, a) = (o_h, a_h) \}.
\$
With this notation, we define observable operators in the real model $f^*$ as
\begin{equation}
    \begin{aligned} \label{eq:6432}
        \left\{\begin{array}{l}\mathbf{q}_{0}= [\PP(o_{1:m}, a_{1:m-1})]_{(o_{1:m}, a_{1:m-1}) \in \cU_1}  \in \mathbb{R}^{|\cU_1|}                                                         \\
            \mathbf{M}_{h}(o, a) = [\PP_{h,o,a}(t_{h+1} \mid t_h)]_{(t_{h+1}, t_h) \in \cU_{h+1} \times \cU_h}  \in \mathbb{R}^{|\cU_{h+1}| \times |\cU_h|} \quad h \le H - m \\
            \Mb_h(o, a) = [\mathbf{1}\{t_h = (o, a, t_{h+1})\}]_{(t_{h+1}, t_h) \in \cU_{h+1} \times \cU_h} \in \RR^{|\cU_{h+1}| \times |\cU_h|} \quad h > H -m\end{array}\right. .
    \end{aligned}
\end{equation}

\begin{proof}[Proof of Lemma \ref{lemma:decodable}]
    We show decodable POMDPs satisfy two conditions in Definition~\ref{def:general:regular:psr} with $\alpha = 1$.
    \paragraph{Condition 1.} For $h > H -m + 1$, we have
    \# \label{eq:6433}
    &\max_{\pi} \sum_{\tau_{h:H}} | \bfm(o_{h:H}, a_{h:H}) \bfx | \cdot \pi(o_{h:H}, a_{h:H}) \notag \\
    & \qquad = \max_{\pi} \sum_{\tau_{h:H}} | \bfe(o_{h:H}, a_{h:H})^\top \bfx | \cdot \pi(o_{h:H}, a_{h:H}) \le \|\bfx\|_1,
    \#
    where the first equality uses \eqref{eq:6432}. For the case $h \le H - m + 1$, we have
    \# \label{eq:6435}
    & \max_{\pi} \sum_{\tau_{h:H}} | \bfm(\tau_{h:H}) \bfx | \cdot \pi(\tau_{h:H}) \notag\\
    & \qquad = \max_{\pi} \sum_{\tau_{h:H}} | \bfe(\tau_{h:h+m-1})^\top \bfx | \cdot \PP\big(\tau_{h+m:H} \mid \tau_{h:h+m-1}\big) \cdot \pi(\tau_{h:H}) \notag\\
    & \qquad = \max_{\pi} \sum_{\tau_{h:h+m-1}} |\bfx(\tau_{h:h+m-1})| \cdot \pi(\tau_{h:h+m-1}) \le \|\bfx\|_1,
    \#
    where the first equality is implied by the definitions of observable operators in \eqref{eq:6432}. 
    \paragraph{Condition 2.} For the case $h > H - m$ and $\bfx \in \RR^{|\cU_{h}|}$, we have
    \# \label{eq:6436}
    \sum_{(o_h, a_h) \in \cO \times \cA} \big\| \Mb_h(o_h, a_h) \bfx \big\|_1 \cdot \pi(o_h, a_h) \le \sum_{(o_h, a_h) \in \cO \times \cA} \big\| \Mb_h(o_h, a_h) \bfx \big\|_1 = \|\bfx\|_1,
    \#
    where the last inequality uses the definition of $\Mb_h$ in~\eqref{eq:6432}.
    For any $h \in [H - m]$ and $\bfx \in \RR^{|\cU_{h}|}$, we have
    \begin{equation}
        \begin{aligned} \label{eq:6437}
    &\sum_{(o_h, a_h) \in \cO \times \cA} \big\| \Mb_h(o_h, a_h) \bfx \big\|_1 \cdot \pi(o_h, a_h)   \\
    &\qquad \le \sum_{(o_h, a_h, t_{h+1}) \in \cO \times \cA \times \cU_{h+1}} \big| \bfe_{t_{h+1}}^\top \Mb_h(o_h, a_h) \bfx \big|    \\
    & \qquad = \sum_{(o_h, a_h, t_{h+1}) \in \cO \times \cA \times \cU_{h+1}} \big| \bfx\big(o_h, a_h, o'_{h + 1:h + m-1}, a'_{h + 1:h + m-1} \big)\big| \\
    & \qquad \qquad \qquad \qquad \qquad \qquad \qquad \times \PP \big( o'_{h+m} \mid o_h, a_h, o'_{h + 1:h + m-1}, a'_{h + 1:h + m-1} \big),
        \end{aligned}
    \end{equation}
    where the first inequality uses $\pi(o_h, a_h) \le 1$ and $t_{h+1} = (o'_{h+1}, a'_{h+1}, \cdots, o'_{h+m})$ in the last equality. Meanwhile, we have
    \$
    & \sum_{o'_{h+m} : t_{h+1} \in \cU_{h+1}} \PP \big( o'_{h+m} \mid o_h, a_h, o'_{h + 1:h + m-1}, a'_{h + 1:h + m-1} \big) \\
    & \qquad = \sum_{o'_{h+m} \in \cO } \PP \big( o'_{h+m} \mid o_h, a_h, o'_{h + 1:h + m-1}, a'_{h + 1:h + m-1} \big) = 1.
    \$
    Combining this with \eqref{eq:6437}, we have
    \# \label{eq:6438}
    \sum_{(o_h, a_h) \in \cO \times \cA} \big\| \Mb_h(o_h, a_h) \bfx \big\|_1 \cdot \pi(o_h, a_h) \le \|\bfx \|_1 \le |\cU_{A,h+1}| \cdot \|\bfx\|_1.
    \#

    Putting \eqref{eq:6433}, \eqref{eq:6435}, \eqref{eq:6436}, and \eqref{eq:6438} together, we conclude the proof of Lemma~\ref{lemma:decodable}.
\end{proof}

\subsection{Low-rank POMDP}  \label{appendix:pomdp:linear}
For the low-rank MDP with $m$-step future sufficiency, we choose the core set as
\# \label{eq:6331l}
\{\cU_h = (\cO \times \cA)^{\min\{m-1, H-h\}} \times \cO\}_{h \in [H]}.
\#
\cite{wang2022embed} shows that, $\MM_h^{\dagger }$ defined in Assumption \ref{asp:fs35} is the inverse of the forward emission operator $\MM_h$.
\begin{lemma}[Pseudo-Inverse of Forward Emission \citep{wang2022embed}]
    In a low-rank MDP, we have
    \$
    (\MM_h^{\dagger}\MM_h)\PP_h^{\pi}=\PP_h^{\pi}
    \$
    for all $h\in[H]$ and $\pi\in\Pi$ when Assumption \ref{asp:fs35} holds. Here $\PP_h^{\pi}\in L^1(\cS)$ and $\PP_h^{\pi}(s_h)$ is the probability of visiting the state $s_h$ in the $h$-th step when following the policy $\pi$.
\end{lemma}
Hence, we define the operator $\mathbf{O}_h(o):L^1(\cS)\to L^1(\cS)$ in the real model $f^*$ as
\$
\bigl(\mathbf{O}_h(o)\bfx\bigr)(s)=\bfx(s)\mathbb{O}_{h}(o \mid s),
\$
for $\bfx\in L^1(\cS)$ and define the observable operators in the true  model $f^*$ as
\begin{equation}
    \begin{aligned} \label{eq:6332l}
        \left\{\begin{array}{l}\mathbf{q}_{0}=\mathbb{M}_{1} \mu_{1} \in L^1(\cU_1) \\ \mathbf{M}_{h}(o, a)=\mathbb{M}_{h+1} \mathbb{T}_{h, a} \mathbf{O}_h(o) \mathbb{M}_{h}^{\dagger} :L^1(\cU_h)\to L^1(\cU_{h+1}) \end{array}\right. ,
    \end{aligned}
\end{equation}
when $h\leq H-m$. For $h>H-m$, we define $\mathbf{M}_{h}(o,a):L^1(\cU_h)\to L^1(\cU_{h+1})$ as
\#
(\mathbf{M}_{h}(o,a)\bfx)(t_{h+1})=\bfx(o,a,t_{h+1}).\nonumber
\#
It is not difficult to verify that these observable operators satisfy conditions. Moreover, we define the operator $\bar{\mathbf{{m}}}_h(o_{h:H},a_{h:H}):L^1(\cS)\to \RR$ as
\#
\mathbf{\bar{m}}_h(o_{h:H},a_{h:H})\mathbf{x}=\int\PP(o_{h:H}\mid a_{h:H},s_h)\mathbf{x}(s_h)\ud s_h.\nonumber
\#
We then have
\#
\mathbf{m}_h(o_{h:H},a_{h:H})\bfx =\mathbf{\bar{m}}_h(o_{h:H},a_{h:H})(\mathbb{M}_{h}^{\dagger}\bfx)\nonumber
\#
for all $\bfx\in L^1(\cA^m\times \cO^{m+1})$. Here $\mathbf{m}_h(o_{h:H},a_{h:H})$ is defined in \eqref{eq:def:m}.

\begin{proof}[Proof of Lemma \ref{lemma:lowrank-pomdp}] For simplicity, we only verify two conditions in Definition~\ref{def:general:regular:psr} for the low-rank POMDPs with the $m$-step future sufficiency.
    \paragraph{Condition 1.} For the case $h > H - m$, we have
    \# \label{eq:6333l}
    \max_{\pi} \int | \bfm(o_{h:H}, a_{h:H}) \bfx | \cdot \pi(o_{h:H}, a_{h:H}) \ud \tau_{h:H}&= \max_{\pi} \int|\bfx(\tau_{h:H}) |  \cdot \pi(o_{h:H}, a_{h:H}) \ud \tau_{h:H}\\
    &\le \|\bfx\|_1\nonumber
    \#
    for any $\bfx \in L(|\cO|^{H-h+1} |\cA|^{H-h+1})$. For the case $h \le H - m$, we have
    \# \label{eq:6334l}
    & \max_{\pi} \int| \bfm(o_{h:H}, a_{h:H}) \bfx | \cdot \pi(o_{h:H}, a_{h:H}) \ud\tau_{h:H} \\
    & \qquad = \max_{\pi} \int_{\tau_{h:H}}\int_{\cS} \PP(o_{h:H}\mid s_h,a_{h:H})\cdot  ( \MM_h^\dagger \bfx)(s_h)  \cdot \pi(o_{h:H}, a_{h:H}) \ud\tau_{h:H}\ud s_h\notag\\
    &\qquad =\max_{\pi}\int_{\cS} ( \MM_h^\dagger \bfx)(s_h)  \Bigl(\int_{\tau_{h:H}}\PP(o_{h:H}\mid s_h,a_{h:H}) \cdot \pi(o_{h:H}, a_{h:H}) \ud\tau_{h:H}\Bigr)\ud s_h.\notag
    \#
    Note that for any $\pi$ we have
    \$
    \int_{\tau_{h:H}}\PP(o_{h:H}\mid s_h,a_{h:H}) \cdot \pi(o_{h:H}, a_{h:H}) \ud\tau_{h:H}= \int_{\tau_{h:H}}\PP^{\pi}(\tau_{h:H}\mid s_h)  \ud\tau_{h:H} = 1,
    \$
    Together with \eqref{eq:6334l}, we further obtain that
    \# \label{eq:6335l}
    \max_{\pi} \int | \bfm(o_{h:H}, a_{h:H}) \bfx | \cdot \pi(o_{h:H}, a_{h:H}) \ud \tau_{h:H} &\le \max_{\pi}\int_{\cS} ( \MM_h^\dagger \bfx)(s_h) \ud s_h\\
    & \leq \|\MM_h^\dagger\|_{1} \cdot \|\bfx\|_1\le  \|\bfx\|_1/\alpha.\nonumber
    \#

    \paragraph{Condition 2.} For the case $h > H - m$, by the definitions of observable operators in \eqref{eq:6332l}, we have
    \# \label{eq:6336l}
    \max_{\pi} \int_{(o_h, a_h) \in \cO \times \cA} \big \| \Mb_h(o_h, a_h) \bfx \big\|_1 \cdot \pi(o_h, a_h) \ud o_h\ud a_h= \|\bfx\|_1.
    \#
    Then we consider the case $h \le H - m$. Similar to \eqref{eq:6334l}, we have
    \# \label{eq:6337l}
    & \max_{\pi} \sum_{(o_h, a_h) \in \cO \times \cA} \big \| \Mb_h(o_h, a_h) \bfx \big\|_1 \cdot \pi(o_h, a_h) \notag \\
    &\qquad \le \max_{\pi} \int_{(o_h, a_h, t_{h+1}) \in \cO \times \cA \times \cU_{h+1}} \int_\cS \PP(t_{h+1},o_h \mid s,a_h) \cdot \Bigl|  \bigl(\Mb_h(o_h, a_h)  \bfx \bigr)(s)\Bigr| \cdot \pi(o_h, a_h)\ud s\ud o_h\ud a_h\ud t_{h+1} \notag \\
    & \qquad \le |\cU_{A, h+1}| \cdot \big\|\MM_h^\dagger \bfx \big\|_1  \le \frac{|\cU_{A,h+1}|}{\alpha} \cdot \|\bfx\|_1.
    \#
    Putting \eqref{eq:6333l}, \eqref{eq:6335l}, \eqref{eq:6336l}, and \eqref{eq:6337l} together, we finish the proof of Lemma~\ref{lemma:lowrank-pomdp}.
\end{proof}

\subsection{Regular PSR} \label{appendix:regular:psr}

\begin{proof}[Proof of Lemma~\ref{lemma:regular:psr}]
    In what follows, we verify that any $\alpha$-regular PSR is an $\alpha$-generalized regular PSR defined in Definition \ref{def:general:regular:psr}.
    \paragraph{Condition 1.}
    For any $h \in [H]$ and $\bfx \in \RR^{|\cU_h|}$, we have
    \# \label{eq:132}
    & \max_{\pi} \sum_{\tau_{h:H}} | \bfm(o_{h: H}, a_{h:H}) \bfx | \cdot \pi(o_{h: H}, a_{h:H}) \notag\\
    & \qquad =  \max_{\pi} \sum_{\tau_{h:H}} | \bfm(o_{h: H}, a_{h:H}) \KK_{h-1} \KK_{h-1}^\dagger \bfx | \cdot \pi(o_{h: H}, a_{h:H}) \notag \\
    & \qquad \le   \max_{\pi} \sum_{\tau_{h:H}}   \sum_{i = 1}^{d_{\psr,h-1}} \big| \bfm(o_{h: H}, a_{h:H}) \KK_{h-1} \bfe_i \big| \cdot \big| \bfe_i^\top \KK_{h-1}^\dagger \bfx | \cdot \pi(o_{h: H}, a_{h:H}).
    \#
    Moreover, for any $i \in [d_{\psr,h-1}]$, we have
    \# \label{eq:133}
    \max_{\pi} \sum_{\tau_{h:H}}    \big| \bfm(o_{h: H}, a_{h:H}) \KK_{h-1} \bfe_i \big| \cdot \pi(o_{h: H}, a_{h:H}) & = \max_{\pi} \sum_{\tau_{h:H}}   \big| \bfm(o_{h: H}, a_{h:H}) \bar{\bfq}(\tau_{h-1}^i)  \big| \cdot \pi(o_{h: H}, a_{h:H}) \notag \\
    & = \max_{\pi}  \sum_{\tau_{h:H}}  \PP^{\pi}(\tau_{h:H}\mid \tau_{h-1}^i) = 1,
    \#
    where the first equality is implied by the fact that $i$-th column of $\KK_{h-1}$ is $\bar{\bfq}(\tau_{h-1}^i)$, and the second equality uses \eqref{eq:linear:weight}. Plugging \eqref{eq:133} into \eqref{eq:132} gives that
    \# \label{eq:134}
    \max_{\pi} \sum_{\tau_{h:H}} | \bfm(o_{h: H}, a_{h:H}) \bfx | \cdot \pi(o_{h: H}, a_{h:H}) \le \big\| \KK_{h-1}^\dagger \bfx\big\|_1 \le \frac{\|\bfx\|_1}{\alpha},
    \#
    where the last inequality is obtained by the definition of regular PSR in Definition~\ref{def:regular:psr}.
    \paragraph{Condition 2.}
    For any $h \in [H]$ and $\bfx \in \RR^{|\cU_h|}$, similar to the derivation of \eqref{eq:132}, we have
    \# \label{eq:135}
    &\max_{\pi} \sum_{(o_h, a_h) \in \cO \times  \cA }  \big\| \Mb_h(o_h, a_h) \bfx \big\|_1 \cdot \pi(o_h, a_h)  \notag \\
    & \qquad \le \max_{\pi} \sum_{(o_h, a_h, u) \in \cO \times  \cA \times \cU_{h+1}} \big| \bfe_u^\top \Mb_h(o_h, a_h) \KK_{h-1} \KK_{h-1}^\dagger \bfx \big| \cdot \pi(o_{h}, a_{h}) \notag\\
    & \qquad \le  \max_{\pi} \sum_{(o_h, a_h, u)  \in \cO \times  \cA \times \cU_{h+1}} \sum_{i = 1}^{d_{\psr, h-1}} \big| \bfe_u^\top \Mb_h(o_h, a_h) \KK_{h-1} \bfe_i \big| \cdot \big| \bfe_i^\top \KK_{h-1}^\dagger \bfx \big| \cdot \pi(o_{h}, a_{h}) .
    \#
    For any $i \in [d_{\psr, h-1}]$, we have
    \#  \label{eq:136}
    &\max_{\pi} \sum_{(o_h, a_h, u)  \in \cO \times  \cA \times \cU_{h+1}} \big| \bfe_u^\top \Mb_h(o_h, a_h) \KK_{h-1} \bfe_i \big| \cdot \pi(o_{h}, a_{h}) \notag \\
    &\qquad = \max_{\pi} \sum_{(o_h, a_h, u)  \in \cO \times  \cA \times \cU_{h+1}} \PP(o_h, u \mid \tau_{h-1}^i, a_h) \cdot \pi(o_{h}, a_{h})  = |\cU_{A, h+1}|.
    \#
    Plugging \eqref{eq:136} into \eqref{eq:135}, we obtain
    \# \label{eq:137}
    \max_{\pi} \sum_{(o_h, a_h) \in \cO \times  \cA }  \big\| \Mb_h(o_h, a_h) \bfx \big\|_1 \cdot \pi(o_h, a_h) \le  |\cU_{A, h+1} | \cdot \big\| \KK_{h-1}^\dagger \bfx \big\|_1 \le \frac{|\cU_{A, h + 1}|}{\alpha} \|\bfx\|_1,
    \#
    where the last inequality uses $\| \KK_{h-1}^\dagger \bfx\|_1 \le \|\KK_{h-1}^{\dagger}\|_1 \cdot \| \bfx \|_1 \le \|\bfx\|_1/\alpha$. Combining \eqref{eq:134} and \eqref{eq:137}, we conclude the proof of Lemma~\ref{lemma:regular:psr}.
\end{proof}

\subsection{PO-bilinear Class} \label{appendix:po:bilinear}

\begin{proof}[Proof of Lemma \ref{lemma:po:bilinear}]
    By the performance difference lemma \citep[see e.g.,][]{jiang@2017,du2021bilinear,uehara2022provably}, we have
    \$
    \sum_{t = 1}^T  V_{f^t} - V^{\pi^t} & = \sum_{t = 1}^T \sum_{h = 1}^H \EE[ g_h^t(\bar{z}_h) - r_h - g_{h+1}^t(\bar{z}_{h+1}); a_{1:h} \sim \pi^t ] \\
    & \le \sum_{t = 1}^T \sum_{h = 1}^H | \la W_h(\pi^t, g^t), X_h(\pi^t) \ra |,
    \$
    where the last inequality is implied by Condition 1 in Definition \ref{def:po:bilinear}.
    Note
    \$
    \EE_{\pi_{\exp}(f^s, h)} l(f^t, \zeta_h) &= \EE [l(\pi^t, g^t, \bar{z}_h, a_h, r_h, o_{h+1}); a_{1:h-1} \sim \pi^s, a_h \sim \mathrm{Unif}(\cA) ] \\
    & = \EE[ r_h + g^{t}_{h+1}(\bar{z}_{h+1}) - g^t_h(\bar{z}_h); a_{1:h-1} \sim \pi^{s}, a_h \sim \pi^{t} ] \\
    & = \la W_h(\pi^t, g^t), X_h(\pi^s) \ra ,
    \$
    where the first equality is implied by the definition of exploration policies in \eqref{eq:def:explore:pobilinear}, the second inequality uses the definition of $l$ in \eqref{eq:def:loss:pobilinear} and importance sampling, and the last equality is implied by Definition~\ref{def:po:bilinear}.
    Then by the same proof of Lemma \ref{lemma:relationship:bilinear} (cf. Appendix~\ref{appendix:bilinear}), we have
    \$
     \sum_{t = 1}^T V_{f^t} - V^{\pi^t}  \le \Big[2\gamma_T(\epsilon,\cX)\sum_{t=1}^T \sum_{h=1}^H \sum_{s=1}^{t-1}  |\EE_{{\pi}_{\exp}(f^s, h)} l(f^t, \zeta_h)|^2\Big]^{1/2} +  2 \min\left\{2\gamma_T(\epsilon, \cX), HT\right\} + HT\epsilon,
    \$
    which concludes the proof of Lemma~\ref{lemma:po:bilinear}.
\end{proof}

\paragraph{Extension to Multi-step PO-bilinear Class}

For some fixed positive integer $K$, we use the notation $(z_{h-1}, o_{h:h+K-1}, a_{h:h+K-2}) = \bar{z}_h^K \in  \bar{\cZ}_h^{K} = \cZ_{h-1} \times \cO^K \times \cA^{K-1}$ and $\cH = \Pi \times \cG$.

    \begin{definition}[$K$-step Value Link Functions]
        Fix a set of policies $\pi^{\mathrm{out}} = \{\pi^{\mathrm{out}}_i: \cO \rightarrow \Delta(\cA)\}_{i = 1}^K$. Value functions $g_h^\pi: \cZ_{h-1} \times \cO^K \times \cA^{K-1} \rightarrow \RR$ at step $h \in [H]$ for a policy $\pi$ are defined as the solution of the following integral equation:
        \$
        \EE[g_h^\pi(z_{h-1}, o_{h:h+K-1}, a_{h:h+K-2}) \mid z_{h-1}, s_h; a_{h:h+K-2} \sim \pi^{\mathrm{out}} ] = V_h^\pi(z_{h-1}, s_h).
        \$
    \end{definition}

    Similar to \citet{uehara2022provably}, we choose $\pi^{\mathrm{out}} = \mathrm{Unif}(\cA)$. The PO-bilinear class is defined as follows.

    \begin{definition}[PO-bilinear Class for POMDPs with multi-step future] \label{def:po:bilinear:multi}
        We say the model is a PO-bilinear class if $\cG$ is realizable with respect to general $K$-step link functions, and there exists $W_h : \Pi \times \cG \rightarrow \cV$ and $X_h: \Pi \rightarrow \cV$ for some Hilbert space $\cV$ such that for all $\pi, \pi' \in \Pi$, $f \in \cG$, and $h \in [H]$
        \begin{itemize}
            \item[1.] We have
                  \$
                  &\left[g_{h+1}\left(\bar{z}_{h+1}^{K}\right) ; a_{1: h-1} \sim \pi^{\prime}, a_{h} \sim \pi, a_{h+1: h+K-1} \sim \mathrm{Unif}(\mathcal{A})\right]+\mathbb{E}\left[r_{h} ; a_{1: h-1} \sim \pi^{\prime}, a_{h} \sim \pi\right] \\ 
                  & \qquad -\mathbb{E}\left[g_{h}\left(\bar{z}_{h}^{K}\right) ; a_{1: h-1} \sim \pi^{\prime}, a_{h: h+K-2} \sim \mathrm{Unif}(\mathcal{A})\right]=\left\langle W_{h}(\pi, g), X_{h}\left(\pi^{\prime}\right)\right\rangle .
                  \$
            \item[2.] $W_h(\pi, g^\pi) = 0$, where $g^\pi \in \cG$ is the value link function of $\pi$.
        \end{itemize}
    \end{definition}

For $f = (\pi, g)$ and $\zeta_h = (\bar{z}_h^K, a_{h+K-1}, r_h, o_{h+K})$. We define the multi-step loss function as 
\# \label{eq:772}
l(f, \zeta_h) = |\cA| \pi_h(a_h \mid \bar{z}_h) \big(g_{h+1}(\bar{z}_{h+1}^K) + r_h \big) - g_h(\bar{z}_h^K).
\#
For the PO-bilinear class with multi-step futures, the exploration policies are defined by
\# \label{eq:773}
\pi_{\exp}(f, h) = \pi \circ_{h} \mathrm{Unif}(\cA) \circ_{h+1} \cdots \circ_{h+K-1} \mathrm{Unif}(\cA),
\#
where $f = (\pi, g)$. With these notations, for any $\{f^t = (\pi^t, g^t)\}_{t \in [T]}$, we have
    \$
    \sum_{t = 1}^T  V_{f^t} - V^{\pi^t} & = \sum_{t = 1}^T \sum_{h = 1}^H \EE[ g_h^t(\bar{z}_h^K) ; a_{1:h-1} \sim \pi^t, a_{h:h+K-2} \sim \mathrm{Unif}(A) ] - \EE[ r_h; a_{1:h} \sim \pi^t] \\
    & \qquad - \EE[g_{h+1}^t(\bar{z}_{h+1}^K); a_{1:h} \sim \pi^t, a_{h+1:h+K-1} \sim \mathrm{Unif}(\cA)] \\
    & \le \sum_{t = 1}^T \sum_{h = 1}^H | \la W_h(\pi^t, g^t), X_h(\pi^t) \ra |,
    \$
    where the last equality is implied by Condition 1 in Definition \ref{def:po:bilinear:multi}. Meanwhile, we have
    \$
    &\EE_{\pi_{\exp}(f^j, h)} l(f^t, \zeta_h) \\
    &= \EE [l(\pi^t, g^t, \bar{z}_h, a_h, r_h, o_{h+1}); a_{1:h-1} \sim \pi^j, a_{h:h+K-1} \sim \mathrm{Unif}(\cA) ] \\
    & = \EE[ r_h  ; a_{1:h-1} \sim \pi^{j}, a_h \sim \pi^{t} ] + \EE[ g^{t}_{h+1}(\bar{z}_{h+1}^K); a_{1:h-1} \sim \pi^j, a_h \sim \pi^t, a_{h+1: h+K-1} \sim \mathrm{Unif}(\cA) ] \\
    & \qquad - \EE[g^t_h(\bar{z}_h^K); a_{1:h-1} \sim \pi^t, a_{h:h+K-2} \sim \mathrm{Unif}(\cA) ] \\
    & = \la W_h(\pi^t, g^t), X_h(\pi^j) \ra ,
    \$
    where the first equality is implied by the definition of exploration policies in \eqref{eq:773}, the second inequality uses the importance sampling, and the last equality is obtained by Definition~\ref{def:po:bilinear:multi}.
    Then by the same proof of Lemma \ref{lemma:relationship:bilinear} (cf. Appendix~\ref{appendix:bilinear}), we have
    \$
     & \sum_{t = 1}^T V_{f^t} - V^{\pi^t} \\
    & \qquad \le \Big[2\gamma_T(\epsilon,\cX)\sum_{t=1}^T \sum_{h=1}^H \sum_{j=1}^{t-1}  |\EE_{{\pi}_{\exp}(f^j, h)} l(f^t, \zeta_h)|^2\Big]^{1/2} +  2 \min\left\{2\gamma_T(\epsilon, \cX), HT\right\} + HT\epsilon,
    \$
    where the $l$ is defined in \eqref{eq:772}. So we know PO-bilinear class with multi-step futures also has a low GEC. The following algorithmic design and analysis are similar to the $1$-step case, and we omit them to avoid repetition.

\section{New $\ell_2$ Eluder Techniques} \label{appendix:eluder}

\subsection{Limitations of Previous Techniques}

We explain why the previous $\ell_1$ eluder argument in \citet{liu2022partially} fails in the posterior sampling scenario, which motivates us to develop new techniques. We first state the $\ell_1$ eluder argument in \citet{liu2022partially}.

\begin{lemma} \label{lemma:l1:eluder}
    Suppose $\left\{w_{t, j}\right\}_{(t, j) \in[T] \times[J]}  \subset \mathbb{R}^{d}$, and $\left\{x_{t, i}\right\}_{(t, i) \in[T] \times[I]} \subset \mathbb{R}^{d}$ satisfy
    \begin{itemize}
        \item $\sum_{s = 1}^{t - 1} \sum_{i = 1}^I \sum_{j = 1}^J | w_{t, j}^\top x_{s, i}| \le \gamma_t$,
        \item $\sum_{i = 1}^I \|x_{t, i}\|_2 \le R_x$,
        \item $\sum_{j = 1}^J \|w_{t,j}\|_2 \le R_w$,
    \end{itemize}
    for any $t \in [T]$. Then we have 
    \$
    \sum_{s = 1}^t \sum_{i = 1}^n \sum_{j = 1}^m | w_{s, j}^\top x_{s, i} |= \tilde{\cO}\Big( d \big( R_x R_w + \max_{s \le t} \gamma_s\big) \Big).
    \$
\end{lemma}

\begin{proof}
    See Proposition 22 in \citet{liu2022partially} for a detailed proof.
\end{proof}

\noindent\textbf{Why $\ell_1$ eluder technique is sufficient for OMLE?} \quad For the optimism-based algorithm OMLE \citep{liu2022partially,zhan2022pac}, choosing $\left\{w_{t, j}\right\}_{(t, j) \in[T] \times[J]}$, and $\left\{x_{t, i}\right\}_{(t, i) \in[T] \times[I]}$ as in \citet{liu2022partially,zhan2022pac}, standard MLE analysis implies that 
\# \label{eq:000}
\sum_{s=1}^{t-1} \Big(\sum_{i = 1}^I \sum_{j = 1}^J |w_{t,j}^\top x_{s, i}| \Big)^2 \le \beta, \quad \forall t \in [T],
\#
 for some fixed $\beta$. In \citet{liu2022partially,zhan2022pac}, $\beta$ only has the logarithmic dependency on $T$. By Cauchy-Schwarz inequality, \eqref{eq:000} gives $\sum_{s=1}^{t-1} \sum_{i = 1}^I \sum_{j = 1}^J |w_{t,j}^\top x_{s, i}|  \le \sqrt{\beta t}$. Together with Lemma~\ref{lemma:l1:eluder}, they can obtain a $\tilde{\cO}(d \sqrt{\beta T})$ upper bound for $\sum_{s = 1}^t \sum_{i = 1}^n \sum_{j = 1}^m | w_{s, j}^\top x_{s, i} |$, which is roughly the regret upper bound (up to polynomial factors).

\vspace{4pt}
\noindent\textbf{Challenges in posterior sampling algorithms.} \quad 
For posterior sampling algorithms, \eqref{eq:000} may not hold anymore. Let 
\$
\sum_{s=1}^{t-1} \Big(\sum_{i = 1}^I \sum_{j = 1}^J |w_{t,j}^\top x_{s, i}| \Big)^2 = \beta_t.
\$
Applying Cauchy-Schwarz inequality and Lemma~\ref{lemma:l1:eluder} only gives a $\tilde{\cO}(d \max_{t \le T} \sqrt{t \beta_t)}$ bound, which is different from the desired $\tilde{\cO}(\sqrt{d_{\GEC} (\sum_{t=1}^T \beta_t) } + \text{polynomial factors} \times \sqrt{T})$ in Definition~\ref{def:gec}. Moreover, to our best knowledge, previous $\ell_2$ eluder analysis \citep{russo2013eluder,jin2021bellman} cannot tackle this challenge, which motivates us to develop a new $\ell_2$ eluder technique.

\subsection{New $\ell_2$ Eluder Techniques} 
Recall that for any distribution $p \in \Delta_{[n]}$ and sequence $\{y_{i}\}_{i \in [n]}$, we denote $\sum_{i \sim p} y_i = \sum_{i = 1}^n p(i) y_i$.

\begin{lemma} \label{lemma:l2:eluder}
    Suppose $\left\{w_{t, j} \in \RR^d\right\}_{(t, j) \in[T] \times[J]}$, $\left\{x_{t, i} \in \RR^d \right\}_{(t, i) \in[T] \times[I]}$, and $\{p_t \in \Delta_{[I]}\}_{t \in [T]} $ satisfy
    \begin{itemize}
        \item $\sum_{s = 1}^{t - 1} \sum_{i \sim p_s} (\sum_{j = 1}^J   | w_{t, j}^\top x_{s, i}|)^2 \le \gamma_t$,
        \item $\sum_{i \sim p_t} \|x_{t, i}\|_2^2 \le R_x^2$,
        \item $\sum_{j = 1}^J \|w_{t,j}\|_2 \le R_w$,
    \end{itemize}
    for any $t \in [T]$. Then for $R > 0$ we have 
    \$
    \sum_{t = 1}^T R \land \sum_{i \sim p_t} \Big(\sum_{j = 1}^J  | w_{t, j}^\top x_{t, i} |\Big) \le 
    \bigg[2d \Big( R^2 T + \sum_{t = 1}^{T} \gamma_t\Big) \cdot \log\Big(1 + \frac{T R_x^2 R_w^2 }{R^2}\Big)\bigg]^{1/2}.
    \$
\end{lemma}

\begin{remark}
    First,  in order to better apply to our analysis of PSR/POMDP, we consider $\sum_{i \sim p}$ for a general distribution $p \in \Delta_{[I]}$ instead of $\sum_{i = 1}^I$/$\sum_{i \sim \mathrm{Unif}([I])}$.  Second, as a byproduct, our bound achieves a $\sqrt{d}$ improvement over previous $\ell_1$ eluder analysis (Lemma~\ref{lemma:l1:eluder}). Finally, when applied to the proof of Lemma~\ref{lemma:general:regular:psr} (Appendix~\ref{appendix:general:psr}), $x_{t, i}$ is independent of $t$ (i.e., for any $i \in [I]$, there exists $x_{i}$ such that $x_{t, i} = x_{i}$ for any $t \in [T]$), and here we consider the case where $x_{t,i}$ varies with $t$ for generality. 
\end{remark}

\begin{remark} \label{remark:issue}
    In our initial arXiv released version, the first constraint in Lemma~\ref{lemma:l2:eluder} is $$\sum_{s = 1}^{t - 1} \bigg (\sum_{i \sim p_s} \sum_{j = 1}^J   | w_{t, j}^\top x_{s, i}| \bigg )^2 \le \gamma_t$$ and we use a wrong inequality 
    $$\sum_{s = 1}^{t - 1} \sum_{i \sim p_s}  \bigg ( \sum_{j = 1}^J   | w_{t, j}^\top x_{s, i}| \bigg )^2  \le \sum_{s = 1}^{t - 1} \bigg (\sum_{i \sim p_s} \sum_{j = 1}^J   | w_{t, j}^\top x_{s, i}| \bigg)^2 \le \gamma_t$$ to finish the proof. In this version, we establish Lemma~\ref{lemma:l2:eluder} with the same proof but a stronger constraint. We thank the authors of \citet{chen2022partially} for pointing out this technical issue and our fix is motivated by \citet{chen2022partially}. We also remark that, in this version, we adjust Lemma~\ref{lemma:decomp:train} to better fit the revised lemma. The adjusted version of Lemmas \ref{lemma:decomp:train} and \ref{lemma:l2:eluder} correspond to Propositions D.2 and C.1 in \citet{chen2022partially}.
\end{remark}

\begin{proof}[Proof of Lemma \ref{lemma:l2:eluder}]
    In order to prove this lemma, it suffices to establish an upper bound for the following optimization problem 
    \begin{equation}
    \begin{aligned} \label{eq:8533}
     & \max _{w, x} \sum_{t = 1}^T R \land \sum_{i \sim p_t}  \Big(\sum_{j = 1}^J  | w_{t, j}^\top x_{t, i} | \Big) \\ 
    \text { subject to}  &\left\{\begin{array}{l} \sum_{s = 1}^{t - 1} \sum_{i \sim p_s} (\sum_{j = 1}^J   | w_{t, j}^\top x_{s, i}|)^2 \le \gamma_t, 
    \\ \sum_{i \sim p_t} \|x_{t, i}\|_2^2 \le R_x^2,  \\ 
    \sum_{j = 1}^J \|w_{t,j}\|_2 \le R_w. \end{array}\right.
    \end{aligned}
    \end{equation}
    Let $w^*_{t, i} = \sum_{j = 1}^J w_{t,j} \cdot \sign(w_{t,j}^\top x_{t, i})$. Then we know 
     \$
    \begin{aligned} \left\{\begin{array}{l} 
    \sum_{t = 1}^T R \land \sum_{i \sim p_t} | x_{t, i}^\top w_{t,i}^*| =  \sum_{t = 1}^T R \land \sum_{i \sim p_t} (\sum_{j = 1}^J  | w_{t, j}^\top x_{t, i} |),  \\
    \sum_{s = 1}^{t - 1} \sum_{i \sim p_s}  | x_{s, i}^\top w_{t,i}^*|^2 \le \sum_{s = 1}^{t - 1} \sum_{i \sim p_s}  ( \sum_{j = 1}^J  | w_{t, j}^\top x_{s, i}|)^2 \le \gamma_t, 
    \\ \sum_{i \sim p_t} \|x_{t, i}\|_2^2 \le R_x^2,   \\ 
     \|w^*_{t,i}\|_2 \le \sum_{j = 1}^J \|w_{t,j}\|_2 \le  R_w, \end{array}\right.\end{aligned}
    \$
    which implies that the optimal value of \eqref{eq:8533} is upper bounded by the optimal value of the following optimization problem:
    \begin{equation}
    \begin{aligned} \label{eq:8534}
     & \max _{w^*, x}  \sum_{t = 1}^T  R \land \sum_{i \sim p_t}  |x_{t, i}^\top w_{t,i}^*| \\ 
    \text { subject to}  &\left\{\begin{array}{l} \sum_{s=1}^{t-1} \sum_{i \sim p_s}   |x_{s, i}^\top w_{t,i}^*|^2 \le \gamma_t, 
    \\ \sum_{i \sim p_t} \|x_{t, i}\|_2^2 \le R_x^2, \\ 
     \|w_{t,i}^*\|_2 \le R_w. \end{array}\right.
    \end{aligned}
    \end{equation}
    Therefore, invoking Lemma \ref{lemma:eluder:sum} concludes our proof.
\end{proof}
    
\begin{lemma} \label{lemma:eluder:sum}
    The optimal value of optimization problem \eqref{eq:8534} is upper bounded by
    \$
    \Big[2d \Big( R^2 T + \sum_{t = 1}^{T} \gamma_t\Big) \cdot \log\Big(1 + \frac{T R_x^2 R_w^2 }{ R^2}\Big)\Big]^{1/2}.
    \$
\end{lemma}

\begin{proof}[Proof of Lemma \ref{lemma:eluder:sum}]
     Let $\Sigma_{t} = \lambda I_d + \sum_{s = 1}^{t-1} \sum_{i \sim p_s} x_{s, i} x_{s, i}^\top$ with $\lambda = R^2 /R_w^2$. We have
     \# \label{eq:8822}
    \sum_{t = 1}^T  R \land \sum_{i \sim p_t}  | x_{t, i}^\top w^*_{t, i} | & \le \sum_{t = 1}^T R  \land \sum_{i \sim p_t}   \big( \|x_{t, i}\|_{\Sigma_{t}^{-1}} \cdot  \|w_{t, i}^*\|_{\Sigma_{t}} \big) \notag \\
    & = \sum_{t = 1}^T  R  \land \sum_{i \sim p_t} \Big(  \|x_{t, i}\|_{\Sigma_{t}^{-1}} \cdot \Big[ \lambda \|w_{t,i}^*\|_2^2  + \sum_{s =1}^{t-1} \sum_{i \sim p_s} |x_{s, i}^\top w_{t,i}^*|^2 \Big]^{1/2} \Big) ,
    \#
    where the first inequality uses the Cauchy-Schwarz inequality, and the last equality uses the definition of $\Sigma_t$. Furthermore, by the first and last constraints in \eqref{eq:8534}, we have
    \# \label{eq:8823}
    \lambda \|w_{t,i}^*\|_2^2  + \sum_{s =1}^{t-1} \sum_{i \sim p_s} |x_{s, i}^\top w_{t,i}^*|^2  \le \lambda \cdot R_w^2 + \gamma_t = R^2  + \gamma_t,
    \#
    where the last equality is obtained by $\lambda = R^2/R_w^2$. Combining \eqref{eq:8822} and \eqref{eq:8823} gives that 
    \# \label{eq:8824}
    \sum_{t = 1}^T  R \land \sum_{i = 1}^I  | x_{t, i}^\top w^*_{t, i} | & \le \sum_{t = 1}^T R \land \sum_{i \sim p_t}  \Big[(R^2  + \gamma_t) \cdot \|x_{t, i}\|_{\Sigma_t^{-1}}^{2}\Big]^{1/2} \notag \\
    & \le \sum_{t = 1}^T R \land \Big[(R^2  + \gamma_t) \cdot \sum_{i \sim p_t} \|x_{t, i}\|_{\Sigma_t^{-1}}^{2}\Big]^{1/2} \notag \\
    & \le \sum_{t = 1}^T  \Big[(R^2 + \gamma_t) \cdot \Big(1 \land \sum_{i \sim p_t}\|x_{t, i}\|_{\Sigma_t^{-1}}^{2} \Big)\Big]^{1/2} \notag\\
    & \le \Big[\sum_{t=1}^T (R^2 + \gamma_t)\Big]^{1/2} \cdot \Big[\sum_{t=1}^T  1 \land \sum_{i \sim p_t} \|x_{t, i}\|_{\Sigma_t^{-1}}^2\Big]^{1/2} ,
    \#
    where the second inequality follows from Jensen's inequality, the third inequality uses the fact that $R \le \sqrt{R^2 + \gamma_t}$ for any $t \in [T]$, and the last inequality is implied by Cauchy-Schwarz inequality.
    Similar to the proof of elliptical potential lemma \citep{dani2008stochastic,rusmevichientong2010linearly,abbasi2011improved}, we further have
    \# \label{eq:8825}
    \sum_{t=1}^T  1 \land \sum_{i \sim p_t} \|x_{t, i}\|_{\Sigma_t^{-1}}^2 & = \sum_{t = 1}^T 1 \land \sum_{i \sim {p_t}} [x_{t, i}^\top \Sigma_{t}^{-1} x_{t, i}]  \notag \\
    & =  \sum_{t = 1}^T 1 \land \tr \Big( \Sigma_{t}^{-1/2}  \big( \sum_{i \sim {p_t}} x_{t, i} x_{t, i}^\top \big) \Sigma_{t}^{-1/2} \Big) \notag \\
    & = \sum_{t = 1}^T 1 \land \tr \Big(\Sigma_t^{-1/2} (\Sigma_{t + 1} - \Sigma_{t}) \Sigma_t^{-1/2}\Big), 
    \#
    where the last equality is implied by the definition of $\Sigma_{t}$. We further have
    \# \label{eq:88251}
    \sum_{t = 1}^T 1 \land \tr \Big(\Sigma_t^{-1/2} (\Sigma_{t + 1} - \Sigma_{t}) \Sigma_t^{-1/2}\Big)
    & \le 2  \sum_{t = 1}^T \log  \Big( 1 + \tr \big( \Sigma_t^{-1/2} (\Sigma_{t + 1} - \Sigma_{t}) \Sigma_t^{-1/2} \big)\Big)  \notag\\
    & \le 2 \sum_{t = 1}^T \log \det \Big( I_d +   \Sigma_t^{-1/2} (\Sigma_{t + 1} - \Sigma_{t}) \Sigma_t^{-1/2} \Big)  ,
    \# 
    where the first inequality uses the fact that $1 \land x \le 2 \log(1 +x)$ for any $x \ge 0$, and the second inequality follows from the fact that $1 + \tr(X) \le \det(I_d + X)$ for any positive semidefinite matrix $X \in \RR^{d \times d}$. Note that 
    \$
    \Sigma_{t+1} = \Sigma_t + (\Sigma_{t+1} - \Sigma_t) = \Sigma_t^{1/2} \big(I_d +  \Sigma_t^{-1/2} (\Sigma_{t + 1} - \Sigma_{t}) \Sigma_t^{-1/2} \big) \Sigma_t^{1/2},
    \$
    which implies that 
    \$
    \log \det \Big( I_d +   \Sigma_t^{-1/2} (\Sigma_{t + 1} - \Sigma_{t}) \Sigma_t^{-1/2} \Big) = \log \det(\Sigma_{t+1}) - \log \det(\Sigma_t).
    \$
    Combining this with \eqref{eq:8825} and \eqref{eq:88251}, we obtain
    \# \label{eq:88252}
    \sum_{t=1}^T  1 \land \sum_{i \sim p_t} \|x_{t, i}\|_{\Sigma_t^{-1}}^2 \le 2 \sum_{t = 1}^T \big( \log \det(\Sigma_{t+1}) - \log \det(\Sigma_t) \big) = 2 \log \Big( \frac{\det \Sigma_{T+1} }{\det \Sigma_1 } \Big).
    \#
    Moreover, by the assumption that $\sum_{i \sim p_t} \|x_{t, i}\|_2^2 \le R_x^2$ for any $t \in [T]$, we have
    \# \label{eq:88253}
    \Sigma_{T+1} = \lambda I_d + \sum_{t = 1}^T \sum_{i \sim p_t} x_{t, i} x_{t, i}^\top \preceq (T R_x^2 + \lambda) \cdot I_d.
    \#
    Putting \eqref{eq:88252}, \eqref{eq:88253} and the fact that $\lambda = R^2/R_w^2$ together, we have
    \# \label{eq:8826}
    \sum_{t=1}^T  1 \land \sum_{i \sim p_t} \|x_{t, i}\|_{\Sigma_t^{-1}}^2 &\le 2 \log\Big( \frac{T R_x^2 + \lambda}{\lambda}\Big)^d  = 2d \log\Big(1 + \frac{T R_x^2 R_w^2 }{ R^2}\Big) .
    \#
    Plugging \eqref{eq:8826} into \eqref{eq:8824} gives 
    \$
    \sum_{t = 1}^T  R \land \sum_{i \sim p_t} | x_{t, i}^\top w^*_{t, i} | \le \Big[2d \Big( R^2 T + \sum_{t = 1}^{T} \gamma_t\Big) \cdot \log\Big(1 + \frac{T R_x^2 R_w^2 }{R^2}\Big)\Big]^{1/2},
    \$
    which concludes the proof of Lemma \ref{lemma:eluder:sum}.
\end{proof}

\section{Proofs for Fully Observable Interactive Decision Making} 

\subsection{Proof for Model-free Approach} \label{sec:model_free_mdp}
The special denominator introduced in \eqref{eqn:posterior_model_free_mdp} originates from \citet{dann2021provably} and it allows us to divide both numerator and denominator by $\exp(-\eta \sum_{s=1}^{t-1} \sum_{h=1}^H l^2_{f^s}((\cP_h f_{h+1}, f_{h+1}), \zeta^s_h))$ as it is free of $f_h$. This allows us to replace $l^2_{f^s}((f_h, f_{h+1}), \zeta_h^s)$ with 
$$
\Delta L^s_h(f_h, f_{h+1}, \zeta_h^s) = l^2_{f^s}((f_h,f_{h+1}),\zeta_h^s) - l^2_{f^s}((\cP_h f_{h+1}, f_{h+1}), \zeta^s_h)
$$
in the theoretical analysis and this is the key to canceling the sampling variance as the conditional expectation of the second term is exactly the conditional variance of $l_{f^s}((f_h,f_{h+1}),\zeta_h^s)$. Therefore, we consider the following potential function:
\# \label{eq:def:potential:free}
{\Phi}_h^t(f) = -\log p^0_h(f_{h}) + \alpha \eta \sum_{s=1}^{t-1} \Delta L^s_h(f_{h}, f_{h+1}, \zeta_h^s) + \alpha \log \EE_{g_h \sim p^0_h} \exp \Big(-\eta \sum_{s=1}^{t-1} \Delta L^s_h(g_h, f_{h+1},\zeta_h^s)\Big),
\#
where $\alpha \in (0,1]$ is a tuning parameter. Therefore, the posterior can be expressed as:
\# \label{eq:posterior:model:free}
p^t(f) \propto \exp\Big(-\sum_{h=1}^H {\Phi}_h^t(f) + \gamma \Delta V_{1,f}(x_1)\Big).
\#
where $\Delta V_{1,f} (x_1) = V_{1,f}(x_1) - V_{1}^*(x_1)$. 

\begin{proof}[Proof of Theorem~\ref{thm:1}]
Throughout this proof, we use we use $d$ and $S^t = \{(f^s, x_h^s, a_h^s, r_h^s, x_{h+1}^s)\}_{(s, h) \in [t] \times [H]}$ to denote GEC and the data collected by exploration policies  $\{\pi_{\exp}(f^s,h)\}_{ (s, h) \in [t] \times [H]}$, respectively.  Recall the definition of generalized eluder coefficient (Definition~\ref{def:gec}), for any sequence of $\{f^t\}_{t=1}^T$, we have
\# \label{eq:45552}
    \sum_{t = 1}^T V_{1, f^t}(x_1) - V_1^{\pi_{f^t}}(x_1) &\le \Big[d \sum_{h=1}^H \sum_{t=1}^T \Big( \sum_{s=1}^{t-1} \EE_{{\pi}_{\mathrm{exp}}(f^s,h)} \ell_{f^s}(f^t, \xi_h) \Big)\Big]^{1/2} + 2\min\{d, HT\} + HT \epsilon' \notag \\
    &\leq  \mu \sum_{h=1}^H \sum_{t=1}^T \Big( \sum_{s=1}^{t-1} \EE_{{\pi}_{\mathrm{exp}}(f^s,h)} \ell_{f^s}(f^t, \xi_h) \Big) + \frac{d}{4\mu} + 2\min\{d, HT\} + H T\epsilon' \notag \\
    & = \frac{0.25\alpha \eta}{\gamma} \sum_{h=1}^H \sum_{t=1}^T \Big( \sum_{s=1}^{t-1} \EE_{{\pi}_{\mathrm{exp}}(f^s,h)} \ell_{f^s}(f^t, \xi_h) \Big) + \frac{\gamma d}{\alpha \eta} + 2\min\{d, HT\} + HT \epsilon'  ,
\#
where $\ell_{f^s}(f^t, \xi_h) = (\EE_{x_{h+1} \sim \PP_h(\cdot| x_h, a_h)} l_{f^s}((f^t_{h}, f^t_{h+1}), \zeta_{h}))^2$ with $\xi_h = (x_h, a_h)$ and $\zeta_h = (x_h, a_h, r_h, x_{h+1})$, the last inequality uses AM-GM inequality, and the equality is because we choose $\mu=0.25\alpha\eta/\gamma$. Then we have
{\small 
\# \label{eq:45553}
& \gamma \sum_{t = 1}^T \EE_{S^{t-1}} \EE_{f^t \sim p^t}  [V_1^*(x_1) - V_1^{\pi_{f^t}}(x_1)] \notag \\
& =  -\gamma \sum_{t = 1}^T \EE_{S^{t-1}} \EE_{f^t \sim p^t} \Delta V_{1,f^t}(x_1) + \gamma \sum_{t = 1}^T \EE_{S^{t-1}} \EE_{f^t \sim p^t} [V_{f^t}(x_1) - V^{\pi_{f^t}}(x_1) ] \notag \\
&  \le -\gamma \sum_{t = 1}^T \EE_{S^{t-1}} \EE_{f^t \sim p^t} \Delta V_{1,f^t}(x_1) + 0.25\alpha \eta \sum_{t = 1}^T \sum_{h=1}^H \sum_{s=1}^{t-1} \EE_{S^{t-1}} \EE_{f^t \sim p^t}\EE_{\pi_{\exp(f^s, h)}} \left( \EE_{x^s_{h+1} \sim \PP_h(\cdot| x^s_h, a^s_h)} l_{f^s}((f_{h}, f_{h+1}), \zeta^s_{h}) \right)^2 \notag \\
& \qquad + \frac{\gamma^2 d}{\alpha \eta} + 2 \gamma \min\{d, HT\} + \gamma HT \epsilon' ,
\#
}
where the first inequality uses the fact that $\Delta V_{1,f}(x_1) = V_{1, f}(x_1) - V_{1}^*(x_1)$ and the second inequality follows from \eqref{eq:45552} and Assumption~\ref{assumption:l}. Then we use the following lemma to upper bound first two terms in \eqref{eq:45553}.
\begin{lemma} \label{lem:low}
Suppose that Assumptions~\ref{assu:realizability},~\ref{assumption:l} and~\ref{assu:discrepancy} hold. If we set $\eta B^2 \leq 0.3$, then for any $t \in [T]$ we have 
$$ 
\begin{aligned}
& -\gamma \EE_{S^{t-1}} \EE_{f \sim p^t}  \Delta V_{1,f}(x_1) + 0.25 \alpha \eta \sum_{h=1}^H  \sum_{s=1}^{t-1} \EE_{S^{t-1}} \EE_{f \sim {p}^t} \EE_{\pi_{\exp(f^s, h)}} \left( \EE_{x^s_{h+1} \sim \PP_h(\cdot| x^s_h, a^s_h)} l_{f^s}(f_h, f_{h+1}, \zeta^s_{h}) \right)^2\\
&\qquad \le \EE_{S^{t-1}} \EE_{f \sim p^t} \Big[\sum_{h=1}^H \Phi_h^t(f) - \gamma \Delta V_{1,f}(x_1) + \log p^t(f)\Big]  + 2H \alpha B_l \eta (t-1) \epsilon + \sum_{h=1}^H\kappa^c_h(\alpha, \epsilon),
\end{aligned}
$$
where $\kappa^c_h(\alpha, \epsilon) = (1-\alpha) \log \EE_{f_{h+1} \sim p^0_{h+1}} p^0_h(\cF_h(\epsilon, f_{h+1}))^{-\alpha/(1-\alpha)}$.
\end{lemma}

\begin{proof}
    See Appendix~\ref{appendix:pf:lem:low} for a detailed proof.
\end{proof}

By Lemma~\ref{lem:low}, we obtain 
{\small 
\#\label{eq:45554}
 & -\gamma \sum_{t = 1}^T \EE_{S^{t-1}} \EE_{f^t \sim p^t} \Delta V_{1,f^t}(x_1) + 0.25\alpha \eta \sum_{t = 1}^T \sum_{h=1}^H \sum_{s=1}^{t-1} \EE_{S^{t-1}} \EE_{f^t \sim p^t}\EE_{\pi_{\exp(f^s, h)}} \left( \EE_{x^s_{h+1} \sim \PP_h(\cdot| x^s_h, a^s_h)} l_{f^s}\big((f_{h}, f_{h+1}), \zeta^s_{h}\big) \right)^2 \notag \\
 & \quad \le \sum_{t= 1}^T \mathbb{E}_{S^{t-1}} \mathbb{E}_{f \sim p^t}\Big(\sum_{h=1}^{H} \Phi_h^{t}(f)-\gamma \Delta V_{1,f}\left(x_1\right)+\log p^t(f)\Big)+H \alpha B_l \eta (T-1) T \epsilon + T \sum_{h=1}^{H} \kappa_{h}^c(\alpha, \epsilon) \notag \\
 & \quad = \sum_{t = 1}^T \mathbb{E}_{S^{t-1}} \inf _{p} \mathbb{E}_{f \sim p}\Big(\sum_{h=1}^{H} \Phi_{h}^{t}(f)-\gamma \Delta V_{1,f}\left(x_1\right)+\log p(f)\Big) + H \alpha B_l \eta (T-1) T \epsilon + T \sum_{h=1}^{H} \kappa_{h}^c(\alpha, \epsilon),
\#
}
where the last equality is implied by Lemma~\ref{lem:Gibbs_var} and \eqref{eq:posterior:model:free}. Hence, it remains to establish the upper bound for the first term in \eqref{eq:45554}, which is the purpose of the following lemma.

\begin{lemma} \label{lem:upper} Suppose that Assumptions~\ref{assu:realizability},~\ref{assumption:l} and~\ref{assu:discrepancy} hold. If we set $\eta B_l^2 \le 0.3$, then for any $t \in [T]$ we have
$$ \EE_{S^{t-1}} \inf_p \EE_{f \sim p} \Big[\sum_{h=1}^H {\Phi}_h^t(f) - \gamma \Delta V_{1,f}(x_1) + \log p(f)\Big] \leq \gamma \epsilon + \alpha \eta (t-1) H \epsilon^2 + \kappa^r(\epsilon).$$
\end{lemma}

\begin{proof}
    See Appendix~\ref{appendix:pf:lem:upper} for a detailed proof.
\end{proof}
Combining \eqref{eq:45553}, \eqref{eq:45554}, and Lemma~\ref{lem:upper}, we have
\# \label{eq:45555}
& \sum_{t = 1}^T \EE_{S^{t-1}} \EE_{f^t \sim p^t}  [V_1^*(x_1) - V_1^{\pi_{f^t}}(x_1) ]  \\
& \qquad \le  T \epsilon  + \frac{\alpha \eta \epsilon (\epsilon + B_l)T(T-1)H}{\gamma} + \frac{T}{\gamma}\kappa^r(\epsilon) + \frac{T}{\gamma} \sum_{h=1}^H \kappa_h^c(\alpha, \epsilon)  + \frac{\gamma d}{\alpha \eta} + 2\min\{d, HT\} + HT \epsilon' . \notag 
\#
 Since we execute exploration policies $\{\pi_{\exp}(f^t,h)\}_{ h \in [H]}$ to collect samples for any $t \in [T]$, we know the total number of interactions is $HT$ and 
 \# \label{eq:45556}
 \EE[\reg(HT)/H] = \sum_{t = 1}^T \EE_{S^{t-1}} \EE_{f^t \sim p^t}  [V_1^*(x_1) - V_1^{\pi_{f^t}}(x_1) ] .
 \#
 Then, we first let $\alpha \to 1^{-}$ and choose $\epsilon = 1/T^2, \epsilon' = 1/\sqrt{T}, \gamma=\sqrt{\frac{T\kappa(1/T^2, p^0)}{B_l^2 d}}$, and $\eta=1/(4B_l^2)$. Together with \eqref{eq:45555} and \eqref{eq:45556}, we obtain that
 $$
\EE[\reg(HT)] \leq \cO(B_l\sqrt{d H^2 T\cdot \kappa(1/T^2, p^0)}).
$$
Therefore, by rescaling $T$ to be $T/H$, we conclude the proof of Theorem~\ref{thm:1}. 
\end{proof}

\subsubsection{Proof of Lemma \ref{lem:low}} \label{appendix:pf:lem:low}
Before continuing, we first establish the following lemma about moment estimations, which will be useful later.
\begin{lemma} \label{lem:moment_control}
    Suppose that Assumptions~\ref{assumption:l} and~\ref{assu:discrepancy} hold. If we set $\eta B^2 \leq 0.3$, then we have $$\EE_{x^s_{h+1} \sim \PP_h(\cdot| x^s_h, a^s_h)} [\Delta L^s_h(f_h, f_{h+1}, \zeta^s_h) ] =\big[\EE_{x^s_{h+1} \sim \PP_h(\cdot| x^s_h, a^s_h)} l_{f^s}\big((f_h,f_{h+1}), \zeta^s_h\big) \big]^2.$$ It further holds that 
    $$
    \log \EE_{x^s_{h+1} \sim \PP_h(\cdot| x^s_h, a^s_h)} [ \exp\left(-\eta \Delta L^s_h(f_h, f_{h+1}, \zeta_h^s)\right) ] \leq -0.25 \eta \big( \EE_{x^s_{h+1} \sim \PP_h(\cdot| x^s_h, a^s_h)} l_{f^s}((f_{h}, f_{h+1}), \zeta_h^s) \big)^2.
    $$
    Similarly, it holds that 
    $$
    \log \EE_{\zeta_h^s \sim \pi_{\exp(f^s, h)}} [ \exp\left(-\eta \Delta L^s_h(f_h, f_{h+1}, \zeta_h^s)\right) ] \leq -0.25 \eta \EE_{\pi_{\exp(f^s,h)}} \big( \EE_{x^s_{h+1} \sim \PP_h(\cdot| x^s_h, a^s_h)} l_{f^s}((f_{h}, f_{h+1}), \zeta_h^s) \big)^2.
    $$
\end{lemma}
\begin{proof}
We denote $\EE_h = \EE_{x^s_{h+1} \sim \PP_h(\cdot| x^s_h, a^s_h)}$ to simplify the notation. We first note that by the generalized completeness, we have
$$
\EE_h l^2_{f^s}\big((\cP_h f_{h+1}, f_{h+1}), \zeta^s_h\big)^2 = \EE_h \Big\{l_{f^s}\big((f_{h}, f_{h+1}), \zeta^s_h\big) - \EE_h \big[l_{f^s}\big((f_{h}, f_{h+1}), \zeta^s_h\big)\big]\Big\}^2,
$$
where the right-hand side is the conditional variance from state transition. Therefore, by the definition of variance, it holds that 
\begin{equation*}
\begin{aligned} \label{eq:77770}
\EE_{h} [\Delta L^s_h(f_h, f_{h+1}, \zeta^s_h) ] &= \EE_h \Big[l^2_{f^s}\big((f_h,f_{h+1}),\zeta_h^s\big) - l^2_{f^s}\big((\cP_h f_{h+1}, f_{h+1}), \zeta^s_h\big)\Big]\\
&=\Big[\EE_h l_{f^s}\big((f_h,f_{h+1}), \zeta^s_h\big) \Big]^2.
\end{aligned}
\end{equation*}
We proceed to estimate the second-order moment: 
$$
\begin{aligned}
&\EE_h [\Delta L^s_h( f_h, f_{h+1}, \zeta_h^s)^2]\\
=& \EE_h \Big\{\Big[l_{f^s}\big((f_{h}, f_{h+1}), \zeta^s_h\big) - l_{f^s}\big((\cP_h f_{h+1},f_{h+1}), \zeta^s_h\big) \Big]^2\Big[l_{f^s}\big((f_h,f_{h+1}), \zeta^s_h\big) + l_{f^s}\big((\cP_h f_{h+1}, f_{h+1}), \zeta^s_h\big)\Big]^2 \Big\}\\
\leq& 4B_l^2\Big[ \EE_h l_{f^s}\big((f_h,f_{h+1}), \zeta_h^s\big) \Big]^2,
\end{aligned}
$$
where we use generalized completeness in the equality and the boundedness of the discrepancy function in the inequality. Note that $\psi(z) = (e^z - 1 - z) / z^2$ is increasing in $z$. Therefore, when $\eta$ is sufficiently small such that $\eta B_l^2 \leq 0.3$, we can control the $\exp(-\eta \Delta L^s_h(f_h, f_{h+1}, \zeta_h^s))$ as follows:
\begin{equation} \label{eqn:dann_online_lemma}
\begin{aligned}
\log \EE_h [ \exp\left(-\eta \Delta L^s_h(f_h, f_{h+1}, \zeta_h^s)\right) ] &\leq \EE_h \left[\exp\left(-\eta \Delta L^s_h(f_h, f_{h+1}, \zeta_h^s\right) - 1\right]\\
&\leq \EE_h \left[- \eta \Delta L^s_h(f_h, f_{h+1}, \zeta_h^s) + 0.56 \eta^2 \Delta L^s_h(f_h,f_{h+1}, \zeta_h^s)^2 \right] \\
&\leq -0.25 \eta \big( \EE_h l_{f^s}((f_{h}, f_{h+1}), \zeta_h) \big)^2,
\end{aligned}
\end{equation}
where the first inequality uses $\log z \leq z - 1$; the second inequality uses $e^z \leq 1 + z + \psi(0.3)z^2$ and $\psi(0.3) \leq 0.56$; the last inequality uses $\eta B_l^2 \leq 0.3$ so that $(4 \times 0.56 \times \eta B_l^2) \leq 0.75$. By replacing $\EE_h$ with $\EE_{\zeta_h^s \sim \pi_{\exp(f^s,h)}}$, the second result follows from the same proof.
\end{proof}

\begin{proof}[Proof of Lemma \ref{lem:low}]
We start with the following decomposition, whose proof can be found in \citet{dann2021provably} or \citet{xiong22b}. 
\begin{equation} \label{eqn:lower_bound_decom}
\begin{aligned}
&\mathbb{E}_{S^{t-1}} \mathbb{E}_{f \sim p^{t}}\Big(\sum_{h=1}^{H} \Phi_{h}^{t}(f)-\gamma \Delta V_{1,f}\left(x_{1}\right)+\log {p}^{t}(f)\Big) \\
&\qquad\geq \underbrace{\mathbb{E}_{S^{t-1}} \mathbb{E}_{f \sim p^t}\Big[-\gamma \Delta V_{1,f}\left(x_{1}\right)+(1-0.5 \alpha) \log \frac{{p}^{t}\left(f_{1}\right)}{p^{0}_{1}\left(f_{1}\right)}\Big]}_{A} \\
&\qquad\qquad+\sum_{h=1}^{H} \underbrace{0.5 \alpha \mathbb{E}_{S^{t-1}} \mathbb{E}_{f \sim p^{t}}\Big[\eta \sum_{s=1}^{t-1} 2 \Delta L^s_{h}\left(f_{h}, f_{h+1}, \zeta_h^s\right)+\log \frac{p^{t}\left(f_{h}, f_{h+1}\right)}{p^{0}_{h}\left(f_{h}\right) p^{0}_{h+1}\left(f_{h+1}\right)}\Big]}_{B_{h}} \\
&\qquad\qquad+\sum_{h=1}^{H} \underbrace{\mathbb{E}_{S^{t-1}} \mathbb{E}_{f \sim p^t}\Big[\alpha \log \mathbb{E}_{\tilde{f}_{h} \sim p^0_h} \exp \Big(-\eta \sum_{s=1}^{t-1} \Delta L^s_h(\tilde{f}_{h}, f_{h+1}, \zeta_h^s)\Big)+(1-\alpha) \log \frac{{p}^{t}\left(f_{h+1}\right)}{p^0_{h+1}\left(f_{h+1}\right)}\Big]}_{C_{h}} .
\end{aligned}
\end{equation}
Then, we bound the three parts separately. As the KL-divergence is non-negative and $\alpha \in (0, 1]$, we know that we can bound term $A$ as follows: 
\begin{equation} \label{eqn:dann_A}
A \geq -\gamma \EE_{S^{t-1}} \EE_{f \sim p^t} \Delta V_{1,f}(x_1).    
\end{equation}
\textbf{Bounding term $B_h$.} We first define 
\# \label{eq:def:tilde:xi}
\tilde{\xi}_{h}^{s}\left(f_{h}, f_{h+1}, \zeta^{s}_h\right)=&-2 \eta \Delta L^s_h\left(f_{h}, f_{h+1}, \zeta_h^s\right) -\log \mathbb{E}_{\zeta_h^s \sim \pi_{\exp(f^s,h)}} \exp \left(-2 \eta \Delta L^s_h\left(f_{h}, f_{h+1}, \zeta_h^s\right)\right).
\#
Then we have the following lemma.

\begin{lemma} \label{lemma:zhang:free}
    For any fixed  $(f, h) \in \cH \times [H]$, it holds for any $t \in [T]$ that 
    $$\EE_{S^{t-1}} \exp\Big(\sum_{s=1}^{t-1} \tilde{\xi}_h^s(f_h, f_{h+1}, \zeta_h^s)\Big) = 1.$$
\end{lemma}

\begin{proof}
    The proof follows \citet{zhang2006,zhang2022mathematical} and we present the detailed proof for completeness.  We prove this lemma by induction. Suppose 
    $$\EE_{S^{t-1}} \exp\Big(\sum_{s=1}^{t-1} \tilde{\xi}_h^s(f_h, f_{h+1}, \zeta_h^s)\Big) = 1$$
    holds. Then we have
    \$
    \EE_{S^{t}} \exp\Big(\sum_{s=1}^{t} \tilde{\xi}_h^s(f_h, f_{h+1}, \zeta_h^s)\Big) &= \EE_{S^{t-1}} \exp\Big(\sum_{s=1}^{t-1} \tilde{\xi}_h^s(f_h, f_{h+1}, \zeta_h^s)\Big) \cdot \EE_{f^t \sim p^t} \EE_{\zeta_h^t \sim \pi_{\exp(f^t,h)}} \exp\big( \tilde{\xi}_h^t(f_h^t, f^t_{h+1}, \zeta_h^t) \big) \\
    & = \EE_{S^{t-1}} \exp\Big(\sum_{s=1}^{t-1} \tilde{\xi}_h^s(f_h, f_{h+1}, \zeta_h^s)\Big) = 1,
    \$
    where the second equality follows from the definition of $\tilde{\xi}_h^t$ in \eqref{eq:def:tilde:xi}. Therefore, we conclude the proof of Lemma~\ref{lemma:zhang:free}.
\end{proof}

By Lemma~\ref{lem:Gibbs_var}, we have
$$
\begin{aligned}
& \mathbb{E}_{f \sim p^t}\Big[\sum_{s=1}^{t-1}-\tilde{\xi}_{h}^{s}\left(f_{h}, f_{h+1}, \zeta_h^s\right)+\log \frac{p^t\left(f_{h}, f_{h+1}\right)}{p^0_h\left(f_{h}\right) p^0_{h+1}\left(f_{h+1}\right)}\Big] \\
&\qquad \geq \inf_{p} \mathbb{E}_{f \sim p}\Big[\sum_{s=1}^{t-1}-\tilde{\xi}_{h}^{s}\left(f_{h}, f_{h+1}, \zeta_h^s\right)+\log \frac{p\left(f_{h}, f_{h+1}\right)}{p^0_h\left(f_{h}\right) p^0_{h+1}\left(f_{h+1}\right)}\Big] \\
&\qquad=-\log \mathbb{E}_{f_{h} \sim p^0_h} \mathbb{E}_{f_{h+1} \sim p^0_{h+1}} \exp \Big(\sum_{s=1}^{t-1} \tilde{\xi}_{h}^{s}\left(f_{h}, f_{h+1}, \zeta_h^s\right)\Big),
\end{aligned}
$$
where we use the fact that Lemma~\ref{lem:Gibbs_var} implies that the $\inf$ is achieved at 
$$
p\left(f_{h}, f_{h+1}\right) \propto p^0_h\left(f_{h}\right) p^0_{h+1}\left(f_{h+1}\right) \exp \Big(\sum_{s=1}^{t-1} \tilde{\xi}^{s}_{h}\left(f_{h}, f_{h+1}, \zeta_h^s\right)\Big).
$$
It then follows that 
\begin{equation} \label{eqn:bh_lemma}
\begin{aligned}
& \mathbb{E}_{S^{t-1}} \mathbb{E}_{f \sim p^t}\Big[\sum_{s=1}^{t-1}-\tilde{\xi}_{h}^{s}\left(f_{h}, f_{h+1}, \zeta_h^s\right)+\log \frac{p^t\left(f_{h}, f_{h+1}\right)}{p^0_h\left(f_{h}\right) p^0_{h+1}\left(f_{h+1}\right)}\Big] \\
&\qquad\geq -\mathbb{E}_{S^{t-1}} \log \mathbb{E}_{f_{h} \sim p^0_h} \mathbb{E}_{f_{h+1} \sim p^0_{h+1}} \exp \Big(\sum_{s=1}^{t-1} \tilde{\xi}_{h}^{s}\left(f_{h}, f_{h+1}, \zeta_h^s\right)\Big) \\
&\qquad\geq -\log \mathbb{E}_{f_{h} \sim p^0_h} \mathbb{E}_{f_{h+1} \sim p^0_{h+1}} \mathbb{E}_{S^{t-1}} \exp \Big(\sum_{s=1}^{t-1} \tilde{\xi}_{h}^{s}\left(f_{h}, f_{h+1}, \zeta_h^s\right)\Big)=0,
\end{aligned}
\end{equation}
where the first inequality uses the above result, the second inequality is due to the convexity of $-\log(\cdot)$ and Jensen's inequality, and the last equality follows from Lemma~\ref{lemma:zhang:free}. Then, we can control $B_h$ as follows:
\begin{equation}\label{lem:dann_B}
\begin{aligned}
    B_h &= 0.5 \alpha \mathbb{E}_{S^{t-1}} \mathbb{E}_{f \sim p^t}\Big[\eta \sum_{s=1}^{t-1} 2 \Delta L^s_h\left(f_{h}, f_{h+1}, \zeta_h^s\right)+\log \frac{p^t\left(f_{h}, f_{h+1}\right)}{p^0_h\left(f_{h}\right) p^0_{h+1}\left(f_{h+1}\right)}\Big]\\
    &\geq 0.5\alpha \EE_{S^{t-1}} \EE_{f \sim p^t} \sum_{s=1}^{t-1}-\log \EE_{\pi_{\exp(f^s,h)}} \exp \big(-2 \eta \Delta L^s_h\left(f_{h}, f_{h+1}, \zeta_h^s\right)\big)\\
    &\geq 0.25 \alpha \eta \sum_{s=1}^{t-1} \EE_{S^{t-1}} \EE_{f \sim p^t} \EE_{\pi_{\exp(f^s,h)}} \left( \EE_{x^s_{h+1} \sim \PP_h(\cdot| x^s_h, a^s_h)} l_{f^s}\big((f_{h}, f_{h+1}), \zeta^s_{h}\big) \right)^2,
\end{aligned}
\end{equation}
where the first inequality uses \eqref{eqn:bh_lemma} and the definition of $\tilde{\xi}^s_h(\cdot)$, and the second inequality uses Lemma~\ref{lem:moment_control}.

\vspace{4pt}

\noindent\textbf{Bounding term $C_h$.} By Lemma~\ref{lem:Gibbs_var}, we know that 
\# \label{eqn:c_h}
& \mathbb{E}_{f \sim p^t}\Big[\alpha \log \mathbb{E}_{\tilde{f}_h \sim p^0_h} \exp \Big(-\eta \sum_{s=1}^{t-1} \Delta L^s_h(\tilde{f}_h, f_{h+1}, \zeta_h^s)\Big)+(1-\alpha) \log \frac{p^t\left(f_{h+1}\right)}{p^0_{h+1}\left(f_{h+1}\right)}\Big] \notag \\
&\qquad \geq (1-\alpha) \inf _{p_h} \mathbb{E}_{f \sim p_h}\Big[\frac{\alpha}{1-\alpha} \log \mathbb{E}_{\tilde{f}_h \sim p^0_h} \exp \Big(-\eta \sum_{s=1}^{t-1} \Delta L^s_h(\tilde{f}_h, f_{h+1}, \zeta_h^s)\Big)+\log \frac{p_h\left(f_{h+1}\right)}{p^0_{h+1}\left(f_{h+1}\right)}\Big] \notag \\
&\qquad=-(1-\alpha) \log \mathbb{E}_{f_{h+1} \sim p^0_{h+1}}\Big(\mathbb{E}_{f_{h} \sim p^0_h} \exp \Big(-\eta \sum_{s=1}^{t-1} \Delta L^s_h(f_{h}, f_{h+1}, \zeta_h^s)\Big)\Big)^{-\alpha /(1-\alpha)},
\#
where we use the fact that the $\inf$ is achieved at 
$$
p_h\left(f_{h+1}\right) \propto p^0_{h+1}\left(f_{h+1}\right) \Big(\mathbb{E}_{f_{h} \sim p^0_h} \exp \Big(-\eta \sum_{s=1}^{t-1} \Delta L^s_h\left(f_{h}, f_{h+1}, \zeta_h^s\right)\Big)\Big)^{-\alpha /(1-\alpha)}.
$$
We now consider a fixed $f_h \in \cF_h(\epsilon, f_{h+1})$, where $ \cF_h(\epsilon, f_{h+1})$ is defined in Assumption~\ref{assu:discrepancy} and takes form $$ \cF_h(\epsilon, f_{h+1}) := \{f_h \in \cF_h: \sup_{x_h, a_h,f^s} |\EE_{x_{h+1}|x_h,a_h} l_{f^s}((f_{h}, f_{h+1}), \zeta_h)| \leq \epsilon\}.$$ 
Note that by the generalized completeness condition (Assumption~\ref{assu:discrepancy}), we know this set is non-empty. We have the following estimation:
    \# \label{eqn:good_fun}
&\Delta L^s_h(f_{h},f_{h+1},\zeta_h^s) \notag \\
&= \Big[l_{f^s}((f_h, f_{h+1}), \zeta_h^s) - l_{f^s}((\cP_h f_{h+1}, f_{h+1}), \zeta_h^s)\Big]\Big[ l_{f^s}((f_h, f_{h+1}), \zeta_h^s) + l_{f^s}((\cP_h f_{h+1}, f_{h+1}), \zeta_h^s)\Big] \notag \\
&= \EE_{h} \Big[l_{f^s}((f_h, f_{h+1}), \zeta_h^s)\Big]\Big[ l_{f^s}((f_h, f_{h+1}), \zeta_h^s) + l_{f^s}((\cP_h f_{h+1}, f_{h+1}), \zeta_h^s)\Big] \notag \\
&\leq 2B_l\epsilon.
\#
Combining \eqref{eqn:c_h}, \eqref{eqn:good_fun}, and the definition of Term $C_h$, we have 
\begin{equation}
    \label{eqn:dann_C}
\begin{aligned}
C_h &= \mathbb{E}_{f \sim p^t}\Big[\alpha \log \mathbb{E}_{\tilde{f}_h \sim p^0_h} \exp \Big(-\eta \sum_{s=1}^{t-1} \Delta L^s_h(\tilde{f}_h, f_{h+1}, \zeta_h^s)\Big)+(1-\alpha) \log \frac{p^t\left(f_{h+1}\right)}{p^0_{h+1}\left(f_{h+1}\right)}\Big]\\
&\geq -(1-\alpha) \log \EE_{f_{h+1} \sim p^0_{h+1}} \Big(\EE_{f_{h} \sim p^0_h} \exp \Big(-\eta \sum_{s=1}^{t-1} \Delta L^s_h(f_{h}, f_{h+1}, \zeta_h^s)\Big)\Big)^{-\alpha/(1-\alpha)}\\
&\geq -(1-\alpha) \log \EE_{f_{h+1} \sim p^0_{h+1}} \Big(p_h^0(\cF_h(\epsilon,f_{h+1})) \exp(-2\eta (t-1)B_l\epsilon)\Big)^{-\alpha/(1-\alpha)}, 
\end{aligned}
\end{equation}
Moreover, by Jensen's inequality, \eqref{eqn:dann_C} implies that 
\begin{equation}
    \begin{aligned} \label{eq:77771}
C_h &\ge \alpha \log \exp(-2\eta (t-1)B_l \epsilon)  -(1-\alpha)\log \EE_{f_{h+1} \sim p^0_{h+1}} p^0_h(\cF_h(\epsilon, f_{h+1}))^{-\alpha/(1-\alpha)}\\
&= -2 \alpha B_l \eta (t-1) \epsilon - \kappa^c_h(\alpha, \epsilon). 
    \end{aligned}
\end{equation}
where the last equality follows from the definition of $\kappa^c_h(\alpha, \epsilon)$:
$$\kappa^c_h(\alpha, \epsilon) = (1-\alpha) \log \EE_{f_{h+1} \sim p^0_{h+1}} p^0_h(\cF_h(\epsilon, f_{h+1}))^{-\alpha/(1-\alpha)}.$$ Combining the decomposition in \eqref{eqn:lower_bound_decom} with \eqref{eqn:dann_A}, ~\eqref{lem:dann_B} and~\eqref{eq:77771}, we conclude the proof of Lemma~\ref{lem:low}.
\end{proof}

\subsubsection{Proof of Lemma \ref{lem:upper}} \label{appendix:pf:lem:upper}

\begin{proof}[Proof of Lemma \ref{lem:upper}]
We also denote $\EE_h = \EE_{x^s_{h+1} \sim \PP_h(\cdot| x^s_h, a^s_h)}$ to simplify the notation. Consider any fixed $f \in \cH$ and for any $\tilde{f}_h \in \cF_h$ that depends on $S^{s-1}$, by \eqref{eqn:dann_online_lemma}, we have 
$$
\EE_{\zeta^s_h} \exp\left(-\eta \Delta L^s_h(\tilde{f}_h, f_{h+1}, \zeta^s_h)\right) - 1 \leq - 0.25 \eta \EE_{\zeta^s_h}\left( \EE_h l_{f^s}((f_h,f_{h+1}), \zeta^s_h) \right)^2 \leq 0,
$$
which has been established when we independently prove Lemma~\ref{lem:moment_control}. Now we define 
$$
W_h^t = \EE_{S^t} \EE_{f \sim p} \log \EE_{\tilde{f}_h\sim p^0_h} \exp \Big(-\eta \sum_{s=1}^t \Delta L^s_h(\tilde{f}_h, f_{h+1}, \zeta^s_h)\Big),
$$
and we recall that the conditional distribution is defined to be 
$$
{q}_h^t (\tilde{f}_h | f_{h+1}, S^{t-1}) := \frac{p^0_h(\tilde{f}_h) \exp \big(-\eta \sum_{s=1}^{t-1} \Delta L^s_h(\tilde{f}_h, f_{h+1}, \zeta_h^s)\big)}{ \EE_{\tilde{f}' \sim p^0_h} \exp \big(-\eta \sum_{s=1}^{t-1} \Delta L^s_h(\tilde{f}’, f_{h+1}, \zeta_h^s)\big)}.
$$
It follows that 
\begin{align*}
    W^{s}_{h}-W^{s-1}_{h} &= \mathbb{E}_{S^{s}} \mathbb{E}_{f \sim p} \log \EE_{\tilde{f}_h \sim p^0_h} \frac{\exp \big(-\eta \sum_{t=1}^{s-1} \Delta L^s_h(\tilde{f}_{h}, f_{h+1}, \zeta^{t})\big)}{\mathbb{E}_{\tilde{f}' \sim p^0_h} \exp \big(-\eta \sum_{t=1}^{s-1} \Delta L^s_h(\tilde{f}^{\prime}, f_{h+1}, \zeta^{t})\big)} \exp \big(-\eta \Delta L^s_h(\tilde{f}_{h}, f_{h+1}, \zeta_h^s)\big)\\
&=\mathbb{E}_{S^{s}} \mathbb{E}_{f \sim p} \log \mathbb{E}_{\tilde{f}_{h} \sim {q}^{s}_{h}\left(\cdot \mid f_{h+1}, S^{s-1}\right)} \exp \big(-\eta \Delta L^s_h(\tilde{f}_{h}, f_{h+1}, \zeta_h^s)\big) \\
&\leq \mathbb{E}_{S^{s}} \mathbb{E}_{f \sim p}\left(\mathbb{E}_{\tilde{f}_{h} \sim {q}^{s}_{h}\left(\cdot \mid f_{h+1}, S^{s-1}\right)} \exp \big(-\eta \Delta L^s_h(\tilde{f}_{h}, f_{h+1}, \zeta_h^s)\big)-1\right) \leq 0.
\end{align*}
Thus, we obtain that 
$$W_h^t = W_h^0 + \sum_{s=1}^t [W_h^s - W^{s-1}_h] \leq 0,$$
and equivalently, we have
\# \label{eq:7777}
\EE_{S^t} \EE_{f \sim p} \log \EE_{\tilde{f}_h \sim p^0_h} \exp\Big(-\eta \sum_{s=1}^t \Delta L^s_h(\tilde{f}_h, f_{h+1}, \zeta_h^s)\Big) \leq 0,
\#
where $p$ is arbitrary. We proceed to show that 
$$
\begin{aligned}
& \mathbb{E}_{S^{t-1}} \mathbb{E}_{f \sim p}\Big[\sum_{h=1}^{H} \Phi_{h}^{t}(f)-\gamma \Delta V_{1,f}\left(x_{1}\right)+\log p(f)\Big] \\
&\qquad= \mathbb{E}_{S^{t-1}} \mathbb{E}_{f \sim p}\Big[-\gamma \Delta V_{1,f}\left(x_1\right)+\alpha \eta \sum_{h=1}^{H} \sum_{s=1}^{t-1} \Delta L^s_h\left(f_{h}, f_{h+1}, \zeta_h^s\right)\\
&\qquad\qquad+\alpha \sum_{h=1}^{H} \log \mathbb{E}_{\tilde{f}_{h} \sim p^0_h} \exp \Big(-\eta \sum_{s=1}^{t-1} \Delta L^s_h(\tilde{f}_{h}, f_{h+1}, \zeta_h^s)\Big)+\log \frac{p(f)}{p_{0}(f)}\Big] \\
&\qquad \leq \mathbb{E}_{S^{t-1}} \mathbb{E}_{f \sim p}\Big[-\gamma \Delta V_{1,f}\left(x_1\right)+\sum_{h=1}^{H} \alpha \eta \sum_{s=1}^{t-1} \big[ \EE_h l_{f^s}\big((f_h, f_{h+1}), \zeta_h^s\big) \big]^2 + \log \frac{p(f)}{p_{0}(f)}\Big] ,
\end{aligned}
$$
where the first equality follows from the definition of $\Phi_h^t$ in \eqref{eq:def:potential:free}, and the inequality uses Lemma~\ref{lem:moment_control} and \eqref{eq:7777}.
Since $p$ is arbitrary, we consider the following hypothesis class $\cH(\epsilon)$ in Definition~\ref{def:kappa}: $$\cH(\epsilon) = \{f \in \cH: V_{1,f^*}(x_1)-V_{1,f}(x_1) < \epsilon; \sup_{x_h, a_h, h,f'} |\EE_{x_{h+1} \sim \PP_h(\cdot|x_h,a_h)} l_{f'}((f_h, f_{h+1}), \zeta_h)| \leq \epsilon\},$$  and take $p(f) = p^0(f) \cdot \mathbf{1}(f \in \cH(\epsilon))/p^0(\cH(\epsilon))$. It follows that 
$$
\begin{aligned}
& \EE_{S^{t-1}} \inf_{p} \EE_{f \sim {p}} \Big[\sum_{h=1}^H {\Phi}_h^t(f) - \gamma \Delta V_{1,f} (x_1) + \log p(f)\Big] \leq \gamma \epsilon + \alpha \eta (t-1) H \epsilon^2 + \kappa^r(\epsilon),
\end{aligned}
$$
which concludes the proof of Lemma~\ref{lem:upper}.
\end{proof}

\subsection{Proof for Model-based Approach} \label{appendix:pf:model:mdp}
To facilitate our analysis, we define 
    \# \label{eq:def:delta:l}
    \Delta L_h^t(f) &= L_h^t(f) - L_h^t(f^*), \quad 
    \Delta V_f = V_f - V^*,
    \#
    where $V_f$ is the optimal value in the model $f$ and $V^*$ is the optimal value in the real model $f^*$. Under these notations, we have 
    \# \label{eq:4800}
    p^t(f) \propto p^0(f) 
    \exp\Big(\gamma \Delta V_f + \sum_{s=1}^{t-1} \sum_{h = 1}^H \Delta L_h^s(f)\Big) .
    \#

\begin{proof}[Proof of Theorem \ref{thm:mdp}]
   Throughout this proof, we use $d$ and $S^t = \{(f^s, x_h^s, a_h^s, r_h^s, x_{h+1}^s)\}_{(s, h) \in [t] \times [H]}$ to denote GEC and the data collected by exploration policies  $\{\pi_{\exp}(f^s,h)\}_{ (s, h) \in [t] \times [H]}$, respectively.  Since we execute exploration policies $\{\pi_{\exp}(f^t,h)\}_{ h \in [H]}$ to collect samples for any $t \in [T]$, we know the total number of interactions is $HT$. We decompose the regret as
\# \label{eq:481}
&\frac{\gamma}{H} \cdot \EE[\reg(HT)] \notag\\
&  = \gamma \sum_{t = 1}^T \EE_{S^{t - 1}} \EE_{f \sim p^t} [V^* - V_{f^*}^{\pi_f}] \notag\\
&  = \sum_{t = 1}^T \EE_{S^{t - 1}} \EE_{f \sim p^t} [- \gamma \Delta V_f] + \gamma \sum_{t = 1}^T \EE_{S^{t - 1}} \EE_{f \sim p^t} [V_f - V_{f^*}^{\pi_f} ] \notag\\
&  \le \sum_{t = 1}^T  \EE_{S^{t - 1}} \EE_{f \sim p^t} [- \gamma \Delta V_f] + \sum_{t = 1}^T  \sum_{h = 1}^H \sum_{s = 1}^{t - 1} \EE_{S^{t - 1}} \EE_{f \sim p^t} \EE_{\pi_{\exp(f^s, h)}} \big[ D_H^2 \big( \PP_{h, f}(x_{h+1}^s \mid x_h^s, a_h^s), \PP_{h, f^*}(x_{h+1}^s \mid x_h^s, a_h^s)\big) \big] \notag \\
& \qquad + \frac{\gamma^2 d}{4} + 2\min\{d, HT\} + \epsilon H T,
\#
where the first equality follows from the definition of $\Delta V_f$ in \eqref{eq:def:delta:l} and the last inequality uses Definition~\ref{def:gec}. Then we use the following lemma to bound the first two terms in \eqref{eq:481}.
\begin{lemma} \label{lemma:mdp:key}
    For any $t \in [T]$, it holds that 
    \$
    &\EE_{S^{t - 1}}\EE_{f \sim p^t} \Big[ - \sum_{h = 1}^H \sum_{s = 1}^{t - 1} \Delta L_h^s(f)  - \gamma \Delta V_f + \log \frac{p^t(f) }{p^0(f) } \Big] \\
    & \qquad \ge \EE_{S^{t - 1}} \EE_{f \sim p^t} [ - \gamma \Delta V_f ] 
    +  \sum_{h = 1}^H \sum_{s = 1}^{t - 1} \EE_{S^{t - 1}} \EE_{f \sim p^t} \EE_{\pi_{\exp(f^s, h)}} \big[ D_H^2 \big( \PP_{h, f}(x_{h+1}^s \mid x_h^s, a_h^s), \PP_{h, f^*}(x_{h+1}^s \mid x_h^s, a_h^s)\big) \big].
    \$
\end{lemma}

\begin{proof}
    See Appendix \ref{appendix:pf:mdp:key} for a detailed proof.
\end{proof}

By Lemma \ref{lemma:mdp:key}, we further obtain 
\# \label{eq:482}
& \sum_{t = 1}^T  \EE_{S^{t - 1}} \EE_{f \sim p^t} [- \gamma \Delta V_f] + \sum_{t = 1}^T  \sum_{h = 1}^H \sum_{s = 1}^{t - 1} \EE_{S^{t - 1}} \EE_{f \sim p^t} \EE_{\pi_{\exp(f^s, h)}} \big[ D_H^2 \big( \PP_{h, f}(x_{h+1}^s \mid x_h^s, a_h^s), \PP_{h, f^*}(x_{h+1}^s \mid x_h^s, a_h^s)\big) \big] \notag\\
& \qquad \le \sum_{t = 1}^T \EE_{S^{t - 1}}\EE_{f \sim p^t} \Big[ - \sum_{h = 1}^H \sum_{s = 1}^{t - 1} \Delta L_h^s(f)  - \gamma \Delta V_f + \log \frac{p^t(f) }{p^0(f) } \Big],  \notag \\
& \qquad = \sum_{t = 1}^T \EE_{S^{t - 1}} \inf_p \EE_{f \sim p} \Big[ - \sum_{h = 1}^H \sum_{s = 1}^{t - 1} \Delta L_h^s(f)  - \gamma \Delta V_f + \log \frac{p(f) }{p^0(f) } \Big],
\#
where the last inequality uses \eqref{eq:4800} and the fact that KL divergence is nonnegative.
Therefore, it remains to establish an upper bound for the right hand side of~\eqref{eq:482}, which is the purpose of the following lemma.

\begin{lemma} \label{lemma:mdp:prior}
    For any $t \in [T]$, it holds that
    \$
    \EE_{S^{t - 1}} \inf_p \EE_{f \sim p} \Big[ - \sum_{h = 1}^H \sum_{s = 1}^{t - 1} \Delta L_h^s(f)  - \gamma \Delta V_f + \log \frac{p(f) }{p^0(f) } \Big] \le \omega(\gamma + \eta H T , p^0)  .
    \$
\end{lemma}

\begin{proof}
    See Appendix \ref{appendix:pf:mdp:prior} for a detailed proof.
\end{proof}

By Lemma~\ref{lemma:mdp:prior}, we have 
\# \label{eq:483}
\sum_{t = 1}^T \EE_{S^{t - 1}} \inf_p \EE_{f \sim p} \Big[ - \sum_{h = 1}^H \sum_{s = 1}^{t - 1} \Delta L_h^s(f)  - \gamma \Delta V_f + \log \frac{p(f) }{p^0(f) } \Big] \le \omega(\gamma + \eta H T, p^0) \cdot T .
\#
Combining \eqref{eq:481}, \eqref{eq:482}, and \eqref{eq:483}, we have
\$
\EE [\reg(HT)] &\le \frac{\omega(\gamma + \eta H T, p^0) \cdot HT}{\gamma} + \frac{\gamma d H}{4} + 2\min\{d, HT\} \cdot H + \epsilon H^2 T \\
& \le \frac{\omega(H T, p^0) \cdot HT}{\gamma} + \frac{\gamma dH }{4} + 2\min\{d, HT\} \cdot H + \epsilon H^2 T \\
& \le 4 \sqrt{d H^2 T \cdot \omega(HT, p^0) } ,  
\$
where the second inequality is obtained by $\eta = 1/2$ and $\gamma < HT/2$, and the last inequality is implied by $\epsilon = 1/\sqrt{H^2T}$, $\gamma = 2 \sqrt{\omega(HT, p^0) T/d}$, and $T \ge d$. Finally, by rescaling $T$ to be $T/H$, we conclude the proof of Theorem~\ref{thm:mdp}.
\end{proof}

\subsubsection{Proof of Lemma \ref{lemma:mdp:key}} \label{appendix:pf:mdp:key}

\begin{proof}
Recall that we define $\Delta L_h^t(f)  = L_h^t(f)  - L_h^t(f^*)$ in \eqref{eq:def:delta:l}. Together with the definition of $L_h^t(f) $ in Line~\ref{line:mdp:l} of Algorithm~\ref{alg:mdp}, we have
\# \label{eq:470}
&\EE_{S^{t - 1}}\EE_{f \sim p^t} \Big[ - \sum_{h = 1}^H \sum_{s = 1}^{t - 1} \Delta L_h^s(f)  - \gamma \Delta V_f + \log \frac{p^t(f) }{p^0(f) } \Big] \notag\\
& \qquad = \EE_{S^{t - 1}} \EE_{f \sim p^t} [ - \gamma \Delta V_f ] +  \EE_{S^{t- 1}} \EE_{f \sim p^t} \Big[  \sum_{h = 1}^H \sum_{s = 1}^{t - 1} \eta \log \frac{\PP_{h, f^*}(x_{h+1}^s \mid x_h^s, a_h^s)}{\PP_{h, f}(x_{h+1}^s \mid x_h^s, a_h^s)} + \log \frac{p^t(f) }{p^0(f) }\Big] .
\#
Then we bound the last term in \eqref{eq:470}. For ease of presentation, we define
    \$
    \tilde{L}_h^s(f)  = \eta \log \frac{\PP_{h, f}(x_{h+1}^s \mid x_h^s, a_h^s)}{\PP_{h, f^*}(x_{h+1}^s \mid x_h^s, a_h^s)},
    \$
    and 
    \# \label{eq:def:zeta}
    \tilde{\zeta}_h^s(f)  =  \tilde{L}_h^s(f)  - \log \EE_{\pi_{\exp(f^s, h)}} \big[ \exp \big(  \tilde{L}_h^s(f)  \big) \big]. 
    \#
    Our proof relies on the following lemma.
    \begin{lemma} \label{lemma:zhang}
        For any $(t, f) \in [T] \times \cH$, we have
        \$
        \EE_{S^{t- 1}} \exp \Big(\sum_{h = 1}^H \sum_{s = 1}^{t - 1} \tilde{\zeta}_h^s(f)  \Big) = 1.
        \$
    \end{lemma}

    \begin{proof}[Proof of Lemma \ref{lemma:zhang}]
    The proof follows from \citet{zhang2006}. We prove the lemma by induction. Suppose  
    \# \label{eq:471}
    \EE_{S^{t- 1}} \exp \Big(\sum_{h = 1}^H \sum_{s = 1}^{t - 1} \tilde{\zeta}_h^s(f)  \Big) = 1
    \#
    holds. Then we have
    \$
    \EE_{S^t} \exp \Big(\sum_{h = 1}^H \sum_{s = 1}^{t } \tilde{\zeta}_h^s(f)  \Big) &= \EE_{S^{t-1}} \exp \Big(\sum_{h = 1}^H \sum_{s = 1}^{t - 1} \tilde{\zeta}_h^s(f)  \Big) \cdot \EE_{f^t \sim p^t} \prod_{h=1}^H \Big(\EE_{\pi_{\exp(f^t,h)}} \exp \big( \tilde{\zeta}_h^t(f)  \big) \Big)\\
    & = \EE_{S^{t-1}} \exp \Big(\sum_{h = 1}^H \sum_{s = 1}^{t - 1} \tilde{\zeta}_h^s(f)  \Big) = 1,
    \$
    where the second equality uses the fact that $\EE_{\pi_{\exp(f,h)}} \exp ( \tilde{\zeta}_h^t(f)  ) = 1$ for all $h \in [H]$ and the last inequality is obtained by the induction assumption \eqref{eq:471}. Therefore, we conclude the proof of Lemma \ref{lemma:zhang}.
   \end{proof}

    Back to our proof, we have
    \$
     \EE_{f \sim p^t} \Big[ - \sum_{h = 1}^H \sum_{s = 1}^{t - 1} \tilde{\zeta}_h^s(f)  + \log \frac{p^t(f) }{p^0(f) }\Big] & \ge \inf_p  \EE_{f \sim p} \Big[ - \sum_{h = 1}^H \sum_{s = 1}^{t - 1} \tilde{\zeta}_h^s(f)  + \log \frac{p(f) }{p^0(f) }\Big] \\
    &  = - \log \EE_{f \sim p^0} \exp \Big( \sum_{h = 1}^H \sum_{s = 1}^{t - 1} \tilde{\zeta}_h^s(f) \Big),
    \$
    where the last equality uses the fact that the minimum is achieved by $p(f)  \propto p^0(f) \cdot \exp(\sum_{h = 1}^H \sum_{s = 1}^{t - 1} \tilde{\zeta}_h^s(f) )$. Hence, we have
    \# \label{eq:4500}
     \EE_{S^{t-1}} \EE_{f \sim p^t} \Big[ - \sum_{h = 1}^H \sum_{s = 1}^{t - 1} \tilde{\zeta}_h^s(f)  + \log \frac{p^t(f) }{p^0(f) } \Big] & \ge - \EE_{S^{t - 1}}  \log \EE_{f \sim p^0} \exp \Big( \sum_{h = 1}^H \sum_{s = 1}^{t - 1} \tilde{\zeta}_h^s(f) \Big)  \notag \\
    &  \ge - \log \EE_{f \sim p^0} \EE_{S^{t - 1}} \exp \Big( \sum_{h = 1}^H \sum_{s = 1}^{t - 1} \tilde{\zeta}_h^s(f) \Big) = 0,
    \#
    where the second inequality uses Jensen's inequality and the last inequality is obtained by Lemma~\ref{lemma:zhang}. Combining \eqref{eq:4500} and the definition of $\tilde{\zeta}_h^t$ in \eqref{eq:def:zeta},  for any $h \in [H]$, we have 
    \# \label{eq:451}
    &\EE_{S^{t-1}} \EE_{f \sim p^t} \Big[ \eta \sum_{h = 1}^H \sum_{s = 1}^{t - 1} \log \frac{\PP_{h, f^*}(x_{h+1}^s \mid x_h^s, a_h^s)}{\PP_{h, f}(x_{h+1}^s \mid x_h^s, a_h^s)} + \log \frac{p^t(f) }{p^0(f) }\Big] \notag\\
    & \qquad \ge - \sum_{s = 1}^{t - 1} \sum_{h = 1}^H \EE_{S^{t - 1}} \EE_{f \sim p^t} \bigg[ \log \Big\{ \EE_{\pi_{\exp(f^s,h)}}  \exp \Big( \eta \log \frac{\PP_{h, f}(x_{h+1}^s \mid x_h^s, a_h^s)}{\PP_{h, f^*}(x_{h+1}^s \mid x_h^s, a_h^s)} \Big)   \Big\} \bigg]\notag \\
    & \qquad \ge \sum_{s = 1}^{t - 1} \sum_{h = 1}^H \EE_{S^{t - 1}} \EE_{f \sim p^t} \EE_{\pi_{\exp(f^s,h)}} \big[ D_H^2 \big( \PP_{h, f}(x_{h+1}^s \mid x_h^s, a_h^s), \PP_{h, f^*}(x_{h+1}^s \mid x_h^s, a_h^s)\big) \big],
    \#
    where the last inequality uses the fact that $\eta = 1/2$ and the following lemma. 
    \begin{lemma} \label{lemma:transition}
        Suppose $\tilde{\eta} = 1/2$, for any $(t, h) \in [T] \times [H]$ we have
        \$
        & \log \Big\{ \EE_{\pi_{\exp(f^t, h)}} \exp \Big(\tilde{\eta} \log \frac{\PP_{h, f}(x_{h+1}^t \mid x_h^t, a_h^t)}{\PP_{h, f^*}(x_{h+1}^t \mid x_h^t, a_h^t)}\Big) \Big\} \le -  \EE_{\pi_{\exp(f^t, h)}} D_H^2 \big( \PP_{h, f}(x_{h+1}^t \mid x_h^t, a_h^t), \PP_{h, f^*}(x_{h+1}^t \mid x_h^t, a_h^t)\big).
        \$
    \end{lemma}
    
    \begin{proof}[Proof of Lemma \ref{lemma:transition}]
    By the inequality $\log x \le x - 1$, we have 
    \# \label{eq:11111}
    & \log  \EE_{\pi_{\exp(f^t, h)}} \exp \Big(\tilde{\eta} \log \Big\{ \frac{\PP_{h, f}(x_{h+1}^t \mid x_h^t, a_h^t)}{\PP_{h, f^*}(x_{h+1}^t \mid x_h^t, a_h^t)}\Big) \Big\} \notag\\
    & \qquad \le  \EE_{\pi_{\exp(f^t, h)}} \exp \Big(\tilde{\eta} \log \frac{\PP_{h, f}(x_{h+1}^t \mid x_h^t, a_h^t)}{\PP_{h, f^*}(x_{h+1}^t \mid x_h^t, a_h^t)}\Big) - 1, \notag\\
    & \qquad =   \EE_{(x_h^t,a_h^t) \sim \pi_{\exp(f^t, h)}} \int \sqrt{ \mathrm{d} \PP_{h, f}(x_{h+1}^t \mid x_h^t, a_h^t) \cdot \mathrm{d} \PP_{h, f^*}(x_{h+1}^t \mid x_h^t, a_h^t)}- 1, 
    \#
    where the last equality uses the fact that $\tilde{\eta} = 1/2$. Moreover, for any two distributions over samples $x \sim \cX$, it holds that 
    \$
    D_H^2(P, Q) = \frac{1}{2} \int \big( \sqrt{\mathrm{d}P(x)} - \sqrt{d Q(x)} \big)^2 = 1 - \int \sqrt{dP(x) dQ(x)} .
    \$
    Together with \eqref{eq:11111}, we conclude the proof of Lemma \ref{lemma:transition}.
   \end{proof}

Plugging \eqref{eq:451} into \eqref{eq:470}, we obtain that 
\$
&\EE_{f \sim p^t} \Big( - \sum_{h=1}^H \sum_{s = 1}^{t - 1} \Delta L_h^s(f)  - \gamma \Delta V_f + \log \frac{p^t(f) }{p^0(f) } \Big) \\
& \qquad \ge \EE_{S^{t - 1}} \EE_{f \sim p^t} [ - \gamma \Delta V_f ]  +   \sum_{h = 1}^H \sum_{s = 1}^{t - 1} \EE_{S^{t - 1}} \EE_{f \sim p^t} \EE_{\pi_{\exp(f^s, h)}} \big[ D_H^2 \big( \PP_{h, f}(x_{h+1}^s \mid x_h^s, a_h^s), \PP_{h, f^*}(x_{h+1}^s \mid x_h^s, a_h^s)\big) \big],
\$
which concludes the proof of Lemma \ref{lemma:mdp:key}.
\end{proof}

\subsubsection{Proof of Lemma \ref{lemma:mdp:prior}} \label{appendix:pf:mdp:prior}

\begin{proof}[Proof of Lemma \ref{lemma:mdp:prior}]
    Taking $\tilde{p}(f)  = p^0(f) \cdot  \mathbf{1}(f \in \cH(\epsilon))/p^0(\cH(\epsilon))$ where $\cH(\epsilon)$ is defined in Definition~\ref{def:omega:mdp}, we obtain that
    \# \label{eq:476}
     \inf_p \EE_{f \sim p} \Big( - \sum_{h = 1}^H \sum_{s = 1}^{t - 1} \Delta L_h^s(f)  - \gamma \Delta V_f + \log \frac{p(f) }{p^0(f) }\Big) 
    &  \le  \EE_{f \sim \tilde{p}} \Big[ - \sum_{h = 1}^H \sum_{s = 1}^{t - 1} \Delta L_h^s(f)  - \gamma \Delta V_f + \log \frac{ \tilde{p}(f) }{p^0(f) } \Big] \notag\\
    & \le (\gamma +  \eta H T ) \epsilon  -  \log p^0\big(\cH(\epsilon)\big),
    \#
    where the last inequality is obtained by our choice of $\tilde{p}$ and definition of $\cH(\epsilon)$ in Definition~\ref{def:omega:mdp}.
    Taking expectation over \eqref{eq:476} gives  
    \$
    \EE_{S^{t - 1}} \inf_p \EE_{f \sim p} \Big[ - \sum_{h = 1}^H \sum_{s = 1}^{t - 1} \Delta L_h^s(f)  - \gamma \Delta V_f + \log \frac{p(f) }{p^0(f) } \Big]  &\le (\gamma +  \eta H T ) \epsilon  -  \log p^0\big(\cH(\epsilon)\big) \\
    & \le \omega(\gamma + \eta H T, p^0), 
    \$
    where the last inequality uses the definition of $\omega(\cdot, \cdot)$ in Definition~\ref{def:omega:mdp}.
    Therefore, we conclude the proof of Lemma~\ref{lemma:mdp:prior}.
\end{proof}

\section{Proofs for Partially Observable Interactive Decision Making}

\subsection{Proof for PSR} \label{appendix:pf:psr}


  To facilitate the analysis, we define 
    \# \label{eq:4441}
    \Delta L_h^t(f) = L_h^t(f) - L_h^t(f^*), \quad \Delta V_f = V_f - V^*,
    \#
    where $L_h^t(f) = \eta \log \PP_f(\tau_{h}^t)$, $V_f$ is the optimal value with respect to the model $f$, and $V^*$ is the optimal value with respect to the real model $f^*$.
    Then the posterior defined in \eqref{eq:generic:posteior} is given by
    \# \label{eq:4442}
    p^t(f)  \propto p^0(f) \exp \Big( \gamma \Delta V_f + \sum_{s = 1}^{t - 1} \sum_{h = 0}^{H - 1} \Delta L_h^s(f) \Big) .
    \#
  Recall that we say a PSR has GEC $d_{\GEC}$ if for any $\{f^t, \pi^t = \pi_{f^t}\}_{t \in [T]}$, it holds that
  \# \label{eq:4440}
    & \sum_{t = 1}^T V_{f^t}^{\pi^t} - V^{\pi^t}  \le \Big[d_{\GEC}  \sum_{t = 1}^{T} \sum_{h = 0}^{H - 1}\sum_{s = 1}^{t - 1} D_H^2 \Big( \PP_{f^t}^{\pi_{\exp}(f^s, h)}, \PP_{f^*}^{\pi_{\exp}(f^s, h)} \Big)\Big]^{1/2} + \sqrt{d_{\GEC} H T },
   \#
   where $\pi_{\exp}(f^s, h)$ is the exploration policy defined in Example~\ref{example:policy:psr}.

\begin{proof}[Proof of Theorem~\ref{thm:psr}] 
    Throughout this proof, we use the shorthand $S^t = \{(f^s, \tau_h^s)\}_{0 \le h \le H-1, 1 \le s \le t}$ and $d = d_{\GEC}$ for simplicity. 
    Since we execute exploration policies $\{\pi_{\exp}(f^t,h)\}_{0 \le h \le H -1}$ to collect samples for any $t \in [T]$, we know the total number of interactions is $HT$. We decompose the regret as
    \# \label{eq:4443}
    \frac{\gamma}{H} \cdot \EE[\reg(HT)] & =\gamma \sum_{t = 1}^T \EE_{S^{t-1}} \EE_{f \sim p^t} [V^* - V_{f^*}^{\pi_f}]  \notag \\
    &  = \sum_{t = 1}^T \EE_{S^{t-1}} \EE_{f \sim p^t} [-\gamma \Delta V_f] + \gamma \sum_{t = 1}^T \EE_{S^{t-1}} \EE_{f \sim p^t} [V_f - V_{f^*}^{\pi_f}] \notag \\
    &  \le \sum_{t = 1}^T \EE_{S^{t-1}} \EE_{f \sim p^t} [-\gamma \Delta V_f]  + \sum_{t = 1}^T \sum_{h = 0}^{H - 1} \sum_{s = 1}^{t-1} \EE_{S^{t-1}} \EE_{f \sim p^t} \Big[ D_H^2 \Big(\PP_{f}^{\pi_{\exp}(f^s, h)}, \PP_{f^*}^{\pi_{\exp}(f^s, h)} \Big) \Big] \notag \\
    & \qquad \qquad  + \frac{\gamma^2 d}{4} + \gamma \sqrt{dHT},  
    \#
    where the last inequality follows from the definition of GEC in Definition~\ref{def:gec} (see also \eqref{eq:4440}). Here the exploration policy $\pi_{\exp}(f,h)$ is defined in Example~\ref{example:policy:psr}. We have the following lemma to bound the first two terms in \eqref{eq:4443}.
    \begin{lemma} \label{lemma:psr:key}
        For any $t \in [T]$, it holds that
        \$
        &  \EE_{S^{t-1}} \EE_{f \sim p^t} [-\gamma \Delta V_f] +  \sum_{h = 0}^{H - 1} \sum_{s = 1}^{t-1} \EE_{S^{t-1}} \EE_{f \sim p^t} \Big[ D_H^2 \Big(\PP_{f}^{\pi_{\exp}(f^s, h)}, \PP_{f^*}^{\pi_{\exp}(f^s, h)} \Big) \Big] \\
        & \qquad \le  \EE_{S^{t-1}} \EE_{f \sim p^t} \Big[ - \sum_{h = 0}^{H - 1} \sum_{s = 1}^{t - 1} \Delta L_h^s(f) - \gamma \Delta V_f + \ln \frac{p^t(f)}{p^0(f)} \Big] .
        \$
    \end{lemma}

\begin{proof}[Proof of Lemma \ref{lemma:psr:key}]
This proof involves the exponential model aggregation analysis. See Appendix~\ref{appendix:pf:psr:key} for a detailed proof.
\end{proof}    
    
    By Lemma~\ref{lemma:psr:key}, we have
    \begin{equation}
    \begin{aligned} \label{eq:4444}
    & \sum_{t = 1}^T \EE_{S^{t-1}} \EE_{f \sim p^t} [-\gamma \Delta V_f] + \sum_{t = 1}^T \sum_{h = 0}^{H - 1} \sum_{s = 1}^{t-1} \EE_{S^{t-1}} \EE_{f \sim p^t} \Big[ D_H^2 \Big(\PP_{f}^{\pi_{\exp}(f^s, h)}, \PP_{f^*}^{\pi_{\exp}(f^s, h)} \Big) \Big] \\
        & \qquad \le \sum_{t = 1}^{T} \EE_{S^{t-1}} \EE_{f \sim p^t} \Big[ - \sum_{h = 0}^{H - 1} \sum_{s = 1}^{t - 1} \Delta L_h^s(f) - \gamma \Delta V_f + \ln \frac{p^t(f)}{p^0(f)} \Big] \\
        & \qquad = \sum_{t = 1}^T \EE_{S^{t-1}} \inf_p \EE_{f \sim p} \Big[ - \sum_{h=0}^{H-1} \sum_{s = 1}^{t-1} \Delta L_h^s(f) - \gamma \Delta V_f + \ln \frac{p(f)}{p^0(f)} \Big],
    \end{aligned}
    \end{equation}
    where the last inequality is implied by the form of posterior in \eqref{eq:4442} and the fact that KL-divergence is non-negative. This term can be further bounded by the following lemma.
    \begin{lemma} \label{lemma:psr:prior}
        For any $t \in [T]$, it holds that
        \$
        \EE_{S^{t-1}} \inf_p \EE_{f \sim p} \Big[ - \sum_{h=0}^{H-1} \sum_{s = 1}^{t-1} \Delta L_h^s(f) - \gamma \Delta V_f + \ln \frac{p(f)}{p^0(f)} \Big] \le \omega(\gamma + \eta H T, p^0). 
        \$
    \end{lemma}
    \begin{proof}[Proof of Lemma \ref{lemma:psr:prior}]
        This lemma is implied by a careful choice of $p$ and the definition of $\omega$ (Definition \ref{def:omega:psr}). See Appendix \ref{appendix:pf:psr:prior} for a detailed proof.
    \end{proof}
    Putting \eqref{eq:4443}, \eqref{eq:4444}, and Lemma~\ref{lemma:psr:prior} together, we obtain
    \$
    \EE[\reg(HT)] & \le \frac{\omega(\gamma + \eta HT, p^0) \cdot HT}{\gamma} +\frac{ \gamma d H}{4}  + \sqrt{dH^3T} \\
    & \le \frac{\omega(HT, p^0) \cdot HT}{\gamma} +\frac{ \gamma d H}{4}  + \sqrt{dH^3T} \\
    & \le 2 \sqrt{d H^2 T (H + \omega(HT, p^0))},
    \$
    where the last inequality uses $\gamma = 2 \sqrt{\omega(HT, p^0) \cdot T/d}$. 
    Therefore, by rescaling $T$ to be $T/H$, we finish the proof of Theorem~\ref{thm:psr}.
\end{proof}
   
  \subsubsection{Proof of Lemma \ref{lemma:psr:key}} \label{appendix:pf:psr:key}
\begin{proof}[Proof of Lemma \ref{lemma:psr:key}]
 By the definition of $\Delta L_h^s(f)$ in \eqref{eq:4441}, we have
    \$
    & \EE_{S^{t-1}} \EE_{f \sim p^t} \Big[ - \sum_{h = 0}^{H - 1} \sum_{s = 1}^{t - 1} \Delta L_h^s(f) - \gamma \Delta V_f + \log \frac{p^t(f)}{p^0(f)} \Big] \\
    & \qquad = \EE_{S^{t-1}} \EE_{f \sim p^t} \Big[  - \eta \sum_{h = 0}^{H - 1} \sum_{s = 1}^{t - 1} \log \frac{\PP_{f}(\tau_h^s)}{\PP_{f^*}(\tau_h^s)} - \gamma \Delta V_f + \log \frac{p^t(f)}{p^0(f)} \Big].
    \$
    To facilitate the following analysis, we define
    \# \label{eq:4553}
    \tilde{L}_h^s(f) = \eta \log \frac{\PP_{f}(\tau_{h}^s)}{\PP_{f^*}(\tau_{h}^s)}, \qquad \tilde{\zeta}_h^s(f) = \tilde{L}_h^s(f) - \log \EE_{\tau_h^s} \big[ \exp\big(\tilde{L}_h^s(f)\big) \big] .
    \#
By the same proof of Lemma~\ref{lemma:zhang}, we have
    \# \label{eq:4554}
    \EE_{S^{t-1}} \exp \Big( \sum_{s = 1}^{t-1} \sum_{h = 0}^{H - 1} \tilde{\zeta}_h^s(f) \Big) = 1, 
    \#
    for any $(t, f) \in [T] \times \cH$. Hence, we have
    \# \label{eq:45541}
  \EE_{f \sim p^t} \Big[  -  \sum_{h = 0}^{H - 1} \sum_{s = 1}^{t - 1} \tilde{\zeta}_h^s(f)  + \log \frac{p^t(f)}{p^0(f)} \Big] 
    & \ge \inf_p \EE_{f \sim p} \Big[  - \sum_{h = 0}^{H - 1} \sum_{s = 1}^{t - 1} \tilde{\zeta}_h^s(f)  + \log \frac{p(f)}{p^0(f)} \Big] \notag \\
    & = - \log \EE_{f \sim p^0} \exp \Big( \sum_{h = 0}^{H-1} \sum_{s = 1}^{t - 1} \tilde{\zeta}_h^s(f) \Big) ,
    \#
    where the first equality uses \eqref{eq:4553} and the last inequality uses the fact that the infimum in \eqref{eq:45541} is achieved by $p(f) \propto p^0(f) \cdot \exp(\sum_{h = 1}^{H - 1}\sum_{s = 1}^{t - 1}\tilde{\zeta}_h^s(f))$. Furthermore, by Jensen's inequality, we have 
    \$
    \EE_{S^{t - 1}} \EE_{f \sim p^t}  \Big[  -  \sum_{h = 0}^{H - 1} \sum_{s = 1}^{t - 1} \tilde{\zeta}_h^s(f)  + \log \frac{p^t(f)}{p^0(f)} \Big] \ge - \log \EE_{f \sim p^0} \EE_{S^{t - 1}} \exp \Big( \sum_{h = 0}^{H-1} \sum_{s = 1}^{t - 1} \tilde{\zeta}_h^s(f) \Big) = 0,
    \$
    where the last inequality follows from \eqref{eq:4554}. Combined with the definition of $\tilde{\zeta}_h^s$ in \eqref{eq:4553}, we obtain
    \$
    &\EE_{S^{t-1}} \EE_{f \sim p^t} \Big[  - \eta \sum_{h = 0}^{H - 1} \sum_{s = 1}^{t - 1} \log \frac{\PP_{f}(\tau_h^s)}{\PP_{f^*}(\tau_h^s)}  + \log \frac{p^t(f)}{p^0(f)} \Big] \\
    & \qquad \ge \sum_{s = 1}^{t -  1} \sum_{h = 0}^{H - 1} \EE_{S^{t- 1}} \EE_{f \sim p^t} \bigg[ \log \Big\{ \EE_{\tau_h^s} \exp\Big(\eta \log \frac{\PP_f(\tau_h^s)}{\PP_{f^*}(\tau_h^s)} \Big) \Big\} \bigg] \\
    & \qquad \ge \sum_{s = 1}^{t -  1} \sum_{h = 0}^{H - 1} \EE_{S^{t- 1}}  \EE_{f \sim p^t} \Big[ D_H^2\Big( \PP_f^{\pi_{\exp}(f^s, h)}(\tau_h^s), \PP_{f^*}^{\pi_{\exp}(f^s, h)}(\tau_h^s) \Big) \Big],
    \$
    where the last inequality can be obtained by a similar proof of Lemma~\ref{lemma:transition} with $\eta = 1/2$. Therefore, we finish the proof of Lemma~\ref{lemma:psr:key}.
\end{proof}

\subsubsection{Proof of Lemma \ref{lemma:psr:prior}} \label{appendix:pf:psr:prior}
\begin{proof}[Proof of Lemma \ref{lemma:psr:prior}]
    This proof is similar to that of Lemma~\ref{lemma:mdp:prior} (Appendix \ref{appendix:pf:mdp:prior}), and we present details here for completeness. Let $\tilde{p}(f)  = p^0(f) \cdot \mathbf{1}(f \in \cH(\epsilon))/p^0(\cH(\epsilon))$ where $\cH(\epsilon)$ is defined in Definition~\ref{def:omega:psr}, we obtain that
    \$
     \inf_p \EE_{f \sim p} \Big( - \sum_{h = 0}^{H - 1} \sum_{s = 1}^{t - 1} \Delta L_h^s(f)  - \gamma \Delta V_f + \log \frac{p(f) }{p^0(f) }\Big) 
    &  \le  \EE_{f \sim \tilde{p}} \Big[ - \sum_{h = 0}^{H - 1} \sum_{s = 1}^{t - 1} \Delta L_h^s(f)  - \gamma \Delta V_f + \log \frac{ \tilde{p}(f) }{p^0(f) } \Big] \notag\\
    & \le (\gamma +  \eta H T ) \epsilon  -  \log p^0\big(\cH(\epsilon)\big),
    \$
    where the last inequality is implied by our choice of $\tilde{p}$ and the definition of $\cH(\epsilon)$ in Definition~\ref{def:omega:mdp}. We further have
    \$
    \EE_{S^{t - 1}} \inf_p \EE_{f \sim p} \Big[ - \sum_{h = 0}^{H - 1} \sum_{s = 1}^{t - 1} \Delta L_h^s(f)  - \gamma \Delta V_f + \log \frac{p(f) }{p^0(f) } \Big]  &\le (\gamma +  \eta H T ) \epsilon  -  \log p^0\big(\cH(\epsilon)\big) \\
    & \le \omega(\gamma + \eta H T, p^0), 
    \$
    where the last inequality uses the definition of $\omega(\cdot, \cdot)$ in Definition~\ref{def:omega:psr}.
    Therefore, we finish the proof of Lemma~\ref{lemma:psr:prior}.
\end{proof}

\subsection{Proof for PO-bilinear Class} \label{appendix:pf:pobilinear}

   Recall that we denote $\cH = \Pi \times \cG$. We first define the following potential function.
\begin{equation}
\Phi_h^t(f) = \eta\sum_{s=1}^{t-1} L^s_h(f), 
\end{equation}
where $L_h^s(f) = (\frac{1}{\nb} \sum_{i = 1}^{\nb} l(f, \zeta_{i,h}^s))^2$ and $l$ is defined in \eqref{eq:def:loss:pobilinear}.
Then we have
$$p^t(f) \propto \exp\Big(-\sum_{h=1}^H  \Phi^t_h (f) + \gamma \Delta V_{f} + \log p^0(f)\Big),$$
 Our proof relies on the following two lemmas.
   
\begin{lemma}[Uniform Concentration] \label{lem:unic:po}
     For all $(t,h,\pi, g) \in [T]\times[H]\times \Pi \times \cG$, with probability at least $1-\delta$, we have
      $$
        \sup_{f \in \cH}\left|\frac{1}{\nb} \sum_{i = 1}^{\nb} l(f, \zeta_{i,h}^s) - \EE_{\pi_{\exp(f^s,h)}}l(f, \zeta_h)\right| \leq \sqrt{\frac{A^2 \cdot \iota}{\nb}},
     $$
where $\iota = \log (|\cG| |\Pi| HT/\delta)$. 
\end{lemma}

\begin{proof}
    This lemma is implied by Azuma-Hoeffding inequality and the union bound argument.
\end{proof}

\begin{lemma} \label{lem:upper_bound}
    It holds that
    $$
        \begin{aligned}
                 & \inf_p\EE_{S^{t-1}} \EE_{f\sim p} \Big[ \sum_{h=1}^H \Phi^t_h(f) - \gamma \Delta V_{f} + \log \frac{p(f)}{p^0(f)} \Big]
            \\
             &\qquad \le \gamma \epsilon + 2\eta H (t-1) \epsilon^2 + \frac{2A^2H(t-1)  \eta \cdot \iota}{\nb}  + \log(|\cG| |\Pi|).
        \end{aligned}
    $$
\end{lemma}
\begin{proof}
    For arbitrary $p$ over $\cF$, we can derive that
    $$
        \begin{aligned}
            & \EE_{S^{t-1}} \EE_{f\sim p} \Big[ \sum_{h=1}^H \Phi^t_h(f) - \gamma \Delta V_{f} + \log \frac{p(f)}{p^0(f)} \Big]
            \\
            & \qquad = \EE_{S^{t-1}}\EE_{f\sim p} \Big[-\gamma \Delta V_{f} + \sum_{h=1}^H \sum_{s=1}^{t-1} \eta L_h^s(f) + \log \frac{p(f)}{p^0(f)}\Big]
            \\
            & \qquad \lesssim \EE_{S^{t-1}} \EE_{f\sim p} \Big[-\gamma \Delta V_{f} + \sum_{h=1}^H \sum_{s=1}^{t-1} 2\eta (\EE_{\pi_{\exp(f^s,h)}}l(f, \zeta_h))^2 +  \frac{2A^2H(t-1) \eta \cdot \iota}{\nb} + \log \frac{p(f)}{p^0(f)}\Big],
        \end{aligned}
    $$
    where we omit the failure probability because we can take $\delta = \frac{1}{\text{poly}(K)}$ to make it non-dominating. We first define the set $\cH(\epsilon) = \{f \in \cH: V^* - V_f \le \epsilon; \sup_{h,f, \zeta_h} | l(f, \zeta_h) | \le \epsilon \}$, where $\zeta_h = (\bar{z}_h, a_h, r_h, o_{h+1})$. Then, by realizability assumption, $\cH(\epsilon)$ is not empty. Since $p$ is arbitrary, we can take $p(f) = p^0(f) I(f \in \cH(\epsilon))/p^0(\cH(\epsilon))$ to obtain that 
    $$
        \begin{aligned}
                 & \EE_{S^{t-1}} \EE_{f\sim p} \left[ \sum_{h=1}^H \Phi^t_h(f) - \gamma \Delta V_{f} + \log \frac{p(f)}{p^0(f)} \right]
            \\
            & \qquad \leq  \gamma \epsilon + 2\eta H (t-1) \epsilon^2 + \frac{2A^2 H(t-1)\eta \cdot \iota}{\nb} + \log(|\cG| |\Pi|),
        \end{aligned}
    $$
    where we use $- \log p^0(\cH(\epsilon)) \leq \log(|\cG| |\Pi|)$ for any $\epsilon > 0$, and this concludes the proof of Lemma~\ref{lem:upper_bound}.
\end{proof}

Equipped with these two lemmas, we are ready to present the proof of Theorem~\ref{thm:pobilinear}.

\begin{proof}[Proof of Theorem~\ref{thm:pobilinear}]
We first note that by the definition of our generalized eluder coefficient (cf. Lemma~\ref{lemma:po:bilinear}), we have
$$
\begin{aligned}
    \sum_{t = 1}^T V_{f^t}(x_1) - V^{\pi_{f^t}}(x_1) &\le \Big[d \sum_{h=1}^H \sum_{t=1}^T \Big( \sum_{s=1}^{t-1} \big( \EE_{{\pi}_{\mathrm{exp}}(f^s,h)} l(f^t, \zeta_h) \big)^2 \Big)\Big]^{1/2} + 2\min\{d, HT\} + \epsilon' HT\\
    &\leq  \mu \sum_{h=1}^H \sum_{t=1}^T \Big( \sum_{s=1}^{t-1} \big(\EE_{{\pi}_{\mathrm{exp}}(f^s,h)} l(f^t, \zeta_h) \big)^2 \Big) + \frac{d}{4\mu} + 2\min\{d, HT\} + \epsilon' HT, 
\end{aligned}
$$
where $l$ is defined in \eqref{eq:def:loss:pobilinear}. Choosing $\mu = \eta/(2\gamma)$ gives 
\begin{equation}
 \begin{aligned} \label{eq:9872} 
& \gamma \EE_{S^{t-1}} \EE_{f \sim p^t}   [V^* - V^{\pi_{f}}] \\
&\qquad \leq -\gamma \EE_{S^{t-1}} \EE_{f \sim p^t} \Delta V_{f} + 0.5\eta \sum_{h=1}^H \sum_{s=1}^{t-1} \EE_{S^{t-1}} \EE_{f \sim p^t} \left( \EE_{\pi_{\exp(f^s, h)}} l(f, \zeta^s_{h}) \right)^2\\
& \qquad \qquad + \frac{\gamma^2d}{2\eta} + \gamma H T \epsilon' + 2 \gamma \min\{d, HT\} \\
& \qquad \lesssim \mathbb{E}_{S^{t-1}} \mathbb{E}_{f \sim p^t}\left(\sum_{h=1}^{H} \Phi_h^{t}(f)-\gamma \Delta V_{f}\left(x_1\right)+\log \frac{p^t(f)}{p^0(f)}\right)+ \frac{A^2H(t-1) \eta \cdot \iota}{\nb} \\
& \qquad \qquad + \frac{\gamma^2d}{2\eta} + \gamma H T \epsilon' + 2 \gamma \min\{d, HT\} ,
\end{aligned}
\end{equation}
where the last inequality uses Lemma~\ref{lem:unic:po}. Meanwhile, we have
\begin{equation}
    \begin{aligned} \label{eq:9873}
& \mathbb{E}_{S^{t-1}} \mathbb{E}_{f \sim p^t}\left(\sum_{h=1}^{H} \Phi_h^{t}(f)-\gamma \Delta V_{f}\left(x_1\right)+\log \frac{p^t(f)}{p^0(f)}\right)+ \frac{A^2H(t-1) \eta \cdot \iota}{\nb} \\
& \qquad = \mathbb{E}_{S^{t-1}} \inf _{p} \mathbb{E}_{f \sim p}\left(\sum_{h=1}^{H} \Phi_{h}^{t}(f)-\gamma \Delta V_{f}\left(x_1\right)+\log \frac{p(f)}{p^0(f)}\right)+ \frac{A^2H(t-1) \eta \cdot \iota}{\nb} \\
& \qquad \le \gamma \epsilon + 2\eta H (t-1) \epsilon^2 +  \frac{3A^2H(t-1) \eta \cdot \iota}{\nb} + \log(|\cG| |\Pi|),
    \end{aligned}
\end{equation}
where the last inequality uses Lemma~\ref{lem:upper_bound}. Combining \eqref{eq:9872} and \eqref{eq:9873}, taking summation over $t \in [T]$, and multiplying by $H \nb$, we obtain that
\$
\EE [\reg(K)] &\lesssim K\epsilon + 2\eta H T K \epsilon^2 + \epsilon'HK + 2\min\{d, HT\} \cdot H \nb  \\
& \qquad + \frac{\gamma H \nb d}{2\eta}  + \frac{3 A^2H^2 T^2 \eta \cdot  \iota}{\gamma} + \frac{K}{\gamma}\log(|\cG| |\Pi|).
\$
Then, we choose $\epsilon = \epsilon'= 1/(HK)$, $\nb =(dH/(A^2\iota))^{-1/3}K^{2/3}$, $T = (d/(H^2 A^2\iota))^{1/3}K^{1/3}$, $\gamma = \eta = K^{1/3} \log(|\cG| |\Pi|)$. It follows that 
$$
\EE[\reg(K)] \leq \cO( (dAHK)^{2/3}\cdot \iota^{1/3}),
$$
which concludes the proof of Theorem~\ref{thm:pobilinear}.
\end{proof}

\section{Interpretation of Results} \label{appendix:result}

\subsection{Discussion on $\log |\cH|$} \label{appendix:log:H}

\paragraph{Value-based Hypothesis.}
Theorems~\ref{thm:be}, \ref{thm:bilinear} and \ref{thm:pobilinear} focus on the finite hypothesis class case and the final results depend on $\log|\cH|$. Similar to \citet{jin2021bellman,du2021bilinear,uehara2022provably}, we can extend these results to the infinite hypothesis setting and replace $\log|\cH|$ by log-covering number \citep{jin2021bellman} or generalization gap \citep{du2021bilinear,uehara2022provably}.

\paragraph{Model-based Hypothesis.}
Our results in Theorems~\ref{thm:mdp} and \ref{thm:psr} depend on $\log|\cH|$ if $\cH$ is a finite model class, and they can be easily extended to the case where $\cH$ contains all possible models. Specifically, we can choose $\bar{\cH}$ as an $\varepsilon$-bracket \citep{zhan2022pac} for $\PP_{f}^{\pi}(\tau_H)$ where $\varepsilon$ is a sufficiently small quantity such as $1/T$. Then we can assume $f^* \in \bar{\cH}$ because the misspecification error is small (regret error is a constant). Hence, our regret can only depend on $\log|\bar{\cH}|$ --the logarithmic $\varepsilon$-bracket number. As shown by \citet[][Appendix D]{zhan2022pac}, the logarithmic $\varepsilon$-bracket number scales polynomially with respect to problem parameters, which further implies we can obtain the desired regret bound. We also remark \citet{liu2022partially} adopt a similar optimistic $\varepsilon$-discretization argument in the context of weakly revealing POMDPs.

\subsection{Discussion on $\delta_{\mathrm{PSR}}$} \label{appendix:delta}

Recall the definition of $\delta_{\psr}$ in \eqref{eq:def:delta}:
\$
\delta_{\psr} = \max_{h \in [H]} \min \big\{ \|\KK_h\|_1 \cdot \|\VV_h\|_1 \mid \bar{\DD}_h = \KK_h \VV_h, \KK_h \in \RR^{|\cU_{h+1}| \times \rank(\bar{\DD}_h)}, \VV_h \in \RR^{ \rank(\bar{\DD}_h) \times |\cO|^{h}|\cA|^h } \big\}.
\$

\begin{lemma} \label{lemma:delta:pomdp}
    For weakly revealing POMDPs, latent MDPs, and decodable POMDPs, we have $\delta_{\psr} \le U_A$.
\end{lemma}

\begin{proof}[Proof of Theorem \ref{lemma:delta:pomdp}]
    For any $h \in [H]$, we can factorize $\bar{\DD}_h = [\PP(t_{h+1} \mid \tau_h)]_{(t_{h+1}, \tau_h) \in \cU_{h+1} \times (\cO \times \cA)^h}$ as
    \$
    \bar{\DD}_h = \underbrace{[\PP(t_{h+1} \mid s_{h+1})]_{(t_{h+1}, s_{h+1}) \in \cU_{h+1} \times \cS}}_{\KK_{h}} \times \underbrace{[\PP(s_{h+1} \mid \tau_{h})]_{(s_{h+1}, \tau_h) \in \cS \times (\cO \times \cA)^h}}_{\VV_h}.
    \$
    For any $s_{h+1}$, we have $\| [\PP(t_{h+1} \mid s_{h+1})]_{t_{h+1} \in \cU_{h+1}} \|_1 = \sum_{t_{h+1}} \PP(t_{h+1} \mid s_{h+1}) \le U_A$, which implies that $\|\KK_h\|_1 \le U_A$. Similarly, we can obtain $\|\VV_h\|_1 \le 1$. Hence, we have $\delta_{\psr} \le U_A$, which concludes the proof of Lemma~\ref{lemma:delta:pomdp}.
\end{proof}

\begin{lemma} \label{lemma:delta:linear-pomdp}
    For low-rank POMDPs, we have $\delta_{\psr} \le U_Aq$, where $q=\sup_{h\in[H]}\Vert\int \psi_h(s_{h+1})\ud s_{h+1}\Vert_{\infty}$.
\end{lemma}

\begin{proof}[Proof of Theorem \ref{lemma:delta:linear-pomdp}]
    For any $h \in [H]$, we have
    \$
    &\int_{\tau_h}\PP(t_{h+1}\mid \tau_h) f(\tau_h)\ud\tau_h=\int_{\tau_h}\int_{s_h}\int_{s_{h+1}}\PP(t_{h+1},s_{h+1},s_h\mid \tau_h) f(\tau_h)\ud\tau_h\ud s_h\ud s_{h+1}\\
    &\qquad=\int_{\tau_h}\int_{s_h}\int_{s_{h+1}}\PP(t_{h+1}\mid s_{h+1})\PP(s_{h+1}\mid s_h,a_h)\PP(s_h\mid \tau_h) f(\tau_h)\ud\tau_h\ud s_h\ud s_{h+1}.\nonumber\\
    &\qquad =\Bigl(\int_{s_{h+1}}\PP(t_{h+1}\mid s_{h+1})\psi_h(s_{h+1})\ud s_{h+1}\Bigr)^\top \Bigl(\int_{\tau_h}\int_{s_h}\phi_h(s_h,a_h)\PP(s_h\mid \tau_h) f(\tau_h)\ud\tau_h\ud s_h\Bigr).\nonumber
    \$
    Therefore, when we define $\KK_h:L^1((\cO\times\cA)^h)\to \RR^d$ and $\VV_h:\RR^d\to L^1(\cU_{h+1})$ as
    \$
    \left\{ \begin{array}{l}
     (\VV_h \bm{v})(t_{h+1})=\Bigl(\int_{s_{h+1}}\PP(t_{h+1}\mid s_{h+1})\psi_h(s_{h+1})\ud s_{h+1}\Bigr)^\top\bm{v}, \\
     \KK_h f=\int_{\tau_h}\int_{s_h}\phi_h(s_h,a_h)\PP(s_h\mid \tau_h) f(\tau_h)\ud\tau_h\ud s_h,      \end{array} \right. 
    \$
    we have $\bar{\DD}_h=\KK_h\circ \VV_h$. For any $\bm{v}\in\RR^d$, we have
    \#
    \int_{t_{h+1}}\vert(\VV_h \bm{v})(t_{h+1})\vert \ud t_{h+1}&=\int_{t_{h+1}}\biggl\vert\Bigl(\int_{s_{h+1}}\PP(t_{h+1}\mid s_{h+1})\psi_h(s_{h+1})\ud s_{h+1}\Bigr)^\top\bm{v}\biggr\vert \ud t_{h+1}\nonumber\\
    &\leq \int_{t_{h+1}}\int_{s_{h+1}}\PP(t_{h+1}\mid s_{h+1})\vert \psi_h(s_{h+1})^\top\bm{v}\vert \ud s_{h+1} \ud t_{h+1}\nonumber\\
    &\leq U_A \cdot \int_{s_{h+1}}\vert \psi_h(s_{h+1})^\top\bm{v}\vert \ud s_{h+1}.\nonumber
    \#
    Since $\Vert\int \psi_h(s_{h+1})\ud s_{h+1}\Vert_{\infty}\leq q$, we have 
    \#
    \int_{s_{h+1}}\vert \psi_h(s_{h+1})^\top\bm{v}\vert \ud s_{h+1}\leq q \cdot \Vert\bm{v}\Vert_1,\nonumber
    \#
     which implies $\Vert\VV_h\Vert_{1\to 1}\leq U_Aq$. Similarly, we have $\Vert\KK_h\Vert_{1\to 1}\leq 1$. Hence, we have $\delta_{\psr} \le U_Aq$, which concludes the proof of Lemma~\ref{lemma:delta:linear-pomdp}.
\end{proof}

\begin{lemma} \label{lemma:delta:regular}
    For $\alpha$-regular PSR, we have $\delta_{\psr} \le U_A^2/\alpha$.
\end{lemma}

\begin{proof}[Proof of Lemma \ref{lemma:delta:regular}]
    Suppose $\bar{\DD}_h = \KK_h \VV_h$ with $\KK_h \in \RR^{|\cU_{h+1}| \times \rank(\bar{\DD}_h)}, \VV_h \in \RR^{ \rank(\bar{\DD}_h) \times |\cO|^{h}|\cA|^h }$ and $\|\KK_{h}^\dagger\|_1 \le 1/\alpha$, we have
    \# \label{eq:11112}
    \| \KK_{h}\|_1 \cdot \|\VV_h\|_1 = \| \KK_{h}\|_1 \cdot \| \KK_{h}^{\dagger} \bar{\DD}_h \|_1 \le  \| \KK_{h}\|_1 \cdot \| \KK_{h}^{\dagger}\|_1 \cdot \| \bar{\DD}_h \|_1 \le \frac{\| \KK_{h}\|_1 \cdot \| \bar{\DD}_h \|_1}{\alpha},
    \#
    where the last inequality uses $\|\KK_{h}^\dagger\|_1 \le 1/\alpha$. Furthermore, for any column of $\bar{\DD}_h$ indexed by $\tau_{h}$, we denote it by $[\bar{\DD}_h]_{\tau_{h}}$. Then we have
    \$
    \|[\bar{\DD}_h]_{\tau_{h}}\|_1 = \sum_{u \in \cU_{h+1}} \PP(u \mid \tau_{h}) \le U_A,
    \$
    which implies that $\|\bar{\DD}_h\|_1 \le U_A$. Similarly, we have $\|\KK_h\|_1 \le U_A$. Together with \eqref{eq:11112}, we conclude the proof of Lemma~\ref{lemma:delta:regular}.
\end{proof}

\subsection{Corollaries of Theorem \ref{thm:general:psr}} \label{appendix:cor}

\begin{corollary}[Weakly Revealing POMDP] \label{cor:weak:revealing}
    For $m$-step $\alpha$-weakly revealing POMDPs, the regret of Algorithm~\ref{alg:psr} satisfies
    \$
     \EE[\reg(T)] \le \cO \biggl( \frac{S A^{2m - 1} H \sqrt{S AT \big(H + \omega(HT, p^0) \big) \cdot \iota}}{\alpha^2} \biggr),
    \$
    where $\iota = \cO\Big( \log\Big(1 + \frac{4  S^3 A^{4m - 2}   T }{ \alpha^4}\Big)\Big)$.
\end{corollary}

\begin{proof}[Proof of Corollary \ref{cor:weak:revealing}]
    By Lemma~\ref{lemma:weak:reveal}, we know any $\alpha$-weakly revealing POMDP is an $(\alpha/\sqrt{S})$-generalized regular PSR. Meanwhile, the choice of the core test set in \eqref{eq:6331} implies $U_A \le A^{m-1}$. Combining these two facts with Theorem~\ref{thm:general:psr} and Lemma~\ref{lemma:delta:pomdp}, we conclude the proof of Corollary~\ref{cor:weak:revealing}.
\end{proof}

\begin{corollary}[Latent MDP]\label{cor:lt}
    For a latent MDP $\cM$ with $\vert \cM\vert=M$, when it satisfies Assumption~\ref{asp:lmgr}, the regret of Algorithm~\ref{alg:psr} satisfies
    \$
    \EE[\reg(T)] \le \cO \biggl( \frac{SMAU_A^2 H \sqrt{SAM T \big(H + \omega(HT, p^0) \big) \cdot \iota}}{\alpha^2} \biggr),
    \$
    where $\iota = \cO\Big( \log\Big(1 + \frac{4  S^3M^3A^2U_A^4   T }{ \alpha^4}\Big)\Big)$. Here $U_A=\max_{h\in[H]}\vert \cU_{A,h}\vert $, and $\cU_{A,h}$ is the set of the action sequences in $\cU_h$.
\end{corollary}
\begin{proof}[Proof of Corollary \ref{cor:lt}]
    By Lemma~\ref{lemma:latent:mdp}, we know any latent MDP is an $(\alpha/\sqrt{SM})$-generalized regular PSR. We showed in Appendix \ref{appendix:latent:mdp} that we can view the latent MDP as a POMDP with $\vert\bar{\cS}\vert=SM$, which implies that $d_{\psr}=SM$. We also show in Appendix \ref{appendix:latent:mdp} that the latent MDP $\cM$ is a generalized regular PSR with the test $\{\bar{\cU}_h\}_{h=1}^H$. When we define $\bar{\cU}_{A,h}$ as the set of the action sequences in $\bar{\cU}_h$, and define $\bar{U}_A=\max_{h\in[H]}\vert \bar{\cU}_{A,h}\vert$, we have $\bar{U}_A\leq U_A$ by the definition of $\{\bar{\cU}_{h+1}\}_{h=1}^H$ in Appendix \ref{appendix:latent:mdp}. Combining these three facts with Theorem~\ref{thm:general:psr} and Lemma~\ref{lemma:delta:pomdp}, we conclude the proof of Corollary~\ref{cor:lt}.
\end{proof}

\begin{corollary}[Decodable POMDP] \label{cor:decodable}
    For $m$-step decodable POMDPs, the regret of Algorithm~\ref{alg:psr} satisfies
    \$
     \EE[\reg(T)] \le \cO \bigl(  A^{2m - 1} H \sqrt{S AT \big(H + \omega(HT, p^0) \big) \cdot \iota} \bigr),
    \$
    where $\iota =  \cO( \log(1 + 4 S  A^{4m - 2}  T ) )$.
\end{corollary}

\begin{proof}[Proof of Corollary \ref{cor:decodable}]
    By Lemma \ref{lemma:decodable}, we know any $m$-step decodable POMDP is a $1$-generalized regular PSR. Meanwhile, the choice of the core test set in \eqref{eq:6431} implies that $U_A \le A^{m-1}$. Combining these two facts with Theorem~\ref{thm:general:psr} and Lemma~\ref{lemma:decodable} finishes the proof of Corollary~\ref{cor:decodable}.
\end{proof}

\begin{corollary}[Low-rank POMDP]\label{cor:lr}
    For a low-rank POMDP that satisfies Assumption \ref{asp:fs35}, the regret of Algorithm~\ref{alg:psr} satisfies
    \$
    \EE[\reg(T)] \le \cO \biggl( \frac{A^{2m - 1} H \sqrt{d A T \big(H + \omega(HT, p^0) \big) \cdot \iota}}{\alpha^2} \biggr),
    \$
    where $\iota = \cO\Big( \log\Big(1 + \frac{4  d^2A^{4m - 2} q^2  T }{ \alpha^4}\Big)\Big)$ and $q=\sup_{s\in\cS,h\in[H],d_0\in [d]}\vert[\psi_h(s)]_{d_0}\vert $. 
\end{corollary}
\begin{proof}[Proof of Corollary \ref{cor:lr}]
    By Lemma~\ref{lemma:lowrank-pomdp}, we know any low-rank MDP that satisfies Assumption~\ref{asp:fs35} is an $\alpha$-generalized regular PSR.  
    By the decomposition of the observable operator in Lemma~\ref{lemma:delta:linear-pomdp}, we have 
    $d_{\psr}=d$. Combining these two facts with Theorem~\ref{thm:general:psr} and Lemma~\ref{lemma:delta:linear-pomdp}, we conclude the proof of Corollary~\ref{cor:lr}.
\end{proof}

\begin{corollary}[Regular PSR] \label{cor:regular:psr}
    For $\alpha$-regular PSR, the regret of Algorithm~\ref{alg:psr} satisfies
     \$
     \EE[\reg(T)] \le \cO \biggl( \frac{A U_A^2 H \sqrt{d_{\psr} \cdot A T \big(H + \omega(HT, p^0) \big) \cdot \iota}}{\alpha^4} \biggr),
     \$
      where $\iota = 2 \log\Big(1 + \frac{4 d_{\psr} A^2U_A^2  \delta_{\psr}^2 T }{ \alpha^4}\Big)$.
\end{corollary}

\begin{proof}[Proof of Corollary \ref{cor:regular:psr}]
    The result is directly implied by Lemma~\ref{lemma:regular:psr} and Theorem~\ref{thm:general:psr}.
\end{proof}

\subsection{Discussion on Large Observation Spaces} \label{appendix:large:obs}

Except for $\log|\cH|$ (or $\omega(HT, p^0)$), our result in Theorem~\ref{thm:general:psr} is independent of the cardinality of observation space $\cO$, which implies that we can potentially tackle the large observation spaces. For example, the following low-rank assumption on the emission matrix allows our regret bound to be independent of $O = |\cO|$.

\begin{assumption}[Low-rank Emission Matrix] \label{assumption:lowrank:observation}
    For any $(s, o, h) \in \cS \times \cO \times [H]$, there exists a known feature $\psi_h:\cO \rightarrow\Delta_{[d_o]}$ such that
    \$
    \OO_h(o \mid s) = \phi_h(s)^\top \psi_h(o),
    \$
     where $\phi_h:\cS \rightarrow\RR_+^{d_o}$ is unknown. We also assume that $\|\phi_h(s)\|_{\infty} \le R$ for any $(s, h) \in \cS \times [H]$, where $R \ge 1$ is a positive real number.
\end{assumption}

\begin{proposition} \label{cor:obs:free}
     For $m$-step $\alpha$-weakly revealing POMDPs satisfying Assumption~\ref{assumption:lowrank:observation}, the regret of Algorithm~\ref{alg:psr} satisfies
    \$
     \EE[\reg(T)] \le {\cO} \biggl( \frac{S^2 A^{2m} H \sqrt{ (SA + d_o) H T \cdot \iota}}{\alpha^2} \biggr) ,
    \$
    where $\iota = \cO\Big( \log\big(1 + \frac{4  S^3 A^{4m - 2}   T }{ \alpha^4}\big) \cdot \log(SAd_o RHT)\Big)$.
\end{proposition}

\begin{proof}[Proof of Proposition \ref{cor:obs:free}]
    Let $\Phi_h = [\phi_h(s)]_{s \in \cS} \in \RR^{d_o \times S}$ and $\Psi_h =  [\psi_h(o)]_{o \in \cO} \in \RR^{d_o \times O} $, Assumption~\ref{assumption:lowrank:observation} equivalents to $\OO_h = \Phi_h^\top \Psi_h$, where $\Psi_h$ is known. Similar to \citet{liu2022partially}, we perform the optimistic $\varepsilon$-discretization argument. For $f = (\TT, \Phi = \{\Phi_h\}_{h \in [H]}, \mu_1) \in \RR^d$ with $d = H(S^2A + Sd_o) + S$, its optimistic $\varepsilon$-discretization $\bar{f} \in \RR^d$ is defined as $\bar{f} = [\bar{f}_i]_{i \in [d]} = [\lceil f_i/\varepsilon \rceil \times \varepsilon]_{i \in [d]}$. To simplify the notation, we denote $\bar{\cH} = \{\bar{f}\}_{f \in \cH}$.
For any $f \in \cH \cup \bar{\cH}$, trajectory $\tau_H = (o_{1:H}, a_{1:H})$ and policy $\pi$, we define
\$
\PP_f^{\pi}(\tau_H) &= \sum_{s_{1:H}} \mu_1(s_1) \cdot \phi_1(s_1)^\top \psi_1(o_1) \cdot \pi_1(o_1 \mid s_1) \cdot \TT_{1, a_1} (s_2 \mid s_1) \times \\
& \cdots \\
& \times \phi_{H-1}(s_{H-1})^\top \psi_{H-1}(o_{H-1}) \cdot \pi_{H-1}(a_{H-1} \mid o_{1: H - 1}, a_{1: H - 2}) \cdot \TT_{H - 1, a_{H - 1}}(s_H \mid s_{H - 1}) \\
& \times \phi_{H}(s_{H})^\top \psi_{H}(o_{H}) \cdot \pi_{H}(a_{H} \mid o_{1: H}, a_{1: H - 1}) .
\$
By the definition of $\bar{f}$, we have
\$
\PP_f^\pi(\tau_H) \le \PP_{\bar{f}}^\pi(\tau_H) 
\$
for any policy $\pi$ and length-$H$ trajectory $\tau_H$. Moreover, if we choose $\varepsilon = 1/(C (S + d_o + A) RHT )$ where $C$ is a sufficiently large constant, we have
\$
\big\| \PP_f^\pi - \PP_{\bar{f}}^\pi \big\|_1 \le 1/T.
\$
Note $\bar{\cH} \in [0, R]^d$, its log-cardinality with respect to $\varepsilon = 1/(C (S + d_o + A) RHT )$ is upper bounded by
\# \label{eq:obs:free}
\cO \big( H(S^2A + S d_o) \cdot \log(SAd_o R H T) \big),
\#
which is independent of $O$. 

Given the hypothesis class $\cH$ that contains all possible models, we choose $\bar{\cH}$ as the input of our algorithm. Then the regret bound only depends on $\log|\bar{\cH}|$ up to a constant misspecification error. Combining Corollary~\ref{cor:weak:revealing} and \eqref{eq:obs:free}, we conclude the proof of Proposition~\ref{cor:obs:free}.
\end{proof}


\section{Auxiliary Lemmas}

\begin{lemma} \label{lem:hellinger:1}
    The following inequalities hold for any distributions $(P, Q)$ and random variables $(X, Y)$:
    \begin{itemize}
        \item $\|P - Q\|_1^2 \le 8 D_H^2(P, Q)$
        \item $\EE_{X \sim P_X} [D_H^2(P_{Y |X}, Q_{Y | X}) ] \le 4 D_H^2(P_{X,Y}, Q_{X,Y})$
    \end{itemize}
\end{lemma}

\begin{proof}
    See Appendix A.2 of \citet{foster2021statistical} for a detailed proof.
\end{proof}

\begin{lemma}
\label{lem:Gibbs_var}
Let $\nu$ be a probability distribution over $x \in \cX$. Then, $\EE_{x \sim \nu} [f(x) + \log\nu(x)]$ is minimized at $\nu(x) \propto \exp(-f(x))$.
\end{lemma}
\begin{proof}
This result is equivalent to the fact that KL-divergence is non-negative. We also remark that this is a corollary of Gibbs variational principle whose proof can be found in \citet[][Lemma 4.10]{van2014probability}.
\end{proof}

\begin{lemma}[Elliptical Potential Lemma \citep{dani2008stochastic,rusmevichientong2010linearly,abbasi2011improved}] \label{lemma:elliptical:potential}
    Let $\{x_i\}_{i \in [T]}$ be a sequence of vectors, where $x_i \in \cV$ for some Hilbert space $\cV$. Let $\Lambda_0$ be a positive-definite matrix and $\Lambda_t = \Lambda_0 + \sum_{i = 1}^{t - 1} x_i x_i^\top$. It holds that
    \$
    \sum_{i = 1}^T \min \{1, \|x_i\|_{\Lambda_i^{-1}}^2 \} \le 2 \log \Big( \frac{ \det(\Lambda_{T+1})}{\det(\Lambda_1)} \Big) .
    \$
\end{lemma}

\begin{proof}
    See Lemma 11 of \citet{abbasi2011improved} for a detailed proof.
\end{proof}

\end{document}